\documentclass[final,5p,times,twocolumn]{elsarticle}
\usepackage[T1]{fontenc}
\usepackage{microtype}
\usepackage{amsmath,amssymb,amsthm,mathtools,bm,mathrsfs}
\usepackage{booktabs,array,tabularx,multirow,makecell}
\usepackage[ruled,vlined,linesnumbered]{algorithm2e}
\usepackage{graphicx}
\usepackage{capt-of}
\usepackage{enumitem}
\usepackage{xcolor}
\usepackage{hyperref}
\usepackage{etoolbox}
\usepackage{placeins}
\usepackage{dblfloatfix}
\usepackage{cuted}
\usepackage{float}

\biboptions{sort&compress}

\hypersetup{
  colorlinks=true,
  linkcolor=red,   
  citecolor=blue,  
  urlcolor=blue,   
  filecolor=red,   
  pdftitle={RCNNM Model, Recovery Guarantees, Optimization, and Experiments},
  pdfauthor={RCNNM theory manuscript}
}

\AtBeginDocument{%
  \hypersetup{
    colorlinks=true,
    linkcolor=red,   
    citecolor=blue,  
    urlcolor=blue,   
    filecolor=red
  }%
}

\setlist{nosep,leftmargin=1.4em}
\SetAlgoCaptionSeparator{.}
\SetAlgoNlRelativeSize{-1}
\SetAlCapFnt{\small}
\SetAlCapNameFnt{\small\bfseries}
\SetKwInput{KwInput}{Input}
\SetKwInput{KwOutput}{Output}
\DontPrintSemicolon
\allowdisplaybreaks[2]
\newcommand{\ExpGraphic}[3]{%
  \IfFileExists{#1}{%
    \includegraphics[width=#2,height=#3,keepaspectratio]{#1}%
  }{%
    \fcolorbox{black!35}{black!3}{%
      \begin{minipage}[c][#3][c]{\dimexpr#2-2\fboxsep-2\fboxrule\relax}
      \centering\footnotesize
      Figure placeholder\\[2pt]
      \texttt{\detokenize{#1}}
      \end{minipage}}%
  }%
}
\newcommand{\ExpGraphicFit}[3]{%
  \IfFileExists{#1}{%
    \includegraphics[width=#2,height=#3,keepaspectratio]{#1}%
  }{%
    \fcolorbox{black!35}{black!3}{%
      \begin{minipage}[c][#3][c]{\dimexpr#2-2\fboxsep-2\fboxrule\relax}
      \centering\footnotesize
      Figure placeholder\\[2pt]
      \texttt{\detokenize{#1}}
      \end{minipage}}%
  }%
}

\newcommand{\ExpGraphicWide}[2]{%
  \IfFileExists{#1}{%
    \includegraphics[width=#2]{#1}%
  }{%
    \fcolorbox{black!35}{black!3}{%
      \begin{minipage}[c][0.24\textheight][c]{\dimexpr#2-2\fboxsep-2\fboxrule\relax}
      \centering\footnotesize
      Figure placeholder\\[2pt]
      \texttt{\detokenize{#1}}
      \end{minipage}}%
  }%
}

\newtheorem{theorem}{Theorem}[section]
\newtheorem{lemma}[theorem]{Lemma}
\newtheorem{proposition}[theorem]{Proposition}
\newtheorem{corollary}[theorem]{Corollary}

\theoremstyle{definition}

\newtheorem{remark}[theorem]{Remark}

\newcommand{\A}{\mathscr A}
\newcommand{\B}{\mathscr B}
\newcommand{\Hh}{\mathbb X}
\newcommand{\V}{\mathbb V}
\newcommand{\Tsp}{\mathbb T}
\newcommand{\Iset}{\mathbb I}
\newcommand{\Kset}{\mathbb K}
\newcommand{\Oset}{\Omega}
\newcommand{\Gset}{\Gamma}
\newcommand{\Lset}{\Lambda}
\newcommand{\Bset}{\mathbb B}
\newcommand{\Cset}{\mathfrak C}
\newcommand{\Sset}{\mathbb S}
\newcommand{\Jtau}{\mathbb J_\tau}
\newcommand{\Dop}{\mathscr D}
\newcommand{\Iop}{\mathscr I}
\newcommand{\KT}{\mathscr N_{\mathbb T}}
\newcommand{\RR}{\mathbb R}
\newcommand{\supp}{\operatorname{supp}}
\newcommand{\range}{\operatorname{range}}
\newcommand{\rank}{\operatorname{rank}}
\newcommand{\vecop}{\operatorname{vec}}
\newcommand{\diag}{\operatorname{diag}}
\newcommand{\Tr}{\operatorname{Tr}}
\newcommand{\SVT}{\operatorname{SVT}}
\newcommand{\Ber}{\operatorname{Ber}}
\newcommand{\sgn}{\operatorname{sgn}}
\newcommand{\st}{\mathrm{s.t.}}

\newcommand{\Max}{\max}
\newcommand{\argmin}{\operatorname*{arg\,min}}
\newcommand{\Min}{\min}

\newcommand{\Ker}{\ker}
\newcommand{\Dim}{\dim}
\newcommand{\degR}{\operatorname{deg}_{R}}
\newcommand{\degC}{\operatorname{deg}_{C}}
\newcommand{\Inf}{\inf}
\newcommand{\Log}{\log}
\newcommand{\Exp}{\exp}
\newcommand{\Bin}{\operatorname{Bin}}

\newcommand{\op}{\mathrm{op}}

\newcommand{\eps}{\varepsilon}
\newcommand{\norm}[1]{\left\lVert #1\right\rVert}
\newcommand{\abs}[1]{\left|#1\right|}
\newcommand{\ip}[2]{\left\langle #1,#2\right\rangle}
\newcommand{\trans}{\mathsf T}
\newcommand{\1}{\mathbf 1}
\newcommand{\E}{\mathbb E}
\newcommand{\Pp}{\mathbb P}

\journal{Knowledge-Based Systems}

\begin{document}

\begin{frontmatter}

\title{Robust Tensor Completion via Reflective Convolution Nuclear Norm Minimization}

\author[aff1]{Weiguo Zhou}
\ead{2776863767@qq.com}
\author[aff1]{Feng Zhang\corref{cor1}}
\ead{zfmath@swu.edu.cn}
\author[aff1]{Wenjin Qin}
\ead{qinwenjin2021@163.com}
\author[aff2]{Jianwen Huang}
\ead{hjw1303987297@126.com}
\cortext[cor1]{Corresponding author.}

\affiliation[aff1]{organization={School of Mathematics and Statistics, Southwest University},
                   city={Chongqing},
                   postcode={400715},
                   country={China}}

\affiliation[aff2]{organization={College of Artificial Intelligence and Advanced Interdisciplinary Sciences, Chongqing Normal University},
                   city={Chongqing},
                   postcode={401331},
                   country={China}}


\begin{abstract}
  Robust tensor completion recovers multidimensional data from partial observations corrupted by sparse gross errors. Existing methods typically promote low rank through tensor nuclear norms, whereas convolutional lifting enforces low rank on a matrix of translated data copies, capturing translation-invariant structure beyond conventional tensor ranks. Constructing these copies requires a boundary condition, and convolutional models conventionally adopt circular continuation for this purpose, in which every shift is a permutation. For a finite nonperiodic tensor, however, circular continuation connects opposite sides of the tensor, creating wrap-around neighborhoods that do not exist in the data. We propose reflective convolution nuclear norm minimization (RCNNM), which replaces circular shifts with endpoint-nonrepeating reflection for robust tensor completion. Reflection breaks the permutation structure of circular lifting, so entries near the boundary occur with nonuniform multiplicities and the lifting obeys the weighted Gram identity $\A_{\mathrm{ref}}^*\A_{\mathrm{ref}}(\mathcal X)=\mathcal D\odot \mathcal X$, with $\mathcal D$ recording these multiplicities. Using this weighted geometry, we show that RCNNM recovers the underlying tensor and the sparse gross errors exactly with high probability under random sampling and sparse corruption. For bounded dense perturbations, the recovery error remains linear in the perturbation level. The same diagonal structure keeps the optimization efficient by giving the two-block ADMM a closed-form entrywise update, while singular-value thresholding is evaluated through a small eigenproblem on the right Gram matrix instead of a full SVD of the lifted matrix. Across color and multispectral image experiments, RCNNM achieves average boundary-PSNR gains of up to 3.26 dB over its circular-lifting counterpart while remaining competitive globally. Because reflection changes the neighborhoods that encounter the finite boundary, the boundary-dominant gain is consistent with the continuation rule, and boundary-restricted metrics are reported in addition to full-image averages.
\end{abstract}

\begin{keyword}
tensor completion \sep convolutional low-rankness \sep reflective boundary
condition \sep nuclear norm minimization
\end{keyword}

\end{frontmatter}

\section{Introduction}

\label{sec:introduction}
Recovering multidimensional data from incomplete observations is a fundamental problem in signal and image processing. Without additional structure the problem is ill posed, and low rank is a natural prior for data whose intrinsic dimension is much smaller than their ambient dimension. In the matrix case, nuclear norm minimization provides a tractable convex surrogate for rank minimization and admits exact-recovery guarantees under suitable sampling and incoherence conditions \cite{candes2009matrix,candes2010power,recht2011simpler}. Low-rank--sparse matrix recovery, including robust principal component analysis and recovery from corrupted or erased entries, explicitly models sparse gross errors in addition to low rank \cite{candes2011rpca,chandrasekaran2011rank,chen2013errors}. Tensor representations preserve multiway structure across modes without fixing a particular matrix unfolding \cite{kolda2009tensor}. Tensor completion and robust tensor recovery exploit this structure to reconstruct missing and corrupted multidimensional data \cite{gandy2011tensor,liu2013tensor,goldfarb2014robust,zhang2023kbsrandomized,yang2026kbssvpn,jiang2023kbssuperpixel,wu2023kbstr}.

Formally, let $\Hh=\RR^{m_1\times\cdots\times m_d}$ be the space of order-$d$ tensors with index set $\Iset=\prod_{a=1}^d\{1,\ldots,m_a\}$, let $\Oset\subseteq\Iset$ be the set of observed indices, and let $\mathscr P_\Oset$ be the coordinate projection onto $\Oset$. Given the observed tensor $\mathcal M\in\Hh$, low-rank tensor completion seeks a low-rank tensor $\mathcal X\in\Hh$ by
\begin{equation}
\Min_{\mathcal X}\quad\rank(\mathcal X)
\qquad\st\quad \mathscr P_\Oset(\mathcal X-\mathcal M)=0.
\label{eq:intro-ltc}
\end{equation}
Here $\rank(\mathcal X)$ is understood generically, with the particular tensor-rank notion determined by the model. Since such rank functions are nonconvex and their direct minimization is generally intractable, practical methods replace them with tractable surrogates, and the choice of rank notion and surrogate distinguishes the tensor completion models discussed below.

Unlike the matrix case, tensor completion has no unique notion of rank, and low-rank structure on a tensor can be characterized through different tensor decompositions, including Tucker and CP decompositions \cite{kolda2009tensor} and the tensor SVD \cite{kilmer2011factorization}. Each of these decompositions induces a different surrogate for the rank term in \eqref{eq:intro-ltc}, and the completion models built on them all impose low rank directly on the tensor $\mathcal X$. Imposing low rank on the tensor itself is natural for globally low-rank data, but translation-invariant structure, such as a pattern that repeats across the tensor, is captured by low rank only after translated copies are aligned. None of them translates the tensor, so no boundary condition is involved. The boundary condition becomes decisive once the tensor is mapped into a lifted matrix, because translating a finite tensor requires choosing how it continues beyond its boundary. That choice determines the lifting operator, its geometry, and the low-rank structure the model captures. The tensor completion models induced by these decompositions therefore do not address the boundary issues created when low rank is imposed on a matrix of translated copies.

Lifting-based models pursue this alternative, imposing low rank on a larger structured matrix derived from the tensor $\mathcal X$ rather than on $\mathcal X$ itself. The lifted matrix can take a Hankel, Toeplitz, patch, or convolutional form, and the low rank of each encodes a different kind of structure. Annihilation-based Hankel lifting has proven effective for image inpainting \cite{jin2015aloha}. Low-rank models built from overlapping patches likewise come with recovery guarantees \cite{cai2024patch}. The last of these constructions is convolutional lifting, which is the focus of the rest of this paper. It stacks translated copies of the tensor as columns of a matrix \cite{liu2023cnnm}. Let $k_a\ge1$ be the kernel length along mode $a$, and collect the kernel offsets in $\Kset=\prod_{a=1}^d\{0,\ldots,k_a-1\}$, so that $K=|\Kset|=\prod_{a=1}^d k_a$ is the number of shifts. Given a family of translation operators $\{\mathscr T_{\bm{t}}:\bm{t}\in\Kset\}$ mapping $\Hh$ into itself, and an enumeration $\bm{t}^{(1)},\ldots,\bm{t}^{(K)}$ of $\Kset$, the lifted matrix is
\begin{equation}
\A_{\mathscr T}(\mathcal X)=
\left[
\vecop(\mathscr T_{\bm{t}^{(1)}}(\mathcal X)),\ldots,
\vecop(\mathscr T_{\bm{t}^{(K)}}(\mathcal X))
\right]\in\RR^{N\times K},
\label{eq:intro-lift}
\end{equation}
where $\vecop$ stacks the entries of a tensor into a column vector and $N=\prod_{a=1}^d m_a$ is the number of entries of $\mathcal X$. The rank of the lifted matrix in \eqref{eq:intro-lift} defines a convolutional rank for the tensor $\mathcal X$. A small convolutional rank means that many convolution kernels annihilate the tensor, so the columns of the lifted matrix are strongly linearly dependent. This prior is natural whenever local structure repeats across the data. Periodic data with sufficiently sparse spectral structure can yield an exactly low-rank lifted matrix, while locally smooth data can give approximately low-rank lifted matrices with many near-duplicate columns and rows.

In the convolutional setting, convolution nuclear norm minimization (CNNM) builds the lifted matrix from shifted versions of the tensor $\mathcal X$ \cite{liu2023cnnm}. CNNM was proposed for tensor completion with arbitrary sampling, a problem that covers time-series forecasting as a special case. In many applications, missing entries may form coherent structures such as scratches, occlusions, or sensor failures, patterns that are not covered by guarantees derived under uniform random sampling. To handle series with trends and nonseasonal dynamics, CNNM has been extended with a learnable orthonormal transform \cite{liu2022learning}. Convolution principal component pursuit (CPCP) extends this idea to tensor recovery with sparse corruptions \cite{li2026cpcp}. Convolutional low-rankness generalizes Fourier-domain sparsity. For a full kernel, the circular convolution matrix is diagonalized by the discrete Fourier transform, and its nuclear norm is proportional to the $\ell_1$ norm of the Fourier coefficients, with the proportionality determined by the DFT normalization \cite{liu2023cnnm}. Together, these results show that translation-induced low rank captures structure that conventional tensor ranks do not describe directly.

The convolutional lifting used in CNNM and related models is the special case of \eqref{eq:intro-lift} in which the translation operators are circular shifts. Every circular shift permutes the entries of the tensor, so the Gram operator of the corresponding lifting $\A_{\mathrm{circ}}$ acts as a uniform scaling,
\begin{equation}
\A_{\mathrm{circ}}^*\A_{\mathrm{circ}}(\mathcal X)=K\,\mathcal X.
\label{eq:intro-circ-gram}
\end{equation}
Equation \eqref{eq:intro-circ-gram} shows that every entry of $\mathcal X$ appears exactly $K$ times in the lifted matrix. Circular continuation suits periodic tensors, but for a finite nonperiodic tensor it connects opposite sides of the tensor, creating wrap-around neighborhoods that do not exist in the data. Because the continuation rule fixes which entries are treated as adjacent within each translated copy, the low-rank model must fit a periodicity absent from the observed data, which can bias the recovered entries near the boundary. Reflective boundary conditions have long been used in image restoration to avoid such artifacts of periodic continuation \cite{ng1999neumann,nagy2004reflexive,donatelli2007boundary}, since mirror extension preserves the local smoothness of the tensor across the boundary. Reflection has also appeared in tensor models, for instance in cosine-transform-based tensor products \cite{xu2019dct}, DCT-based robust completion within the M-product framework \cite{karmakar2026twctv}, and reflective block-convolution tensor decompositions \cite{molavi2023reflective}. In these models, however, reflection is inherited from the transform or factorization rather than chosen for how it treats the finite boundary. CNNM and its robust CPCP extension both use circular shifts in their convolutional liftings \cite{liu2023cnnm,li2026cpcp}. We instead study how convolutional low-rank lifting and its recovery geometry change when circular shifts are replaced by endpoint-nonrepeating reflection under sparse corruption and missing entries.

Unlike those transform-driven constructions, reflection changes the lifting operator itself and hence its Gram geometry. When the translation operators are instead reflected shifts, which continue the tensor beyond its boundary by endpoint-nonrepeating reflection, the resulting lifting is denoted $\A_{\mathrm{ref}}$. Shifts are then no longer permutations, and the multiplicity of an entry of $\mathcal X$ depends on its distance to the boundary. This departure from the circular case is captured by the weighted Gram identity
\begin{equation}
\A_{\mathrm{ref}}^*\A_{\mathrm{ref}}(\mathcal X)=\mathcal D\odot \mathcal X,
\label{eq:intro-ref-gram}
\end{equation}
where $\mathcal D$ records how many times each entry of $\mathcal X$ occurs in the lifted matrix and $\odot$ denotes the entrywise product. Since the zero offset always contributes a column equal to $\vecop(\mathcal X)$, the reflective lifting is injective and the induced reflective convolution nuclear norm is a genuine norm rather than a seminorm. On this basis, we formulate reflective convolution nuclear norm minimization (RCNNM) for robust tensor completion. The precise optimization model and its special cases are treated in Section~\ref{sec:rcnnm-model}, following the definition of reflective lifting and its induced norm.

The main contributions are as follows.
\begin{itemize}
\item We introduce an arbitrary-order endpoint-nonrepeating reflective convolutional lifting. Its nonuniform multiplicity geometry is captured by the weighted Gram identity \eqref{eq:intro-ref-gram}, which also shows that the induced reflective convolution nuclear norm is a genuine norm.

\item We formulate the RCNNM model and prove deterministic identifiability with strict dual-certificate conditions, high-probability exact recovery under random sampling and sparse corruption, and stability under bounded dense perturbations.

\item We derive a two-block ADMM whose diagonal Gram structure gives a closed-form entrywise update, while singular-value thresholding is evaluated through a small eigenproblem on the right Gram matrix. Experiments on synthetic and real image data show that the improvement over the circular-boundary version is most pronounced near the boundary, where the continuation rule has its direct effect.
\end{itemize}

The remainder of the paper is organized as follows. Section~\ref{sec:preliminaries} introduces the notation and preliminaries. Section~\ref{sec:core-model} defines the reflective lifting and the RCNNM model. Recovery guarantees are developed in Section~\ref{sec:core-det}, and the optimization method is presented in Section~\ref{sec:optimization}. Section~\ref{sec:experiments} reports the numerical results, followed by the conclusion in Section~\ref{sec:conclusion}. 

\section{Notation and Preliminaries}
\label{sec:preliminaries}

This section establishes notation and reviews key matrix-theoretic preliminaries for subsequent analysis. Section~\ref{sec:notation} sets typographical conventions for scalars, vectors, matrices and tensors, and compiles inner-product, norm, and operator notation. Section~\ref{sec:matrix-geometry} covers low-rank matrix geometry required for the analysis of recovery guarantees in Section~\ref{sec:core-det}.

\subsection{Notation}
\label{sec:notation}

Tensors are denoted by calligraphic letters, e.g., $\mathcal X$; matrices by boldface uppercase letters, e.g., $\bm X$; vectors by boldface lowercase letters, e.g., $\bm x$; scalars by lowercase letters, e.g., $x$; and operators by script capitals, e.g., $\mathscr L$. General spaces and sets are denoted by blackboard bold, e.g., $\Hh$, whereas sampling or support sets are denoted by uppercase Greek-letter symbols such as $\Oset$. The order-$d$ tensor space $\Hh=\RR^{m_1\times\cdots\times m_d}$ has index set $\Iset=\prod_{a=1}^d\{1,\ldots,m_a\}$ and dimension $N=\prod_{a=1}^d m_a$. $\mathcal X_{\bm{i}}$ denotes the entry indexed by $\bm{i}=(i_1,\ldots,i_d)\in\Iset$ for $\mathcal X\in\Hh$, and $\vecop(\mathcal X)\in\RR^N$ denotes its vectorization. 

Frobenius, entrywise $\ell_1$, and entrywise $\ell_\infty$ norms are denoted by $\norm{\cdot}_F$, $\norm{\cdot}_1$, and $\norm{\cdot}_\infty$, respectively. $\norm{\bm x}_2$ denotes the Euclidean norm of $\bm x$; $\norm{\bm X}_2$ and $\norm{\bm X}_*$ denote the spectral and nuclear norms of $\bm X$, respectively. $\norm{\mathscr L}_{\op}$ denotes the induced operator norm associated with $\mathscr L$, and $\norm{\mathscr L}_{F\to F}$ denotes its Frobenius-to-Frobenius induced norm, $\mathscr L^*$ its adjoint, $\mathscr L^\dagger$ its Moore--Penrose pseudoinverse. The identity matrix of size $n$ is written $\bm I_n$, with the subscript omitted when the size is clear, and $\1_{\{E\}}$ denotes the indicator of an event $E$. For $\mathcal X,\mathcal Y\in\Hh$, the Euclidean inner product is $\ip{\mathcal X}{\mathcal Y}=\sum_{\bm{i}\in\Iset}\mathcal X_{\bm{i}}\mathcal Y_{\bm{i}}$; and for $\bm X,\bm Y\in\RR^{m\times n}$, the Frobenius inner product is $\ip{\bm X}{\bm Y}=\sum_{i,j}\bm X_{ij}\bm Y_{ij}=\Tr(\bm X^\trans\bm Y)$, where $\bm X^\trans$ denotes the transpose and $\Tr(\cdot)$ denotes the trace.

For an index set $\Oset\subseteq\Iset$, $\mathscr P_\Oset:\Hh\to\Hh$ denotes
the coordinate projection
\[
[\mathscr P_\Oset(\mathcal X)]_{\bm{i}}=
\begin{cases}
\mathcal X_{\bm{i}}, & \bm{i}\in\Oset,\\
0, & \bm{i}\notin\Oset.
\end{cases}
\]
Let $\Oset^c=\Iset\setminus\Oset$ denote its complement. Then $\mathscr P_{\Oset^c}=\Iop-\mathscr P_\Oset$, where $\Iop$ is the identity operator on $\Hh$. For nonzero $\mathcal U,\mathcal V\in\Hh$, $\mathcal U\otimes\mathcal V$ denotes the rank-one linear operator on $\Hh$ defined by $(\mathcal U\otimes\mathcal V)(\mathcal H)=\ip{\mathcal V}{\mathcal H}\mathcal U$ for $\mathcal H\in\Hh$; after vectorization it acts on $\RR^N$ as multiplication by the rank-one matrix $\vecop(\mathcal U)\vecop(\mathcal V)^\trans$, and the symbol $\otimes$ is used only in this operator sense.

Throughout, the kernel lengths are assumed to satisfy $1\le k_a\le m_a$ for every mode $a$, so that the number $K=\prod_a k_a$ of shifts does not exceed the number $N=\prod_a m_a$ of tensor entries; every lifted matrix considered in the paper then has $N$ rows and $K$ columns with $K\le N$.

\subsection{Preliminaries}
\label{sec:matrix-geometry}
Unless otherwise specified, all SVDs are compact. Let $\bm Y=\bm U\bm\Sigma\bm V^\trans$ be the SVD of the rank-$r$ matrix $\bm Y\in\RR^{N\times K}$, so that $\bm U\in\RR^{N\times r}$, $\bm V\in\RR^{K\times r}$, and $\bm\Sigma=\diag(\sigma_1,\ldots,\sigma_r)$, and set $\bm P_U=\bm U\bm U^\trans$ and $\bm P_V=\bm V\bm V^\trans$. The linear space attached to the column and row spaces of $\bm Y$ is
\begin{equation}
\mathbb T_{\bm Y}
=
\left\{
\bm U\bm Q^\trans+\bm P\bm V^\trans:
\bm Q\in\RR^{K\times r},\ \bm P\in\RR^{N\times r}
\right\}.
\label{eq:prelim-tangent}
\end{equation}
The orthogonal projections onto $\mathbb T_{\bm Y}$ and onto its orthogonal complement $\mathbb T_{\bm Y}^\perp$ are
\begin{align}
\mathscr P_{\mathbb T_{\bm Y}}(\bm Z)
&=\bm P_U\bm Z+\bm Z\bm P_V-\bm P_U\bm Z\bm P_V,
\label{eq:prelim-tangent-proj}\\
\mathscr P_{\mathbb T_{\bm Y}^\perp}(\bm Z)
&=(\bm I-\bm P_U)\bm Z(\bm I-\bm P_V).
\label{eq:prelim-tangent-proj-perp}
\end{align}
The set of all sub-gradients of the nuclear norm at a rank-$r$ matrix $\bm Y$ is the subdifferential $\partial\norm{\bm Y}_*$, which has the standard form
\begin{equation}
\partial\norm{\bm Y}_*
=
\left\{
\bm U\bm V^\trans+\bm W:
\mathscr P_{\mathbb T_{\bm Y}}(\bm W)=0,\ \norm{\bm W}_2\le1
\right\}.
\label{eq:prelim-nuclear-subgrad}
\end{equation}

\section{Reflective Convolution Nuclear Norm Minimization}
\label{sec:core-model}

We first construct the reflective lifting underlying RCNNM, then derive its multiplicity-weighted Gram identity and use this structure to formulate the robust tensor completion model with missing observations and sparse corruptions.

\subsection{Reflective Convolution and Lifting}

We first specify the endpoint-nonrepeating reflective boundary extension adopted in RCNNM, which is applied independently along each tensor mode. Given the kernel offsets $\Kset$, the convolution considers indices $j=i_a+t_a$ satisfying $1\le i_a\le m_a$ and $0\le t_a\le k_a-1\le m_a-1$, so that only the range $1\le j\le 2m_a-1$ is relevant. On this interval, we define the one-dimensional reflection map
\[
\phi_a(j)=
\begin{cases}
j, & 1\le j\le m_a,\\
2m_a-j, & m_a<j\le 2m_a-1.
\end{cases}
\]
Within this range, the reflected index sequence reads $1,2,\ldots,m_a-1,m_a,m_a-1,\ldots,2,1$. For a multi-index, the composite reflection map is defined as
\[
\phi(\bm{i}+\bm{t})=
\bigl(\phi_1(i_1+t_1),\ldots,\phi_d(i_d+t_d)\bigr).
\]

Let $\mathcal G=(g_{\bm{t}})_{\bm{t}\in\Kset}\in\RR^{k_1\times\cdots\times k_d}$ be a kernel tensor indexed by the offsets in $\Kset$. We adopt $\circledast$ to denote the convolution operator, which employs the endpoint-nonrepeating reflective boundary rule defined above rather than the circular continuation used in CNNM~\cite{liu2023cnnm}. Specifically, for $\mathcal X\in\Hh$, $\bm{i}\in\Iset$ and kernel tensor $\mathcal G$, we define
\begin{equation}
[\mathcal X\circledast\mathcal G]_{\bm{i}}
=
\sum_{\bm{t}\in\Kset}g_{\bm{t}}\,\mathcal X_{\phi(\bm{i}+\bm{t})}.
\label{eq:core-ref-convolution}
\end{equation}

The one-dimensional case yields an explicit matrix representation. In this one-dimensional setting, we write $\phi=\phi_1$. For
$\bm{x}=[x_1,\ldots,x_m]^\trans\in\RR^m$ and a length-$k$ kernel
$\bm{g}=[g_1,\ldots,g_k]^\trans$, where $g_s$ corresponds to offset $s-1$, we define
the $m\times k$ reflective convolution matrix entrywise by
\begin{equation}
\bigl[\bm{A}_k^{\mathrm{ref}}(\bm{x})\bigr]_{i,s}
=
x_{\phi(i+s-1)},
\qquad
1\le i\le m,\quad 1\le s\le k.
\label{eq:core-ref-conv-matrix-1d}
\end{equation}
Then
\begin{equation}
\bm{x}\circledast\bm{g}=\bm{A}_k^{\mathrm{ref}}(\bm{x})\bm{g}.
\label{eq:core-ref-conv-matrix-action-1d}
\end{equation}
For example, when $m=5$ and $k=3$,
\begin{equation}
\bm{A}_3^{\mathrm{ref}}(\bm{x})=
\begin{bmatrix}
x_1 & x_2 & x_3\\
x_2 & x_3 & x_4\\
x_3 & x_4 & x_5\\
x_4 & x_5 & x_4\\
x_5 & x_4 & x_3
\end{bmatrix}.
\label{eq:core-ref-conv-matrix-example}
\end{equation}
The first three rows correspond to standard consecutive windows, whereas the last two extend beyond the right boundary using the prescribed reflection rule.

The same construction extends directly to tensors. For
$\bm{t}\in\Kset$ and $\bm{i}\in\Iset$, we define the reflected-shift operator
$\mathscr R_{\bm{t}}:\Hh\to\Hh$ by
\begin{equation}
[\mathscr R_{\bm{t}}(\mathcal X)]_{\bm{i}}=\mathcal X_{\phi(\bm{i}+\bm{t})}.
\label{eq:core-Rt}
\end{equation}
We enumerate the offsets as $\bm{t}^{(1)},\ldots,\bm{t}^{(K)}$ following the ordering used to
vectorize the kernel tensor. The reflective convolution lifting
$\A_{\mathrm{ref}}(\mathcal X)\in\RR^{N\times K}$ is then given by
\begin{equation}
\A_{\mathrm{ref}}(\mathcal X)=
\left[
\vecop(\mathscr R_{\bm{t}^{(1)}}(\mathcal X)),\ldots,
\vecop(\mathscr R_{\bm{t}^{(K)}}(\mathcal X))
\right].
\label{eq:core-A}
\end{equation}
By construction,
\begin{equation}
\vecop(\mathcal X\circledast\mathcal G)
=
\A_{\mathrm{ref}}(\mathcal X)\vecop(\mathcal G).
\label{eq:core-convolution-action}
\end{equation}
Thus $\A_{\mathrm{ref}}(\mathcal X)$ is the convolution matrix induced by the reflective
boundary rule. It takes the generic $N\times K$ form in
\eqref{eq:intro-lift}, where translations are generated via reflection. For notational brevity, we write $\A:=\A_{\mathrm{ref}}$ from Section~\ref{sec:rcnnm-model} onward whenever the reflective lifting is meant; $\A_{\mathrm{circ}}$ is retained for the circular counterpart.

Observe that $\rank(\A_{\mathrm{ref}}(\mathcal X))$ quantifies the reflective convolutional low-rankness of $\mathcal X$. If $\rank(\A_{\mathrm{ref}}(\mathcal X))<K$, then there exists a nonzero kernel tensor $\mathcal G$ such that
\[
\A_{\mathrm{ref}}(\mathcal X)\vecop(\mathcal G)=0,
\]
or equivalently $\mathcal X\circledast\mathcal G=0$. Such a kernel is called a finite
reflected-boundary annihilating filter. Since the zero offset belongs to
$\Kset$, one column of $\A_{\mathrm{ref}}(\mathcal X)$ equals $\vecop(\mathcal X)$; hence the lifting map
$\A_{\mathrm{ref}}:\Hh\to\RR^{N\times K}$ is injective.

\paragraph*{Weighted Gram identity}
For circular lifting, every entry appears exactly $K$ times in the lifted matrix and
$\A_{\mathrm{circ}}^*\A_{\mathrm{circ}}=K\Iop$.  Reflection breaks this uniform multiplicity near the boundary.  For $\bm{u}\in\Iset$, let $\mathcal E_{\bm{u}}\in\Hh$ be the canonical
basis tensor:
$(\mathcal E_{\bm{u}})_{\bm{u}}=1$ and $(\mathcal E_{\bm{u}})_{\bm{v}}=0$ for $\bm{v}\ne\bm{u}$.  Define the corresponding
reflection atom and multiplicity by
\[
\bm{A}_{\bm{u}}=\A_{\mathrm{ref}}(\mathcal E_{\bm{u}}),\qquad d_{\bm{u}}=\norm{\bm{A}_{\bm{u}}} _F^2.
\]
Every entry of $\bm{A}_{\bm{u}}$ is zero or one, so $d_{\bm{u}}$ counts the number of lifted
positions that contain the original coordinate $\bm{u}$.  Equivalently,
\begin{equation}
d_{\bm{u}}=
\#\bigl\{
(\bm{i},\bm{t})\in\Iset\times\Kset:\phi(\bm{i}+\bm{t})=\bm{u}
\bigr\}.
\label{eq:core-multiplicity-count}
\end{equation}
Accordingly, $\mathcal D=(d_{\bm{u}})_{\bm{u}\in\Iset}$ is the multiplicity tensor on the original domain.

\begin{proposition}[Weighted Gram identity]
\label{thm:core-gram}
For every $\mathcal X\in\Hh$,
\begin{equation}
\A_{\mathrm{ref}}^*\A_{\mathrm{ref}}(\mathcal X)=\mathcal D\odot \mathcal X.
\label{eq:core-gram}
\end{equation}
Consequently,
\begin{equation}
d_{\min}\norm{\mathcal X}_F^2
\le
\norm{\A_{\mathrm{ref}}(\mathcal X)}_F^2
\le
d_{\max}\norm{\mathcal X}_F^2,
\label{eq:core-weighted}
\end{equation}
where $d_{\min}=\Min_{\bm{u}}d_{\bm{u}}>0$ and $d_{\max}=\Max_{\bm{u}}d_{\bm{u}}$.
Thus, $\A_{\mathrm{ref}}$ is injective and $\norm{\mathcal X}_{\mathrm{RCNN}}:=\norm{\A_{\mathrm{ref}}(\mathcal X)}_*$ is a norm.
\end{proposition}
The proof of Proposition~\ref{thm:core-gram} is given in \ref{app:pub-geometry}.
\begin{remark}
\label{rem:gram-supplement}
Let $\A_{\mathrm{ref}}^*\A_{\mathrm{ref}}=\Dop$, where $\Dop$ denotes the pointwise multiplication operator $\mathcal X\mapsto \mathcal D\odot \mathcal X$.
Unlike circular convolutional lifting satisfying $\A_{\mathrm{circ}}^*\A_{\mathrm{circ}}=K\Iop$ with uniform constant weights, reflective boundary extension yields spatially nonuniform multiplicities $d_{\bm u}$, so that $\Dop$ is generally not a scalar multiple of the identity operator.
This diagonal operator structure gives rise to the two-sided norm bound~\eqref{eq:core-weighted} and consequently the injectivity of $\A_{\mathrm{ref}}$.
Furthermore, the entrywise action of $\Dop$ enables the closed-form per-coordinate update in the ADMM solver and avoids a large linear-system solve; see Section~\ref{sec:optimization}.
\end{remark}

\subsection{RCNNM Model}
\label{sec:rcnnm-model}

Let $\Oset\subseteq\Iset$ be the observed index set. We assume that the observed tensor is generated by
\begin{equation}
\mathscr P_\Oset(\mathcal L_0+\mathcal S_0-\mathcal M)=0,
\qquad
\mathscr P_{\Oset^c}(\mathcal S_0)=0,
\label{eq:core-observation-model}
\end{equation}
where the clean component $\mathcal L_0$ has low reflective convolutional rank, i.e.,
$\rank(\A(\mathcal L_0))\ll K$, and $\mathcal S_0$ represents sparse gross errors on the observed entries. The support restriction in \eqref{eq:core-observation-model} is natural because corruption on an unobserved coordinate cannot be distinguished from a change in the unobserved clean tensor.

A direct formulation would penalize the reflective convolutional rank
$\rank(\A(\mathcal L))$ together with the number of nonzero corruption entries:
\begin{equation}
\Min_{\mathcal L,\mathcal S}\quad
\rank(\A(\mathcal L))+\lambda\norm{\mathcal S}_0
\quad
\st\quad
\left\{
\begin{aligned}
\mathscr P_\Oset(\mathcal L+\mathcal S-\mathcal M)&=0,\\
\mathscr P_{\Oset^c}(\mathcal S)&=0.
\end{aligned}
\right.
\label{eq:core-nonconvex-model}
\end{equation}
where $\norm{\mathcal S}_0$ counts the nonzero entries of $\mathcal S$. Both terms in \eqref{eq:core-nonconvex-model} are nonconvex. To obtain a tractable formulation, we replace the rank of the lifted matrix by its nuclear norm, the standard convex surrogate for matrix rank \cite{candes2009matrix,recht2011simpler}, and define the reflective convolution nuclear norm as
\begin{equation}
\norm{\mathcal L}_{\mathrm{RCNN}}
:=
\norm{\A(\mathcal L)}_*.
\label{eq:core-rcnn-norm}
\end{equation}
The sparse $\ell_0$ term is relaxed in the same way by the entrywise $\ell_1$ norm. Proposition~\ref{thm:core-gram} shows that $\norm{\cdot}_{\mathrm{RCNN}}$ is a genuine norm on $\Hh$. The resulting convex model is
\begin{equation}
\Min_{\mathcal L,\mathcal S}\quad
\norm{\A(\mathcal L)}_*+\lambda\norm{\mathcal S}_1
\qquad
\st\quad
\left\{
\begin{aligned}
\mathscr P_\Oset(\mathcal L+\mathcal S-\mathcal M)&=0,\\
\mathscr P_{\Oset^c}(\mathcal S)&=0.
\end{aligned}
\right.
\label{eq:core-rcnnm}
\end{equation}
where $\lambda>0$ balances reflective convolutional low-rankness and sparsity. The first constraint enforces consistency with the observations, whereas the second restricts the sparse component to observed coordinates.

Our model \eqref{eq:core-rcnnm} is equivalent to the following reflective-convolution tensor completion problem when there is no sparse corruption, i.e., $\mathcal S=0$:
\begin{equation}
\Min_{\mathcal L}\quad
\norm{\A(\mathcal L)}_*
\qquad
\st\quad
\mathscr P_\Oset(\mathcal L-\mathcal M)=0.
\label{eq:core-special-completion}
\end{equation}
When all entries are observed, i.e., $\Oset=\Iset$, model~\eqref{eq:core-rcnnm} reduces to a tensor robust principal component analysis (TRPCA)-type formulation:
\begin{equation}
\Min_{\mathcal L,\mathcal S}\quad
\norm{\A(\mathcal L)}_*+\lambda\norm{\mathcal S}_1
\qquad
\st\quad
\mathcal L+\mathcal S-\mathcal M=0.
\label{eq:core-special-trpca}
\end{equation}
It retains the standard low-rank--plus--sparse decomposition structure of TRPCA \cite{lu2020trpca}, with the conventional tensor nuclear norm replaced by the reflective convolution nuclear norm.

To account for small dense noise in addition to sparse gross corruption, we consider the perturbed observation model
\begin{equation}
\begin{aligned}
\mathscr P_\Oset(\mathcal M)
&=
\mathscr P_\Oset(\mathcal L_0+\mathcal S_0+\mathcal N),\\
\mathscr P_{\Oset^c}(\mathcal S_0)&=0,
\qquad
\norm{\mathscr P_\Oset(\mathcal N)}_F\le\eps,
\end{aligned}
\label{eq:core-noisy-observation}
\end{equation}
where $\mathcal N$ denotes dense noise. Following the stable low-rank--sparse recovery formulation \cite{zhou2010stable}, we use the noise-tolerant RCNNM model
\begin{equation}
\Min_{\mathcal L,\mathcal S}\quad
\norm{\A(\mathcal L)}_*+\lambda\norm{\mathcal S}_1
\qquad
\st\quad
\left\{
\begin{aligned}
\norm{\mathscr P_\Oset(\mathcal L+\mathcal S-\mathcal M)}_F&\le\eps,\\
\mathscr P_{\Oset^c}(\mathcal S)&=0.
\end{aligned}
\right.
\label{eq:core-noisy-constraint}
\end{equation}
This formulation relaxes the exact data-consistency constraint in \eqref{eq:core-rcnnm} while retaining the support constraint on the sparse component. When $\eps=0$, \eqref{eq:core-noisy-constraint} reduces to the noiseless RCNNM model \eqref{eq:core-rcnnm}. Section~\ref{sec:core-det} establishes exact recovery for \eqref{eq:core-rcnnm} and stable recovery for \eqref{eq:core-noisy-constraint}.

\section{Recovery Guarantees}
\label{sec:core-det}

We next establish recovery guarantees for the RCNNM model \eqref{eq:core-rcnnm} under the observation models in Section~\ref{sec:rcnnm-model}. We first introduce the geometric quantities induced by reflective lifting, then establish exact recovery under random sampling and random-sign sparse corruption, and finally derive a stability guarantee in the presence of bounded dense measurement noise.

\subsection{Geometric Conditions}

Let
\[
\bm{X}_0=\A(\mathcal L_0)=\bm{U}\bm{\Sigma}\bm{V}^\trans
\]
be the compact SVD of the lifted clean tensor, with
$\bm{U}\in\RR^{N\times r}$ and $\bm{V}\in\RR^{K\times r}$.  Let
$\Tsp=\mathbb T_{\bm{X}_0}$ be the linear space in
\eqref{eq:prelim-tangent}, with orthogonal projections $\mathscr P_{\Tsp}$ and
$\mathscr P_{\Tsp^\perp}$.  Because the optimization variable is the original tensor
rather than the lifted matrix, we also use its preimage under $\A$
\begin{equation}
\V=\{\mathcal H\in\Hh:\A(\mathcal H)\in\Tsp\},
\qquad s=\Dim(\V).
\label{eq:core-V}
\end{equation}
Let $\mathscr P_\V$ and $\mathscr P_{\V^\perp}$ denote the orthogonal
projections onto $\V$ and $\V^\perp$, respectively, and let
$\mathscr I_\V$ denote the identity operator on $\V$.

The coordinate coherence of $\V$ is
\begin{equation}
\mu_\V=
\frac{N}{s}
\Max_{\bm{u}\in\Iset}\norm{\mathscr P_\V (\mathcal E_{\bm{u}})}_F^2.
\label{eq:core-muv}
\end{equation}
The reflective lifting is not a scaled isometry, so a perturbation in
$\V^\perp$ may acquire a component in $\Tsp$ after lifting.  We measure the
remaining component in $\Tsp^\perp$ by
\begin{equation}
\beta_{\Tsp}
=
\Inf_{\substack{\mathcal H\in\V^\perp\\\norm{\mathcal H}_F=1}}
\norm{\mathscr P_{\Tsp^\perp}(\A(\mathcal H))}_F.
\label{eq:core-beta}
\end{equation}
Set
\[
\mathcal B_0=\A^*(\bm{U}\bm{V}^\trans),
\qquad
\KT=\A^*|_{\Tsp^\perp}:\Tsp^\perp\to\V^\perp.
\]
When $\beta_{\Tsp}>0$, $\KT$ is onto $\V^\perp$ and
$\norm{\KT^\dagger}_{F\to F}\le\beta_{\Tsp}^{-1}$.  Finally define
\begin{equation}
\mu_b=
\Max\left\{
1,\,
\frac{NK}{rd_{\max}}\norm{\mathscr P_\V (\mathcal B_0)}_\infty^2
\right\},
\label{eq:core-mub}
\end{equation}
and
\begin{equation}
\chi_{\Tsp}=\frac{\beta_{\Tsp}^2}{d_{\max}},
\qquad
c_{\mathrm{ref}}=
\norm{\KT^\dagger(\mathscr P_{\V^\perp}(\mathcal B_0))}_2.
\label{eq:core-geometry}
\end{equation}
The quantity $\beta_{\Tsp}$ measures the smallest $\Tsp^\perp$ component produced by lifting a unit-Frobenius-norm direction in $\V^\perp$. Thus $\beta_{\Tsp}>0$
rules out nonzero directions in $\V^\perp$ whose lifts lie entirely in $\Tsp$.
Since $\norm{\A(\mathcal H)}_F\le\sqrt{d_{\max}}\norm{\mathcal H}_F$,
$\chi_{\Tsp}=\beta_{\Tsp}^2/d_{\max}\in[0,1]$ is a multiplicity-normalized measure of
this separation; Proposition~\ref{prop:pub-angle} gives the corresponding
angle bound.

The second compatibility condition is measured by $c_{\mathrm{ref}}$.  The
pullback $\mathcal B_0=\A^*(\bm{U}\bm{V}^\trans)$ need not lie in $\V$, so
$\mathscr P_{\V^\perp}(\mathcal B_0)$ must be cancelled by a correction in $\Tsp^\perp$.  The
minimum-Frobenius-norm correction is
$-\KT^\dagger(\mathscr P_{\V^\perp}(\mathcal B_0))$, and $c_{\mathrm{ref}}$ is its spectral norm.
In particular, $c_{\mathrm{ref}}=0$ when $\mathscr P_{\V^\perp}(\mathcal B_0)=0$.

\subsection{Exact Recovery}
\label{sec:core-prob}

Assume that each coordinate is observed independently with probability
$p$ and that an observed entry is corrupted with probability $\gamma$.
Let $\Gset$ be the corruption support,
$\Lset=\Oset\setminus\Gset$ the clean observed set, and
\[
\mathcal S_0^{\mathrm{sgn}}=\sgn(\mathcal S_0),
\]
with $\mathcal S_0^{\mathrm{sgn}}$ extended by zero outside $\Gset$.  With
$q=p(1-\gamma)$,
\begin{align*}
\Pp(\bm{u}\in\Lset)&=q,\\
\Pp(\bm{u}\in\Gset)&=p\gamma,\\
\Pp(\bm{u}\in\Oset^c)&=1-p.
\end{align*}
Conditional on $\Gset$, the corruption signs are independent Rademacher
variables.  The nonzero magnitudes are arbitrary deterministic values that do
not depend on the observation or corruption masks.

\begin{theorem}[Exact recovery]
\label{thm:core-main}
Under the random observation and corruption model described above, assume
$1\le r<K$, $s>0$, and
\[
\mu_b\le\mu_0,\qquad
\chi_{\Tsp}\ge\chi_0,\qquad
c_{\mathrm{ref}}\le c_0,
\]
for fixed constants
$\mu_0\ge1$, $\chi_0\in(0,1]$, and $c_0\in[0,1/2)$.
Let $c_\lambda>0$ be fixed and define
\begin{equation}
\ell_{NK}:=\Log(e(N+K)),
\qquad
\lambda:=c_\lambda\frac{\beta_{\Tsp}}{\sqrt{pN}}.
\label{eq:core-lambda}
\end{equation}
There exist constants
$C_s,C_m,c_r,C_1,C_2>0$ and
$c_\gamma\in(0,1/2]$, depending only on the fixed bounds
$\mu_0,\chi_0,c_0$, on $c_\lambda$, and on universal concentration
constants. If
\begin{equation}
qN\ge C_s\mu_\V s\ell_{NK}^2,
\qquad
qN\ge C_m\ell_{NK}^2,
\label{eq:core-sampling}
\end{equation}
\begin{equation}
r
\le
c_r
\frac{\chi_{\Tsp}q^2K}
{\mu_b p\ell_{NK}^2},
\qquad
\gamma\le c_\gamma,
\label{eq:core-rank}
\end{equation}
then the RCNNM model \eqref{eq:core-rcnnm} exactly recovers $(\mathcal L_0,\mathcal S_0)$ with probability at least
\begin{equation}
1-C_1(N+K)^{-C_2}.
\label{eq:core-probability}
\end{equation}
\end{theorem}

The sampling requirement depends on both the dimension and the coordinate coherence of $\V$. A smaller $\chi_{\Tsp}$ lowers the admissible rank, whereas
$c_{\mathrm{ref}}<1/2$ controls the spectral-norm correction required by the
reflective pullback.  A geometric interpretation of $\chi_{\Tsp}$ is given in \ref{app:pub-angle-sec}, and the bound for $c_{\mathrm{ref}}$ is derived in \ref{app:pub-reference}.  The quantity $\mu_b$ remains instance dependent in our analysis. The scaling
\eqref{eq:core-lambda} is used to state the recovery theorem and depends on
the unknown subspace geometry through $\beta_{\Tsp}$; it is not used as the
practical parameter-selection rule in the image experiments. In practice,
$\lambda$ is selected on held-out validation data as described in
Section~\ref{sec:parameter-sensitivity}.

The proof proceeds through an approximate dual certificate. The golfing construction and sparse-certificate estimates are developed in \ref{app:pub-golf} and \ref{app:pub-sparse}, respectively, while the final proof of Theorem~\ref{thm:core-main} is completed in \ref{app:pub-closure}.

\subsection{Stability under Dense Noise}

Theorem~\ref{thm:core-main} treats the noiseless RCNNM model \eqref{eq:core-rcnnm}. We now consider the perturbed observation model \eqref{eq:core-noisy-observation} and the noise-tolerant RCNNM model \eqref{eq:core-noisy-constraint}. The following result shows that the reconstruction error remains proportional to the dense-noise level.

\begin{corollary}[Stable recovery]
\label{cor:core-stability}
Suppose that all assumptions and recovery conditions in Theorem~\ref{thm:core-main} hold. Let
$(\widehat{\mathcal L},\widehat{\mathcal S})$ be any minimizer of \eqref{eq:core-noisy-constraint}. Then, with probability at least $1-C_1(N+K)^{-C_2}$,
\begin{equation}
\left(
\norm{\widehat{\mathcal L}-\mathcal L_0}_F^2
+
\norm{\widehat{\mathcal S}-\mathcal S_0}_F^2
\right)^{1/2}
\le
C_{\mathrm{stab}}\eps,
\label{eq:core-stability-bound}
\end{equation}
where $C_{\mathrm{stab}}>0$ depends only on the sampling and recovery parameters appearing in Theorem~\ref{thm:core-main}.
\end{corollary}

The proof is given in \ref{app:stable-recovery-bound}. Thus, the recovery error scales linearly with the dense-noise level $\eps$.

\section{Optimization}
\label{sec:optimization}

We next solve the noiseless RCNNM model \eqref{eq:core-rcnnm}. Following the standard variable-splitting strategy used in nuclear-norm tensor recovery \cite{jiang2019rtc,li2026cpcp}, we introduce an auxiliary matrix for the reflective lifting and apply a two-block ADMM. This separation leaves the nonsmooth nuclear and $\ell_1$ terms in simple proximal subproblems, while the reflective Gram identity $\A^*\A=\Dop$ gives a closed-form update for the clean tensor.

\subsection{ADMM Formulation and Variable Updates}

Introducing the auxiliary variable $\bm Z=\A(\mathcal L)$ gives the equivalent constrained problem
\begin{equation}
\Min_{\mathcal L,\mathcal S,\bm Z}\quad
\norm{\bm Z}_*+\lambda\norm{\mathcal S}_1
\qquad
\st\quad
\left\{
\begin{aligned}
\A(\mathcal L)-\bm Z&=0,\\
\mathscr P_\Oset(\mathcal L+\mathcal S-\mathcal M)&=0,\\
\mathscr P_{\Oset^c}(\mathcal S)&=0.
\end{aligned}
\right.
\label{eq:core-split}
\end{equation}
Here $\mathcal L$ is the clean tensor and $\mathcal S$ is the sparse corruption tensor. Let $\bm Y_1\in\RR^{N\times K}$ and $\mathcal Y_2\in\Hh$ be the Lagrange multipliers associated with the first two equality constraints, respectively, and let $\tau_1,\tau_2>0$ be the corresponding penalty parameters. The support constraint $\mathscr P_{\Oset^c}(\mathcal S)=0$ is enforced directly in the $\mathcal S$-subproblem. The augmented Lagrangian of \eqref{eq:core-split} is
\begin{equation}
\begin{aligned}
\mathfrak L={}&
\norm{\bm Z}_*+\lambda\norm{\mathcal S}_1
+\ip{\bm Y_1}{\A(\mathcal L)-\bm Z}
+\frac{\tau_1}{2}\norm{\A(\mathcal L)-\bm Z}_F^2\\
&+
\ip{\mathcal Y_2}{\mathscr P_\Oset(\mathcal L+\mathcal S-\mathcal M)}
+\frac{\tau_2}{2}
\norm{\mathscr P_\Oset(\mathcal L+\mathcal S-\mathcal M)}_F^2.
\end{aligned}
\label{eq:core-augLag}
\end{equation}
At iteration $k$, ADMM alternates between the clean tensor $\mathcal L$ and the joint block $(\bm Z,\mathcal S)$, followed by multiplier and penalty updates.

\paragraph*{Update of $\mathcal L$}
With $(\bm Z^k,\mathcal S^k)$ and the multipliers fixed, the clean tensor is updated by solving
\begin{equation}
\begin{aligned}
\mathcal L^{k+1}
=
\argmin_{\mathcal L}\quad&
\frac{\tau_1^k}{2}
\norm{
\A(\mathcal L)-\bm Z^k+\frac{\bm Y_1^k}{\tau_1^k}
}_F^2\\
&+
\frac{\tau_2^k}{2}
\norm{
\mathscr P_\Oset\!\left(
\mathcal L+\mathcal S^k-\mathcal M
+\frac{\mathcal Y_2^k}{\tau_2^k}
\right)
}_F^2 .
\end{aligned}
\label{eq:core-L-subproblem}
\end{equation}
Using the reflective Gram identity $\A^*\A=\Dop$, the solution is
\begin{equation}
\mathcal L^{k+1}
=
(\tau_1^k\Dop+\tau_2^k\mathscr P_\Oset)^{-1}
(\mathcal Q^k),
\label{eq:core-L-update}
\end{equation}
where
\begin{equation}
\mathcal Q^k=
\A^*(\tau_1^k\bm Z^k-\bm Y_1^k)
+\mathscr P_\Oset\!\left[
\tau_2^k(\mathcal M-\mathcal S^k)-\mathcal Y_2^k
\right].
\label{eq:core-Q-update}
\end{equation}
The inverse operator in \eqref{eq:core-L-update} is an entry-wise tensor division, so the update can be evaluated directly.

\paragraph*{Update of $\mathcal S$}
For fixed $\mathcal L^{k+1}$, completing the square and using
$\mathscr P_{\Oset^c}(\mathcal S)=0$ reduce the sparse-component subproblem to
\begin{equation}
\mathcal S^{k+1}
=
\argmin_{\mathscr P_{\Oset^c}(\mathcal S)=0}\;
\lambda\norm{\mathcal S}_1
+
\frac{\tau_2^k}{2}
\norm{
\mathcal S-
\mathscr P_\Oset\!\left(
\mathcal M-\mathcal L^{k+1}
-\frac{\mathcal Y_2^k}{\tau_2^k}
\right)
}_F^2.
\label{eq:core-S-subproblem}
\end{equation}
Let
\[
\mathscr S_\eta(x)
:=
\sgn(x)\Max\{|x|-\eta,0\}.
\]
denote the entrywise soft-thresholding operator. Hence,
\begin{equation}
\mathcal S^{k+1}
=
\mathscr S_{\lambda/\tau_2^k}
\!\left[
\mathscr P_\Oset\!\left(
\mathcal M-\mathcal L^{k+1}
-\frac{\mathcal Y_2^k}{\tau_2^k}
\right)
\right].
\label{eq:core-S-update}
\end{equation}

\paragraph*{Update of $\bm Z$}
Likewise, completing the square gives the nuclear-norm proximal subproblem
\begin{equation}
\bm Z^{k+1}
=
\argmin_{\bm Z}\;
\norm{\bm Z}_*
+
\frac{\tau_1^k}{2}
\norm{
\bm Z-
\left(
\A(\mathcal L^{k+1})
+\frac{\bm Y_1^k}{\tau_1^k}
\right)
}_F^2.
\label{eq:core-Z-subproblem}
\end{equation}
For a matrix
$\bm H=\bm U\diag(\sigma_1,\ldots,\sigma_r)\bm V^\trans$
with compact SVD, the singular-value-thresholding operator at level
$\tau>0$ is defined by
\begin{equation}
\begin{aligned}
\SVT_\tau(\bm H)
&=
\bm U\,
\diag\!\bigl((\sigma_i-\tau)_+\bigr)_{i=1}^{r}
\,\bm V^\trans,\\
(x)_+&:=\Max\{x,0\}.
\end{aligned}
\label{eq:core-svt-def}
\end{equation}
Thus, SVT applies scalar soft thresholding to the singular values and is
the proximal mapping of the nuclear norm. Therefore,
\begin{equation}
\bm Z^{k+1}
=
\SVT_{1/\tau_1^k}
\left(
\A(\mathcal L^{k+1})
+\frac{\bm Y_1^k}{\tau_1^k}
\right).
\label{eq:core-Z-update}
\end{equation}
Once $\mathcal L^{k+1}$ is fixed, the $\mathcal S$- and $\bm Z$-subproblems are separable and hence form the second ADMM block.

\paragraph*{Updates of $\bm Y_1$, $\mathcal Y_2$, $\tau_1$, and $\tau_2$}
The Lagrange multipliers are updated by
\begin{align}
\bm Y_1^{k+1}
&=
\bm Y_1^k
+\tau_1^k\bigl(
\A(\mathcal L^{k+1})-\bm Z^{k+1}
\bigr),
\label{eq:core-Y1-update}\\
\mathcal Y_2^{k+1}
&=
\mathcal Y_2^k
+\tau_2^k
\mathscr P_\Oset(
\mathcal L^{k+1}+\mathcal S^{k+1}-\mathcal M).
\label{eq:core-Y2-update}
\end{align}
The penalty parameters are updated as
\begin{equation}
\tau_1^{k+1}=\Min\{\varrho\tau_1^k,\tau_{\max}\},
\qquad
\tau_2^{k+1}=\Min\{\varrho\tau_2^k,\tau_{\max}\}.
\label{eq:core-penalty-update}
\end{equation}
where $\varrho>1$ is the penalty factor and $\tau_{\max}<\infty$ is the upper bound.

The iteration is monitored by the normalized primal residuals
\begin{equation}
r_1^k=
\frac{\norm{\A(\mathcal L^k)-\bm Z^k}_F}
{\norm{\A(\mathcal L^k)}_F},
\qquad
r_2^k=
\frac{\norm{\mathscr P_\Oset(\mathcal L^k+\mathcal S^k-\mathcal M)}_F}
{1+\norm{\mathcal M}_F}.
\label{eq:stopping-residuals}
\end{equation}
For completeness, we set $r_1^k=+\infty$ whenever
$\norm{\A(\mathcal L^k)}_F=0$. The iteration terminates when
$\Max\{r_1^k,r_2^k\}\le\varepsilon_{\mathrm{tol}}$ or the prescribed maximum number of iterations is reached.

\subsection{Gram-Based SVT and Algorithm}

The definition in \eqref{eq:core-svt-def} evaluates SVT through a direct SVD. In the RCNNM model \eqref{eq:core-rcnnm}, however, $\A(\mathcal L)\in\RR^{N\times K}$ with $N\ge K$, so the same operator can be evaluated through the smaller $K\times K$ right Gram matrix. Define
\begin{equation}
\bm H^k
=
\A(\mathcal L^{k+1})
+\frac{\bm Y_1^k}{\tau_1^k},
\qquad
\bm G^k=(\bm H^k)^\trans\bm H^k.
\label{eq:core-gram-def}
\end{equation}
If $\bm H^k=\bm U\bm\Sigma\bm V^\trans$, then
$\bm G^k=\bm V\bm\Sigma^2\bm V^\trans$. Thus, the eigenvectors of
$\bm G^k$ are the right singular vectors of $\bm H^k$, and its eigenvalues
are the squared singular values. Let
\[
\bm G^k
=
\bm V\diag(\xi_1,\ldots,\xi_K)\bm V^\trans,
\qquad
\sigma_i=\sqrt{\Max\{\xi_i,0\}}.
\]
For a threshold $\tau>0$, define the retained-index set
\[
\Jtau=\{i:\sigma_i>\tau\},\qquad
\bm V_\tau=[\bm v_i]_{i\in\Jtau},\qquad
\bm\Sigma_\tau=\diag(\sigma_i)_{i\in\Jtau}.
\]
Only these retained components contribute to the thresholded matrix. Hence,
\begin{equation}
\SVT_\tau(\bm H^k)
=
\bm H^k\bm V_\tau
\bigl(\bm I-\tau\bm\Sigma_\tau^{-1}\bigr)
\bm V_\tau^\trans,
\label{eq:gram-svt}
\end{equation}
with the convention that the result is zero when $\Jtau=\varnothing$.

\begin{proposition}[Gram form of singular-value thresholding]
\label{prop:gram-svt}
For every $\bm H\in\RR^{N\times K}$ with $N\ge K$, the Gram construction in \eqref{eq:gram-svt} gives the same singular-value-thresholding operator as a direct SVD in exact arithmetic.
\end{proposition}

The proof is given in \ref{app:gram-svt-proof}. As with other Gram-based factorizations, forming $\bm H^\trans\bm H$ squares the condition number; a direct SVD can therefore be preferable when very small singular values must be resolved accurately.

Algorithm~\ref{alg:rcnnm-admm} summarizes the complete iteration.

\begin{algorithm}[!t]
\caption{Two-block ADMM for RCNNM \eqref{eq:core-rcnnm}}
\label{alg:rcnnm-admm}
\small
\KwInput{Observed tensor $\mathcal M$, index set $\Oset$, parameter $\lambda$, reflective operator $\A$ and multiplicity operator $\Dop$, initial penalty parameters $\tau_1^0,\tau_2^0$, penalty factor $\varrho$, upper bound $\tau_{\max}$, stopping tolerance $\varepsilon_{\mathrm{tol}}$, and maximum iteration $K_{\max}$.}
Initialize $\mathcal L^0,\bm Z^0,\bm Y_1^0,\mathcal Y_2^0$ and $\mathcal S^0$ with $\mathscr P_{\Oset^c}(\mathcal S^0)=0$\;
Set $r_1^0=r_2^0=+\infty$ and $k=0$\;
\While{$\Max\{r_1^k,r_2^k\}>\varepsilon_{\mathrm{tol}}$ and $k<K_{\max}$}{
Update $\mathcal L^{k+1}$ by \eqref{eq:core-L-update}\;
Update $\mathcal S^{k+1}$ by \eqref{eq:core-S-update}\;
Form $\bm H^k$ and $\bm G^k$ by \eqref{eq:core-gram-def}\;
Update $\bm Z^{k+1}$ by \eqref{eq:gram-svt} with $\tau=1/\tau_1^k$\;
Update $\bm Y_1^{k+1}$ and $\mathcal Y_2^{k+1}$ by
\eqref{eq:core-Y1-update}--\eqref{eq:core-Y2-update}\;
Update $\tau_1^{k+1}$ and $\tau_2^{k+1}$ by \eqref{eq:core-penalty-update}\;
Compute $r_1^{k+1}$ and $r_2^{k+1}$ by \eqref{eq:stopping-residuals}\;
Set $k\leftarrow k+1$\;
}
Set $\widehat{\mathcal L}=\mathcal L^k$ and $\widehat{\mathcal S}=\mathcal S^k$\;
\textbf{Output:} Recovered clean tensor $\widehat{\mathcal L}$ and sparse tensor $\widehat{\mathcal S}$\;
\end{algorithm}

\begin{figure*}[!t]
\centering
\includegraphics[width=\textwidth]{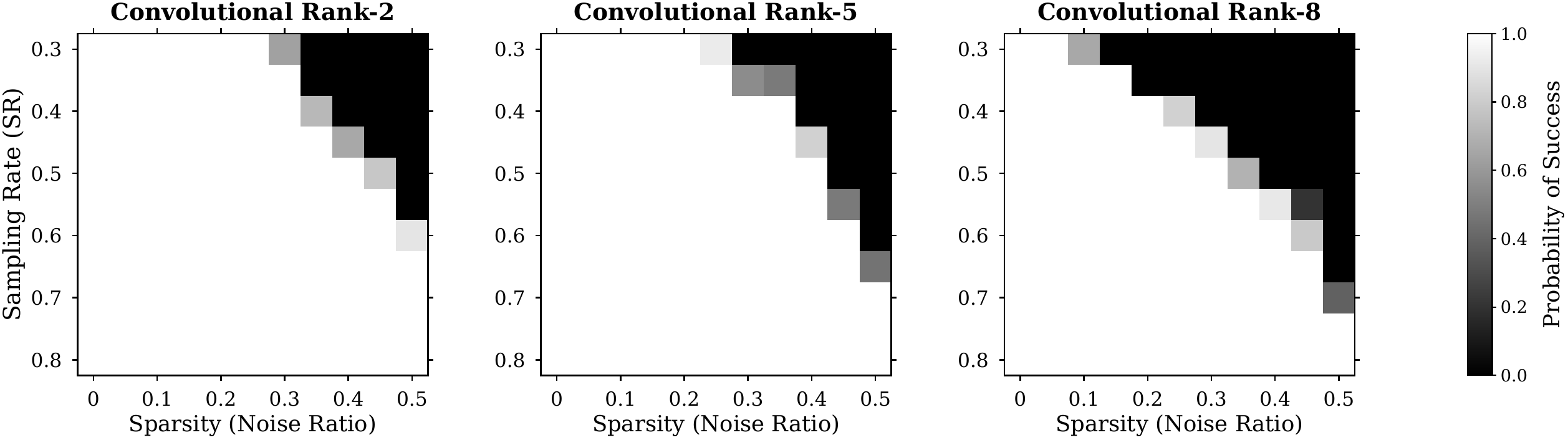}
\caption{Empirical phase-transition results.  Each grid value is the fraction
of successful recoveries over 10 independent trials as a function of lifted
rank $r$ and corruption probability $\gamma$, shown for different sampling
probabilities $p$.  A trial is counted as successful when
$\mathrm{RSE}<10^{-2}$.}
\label{fig:phase-transition}
\end{figure*}

\begin{figure}[!t]
\centering
\includegraphics[width=\columnwidth]{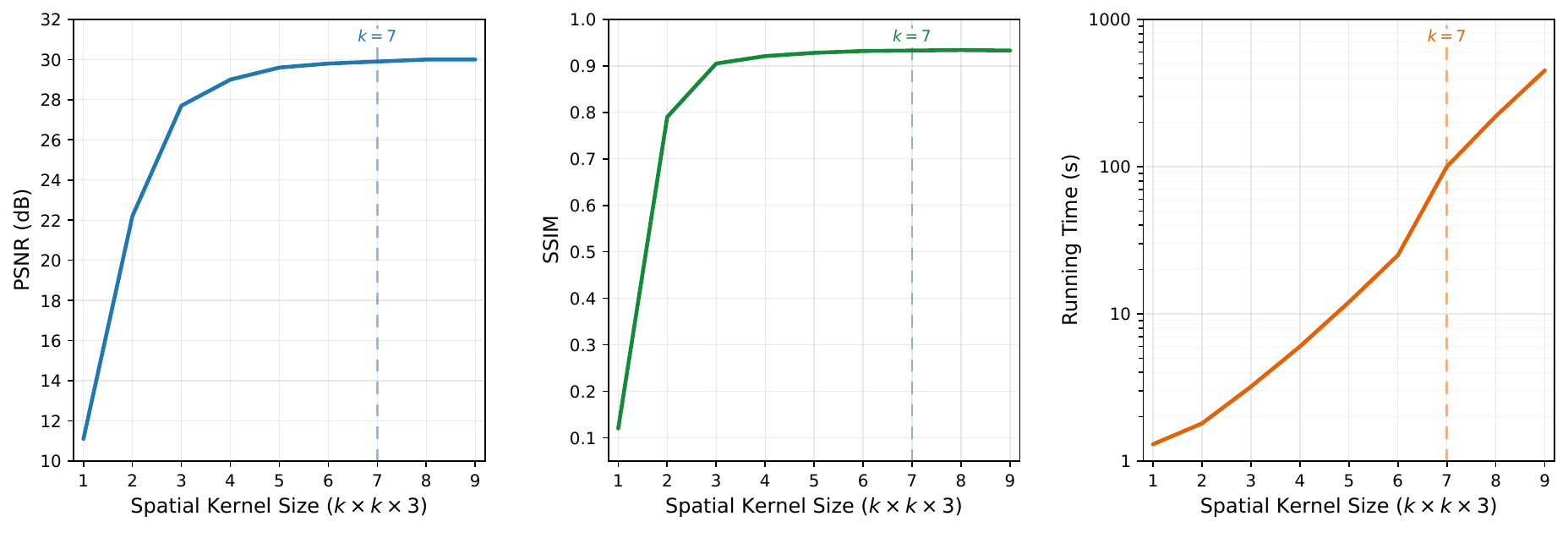}
\caption{Kernel-size sensitivity on Flower with $\mathrm{SR}=0.4$ and $\gamma=0.1$ under random missing. PSNR, SSIM, and running time are shown for $[k,k,3]$ kernels with $k=1,\ldots,9$; the selected setting $k=7$ is highlighted.}
\label{fig:k-choice}
\end{figure}

\begin{figure}[!t]
\centering
\includegraphics[width=\columnwidth]{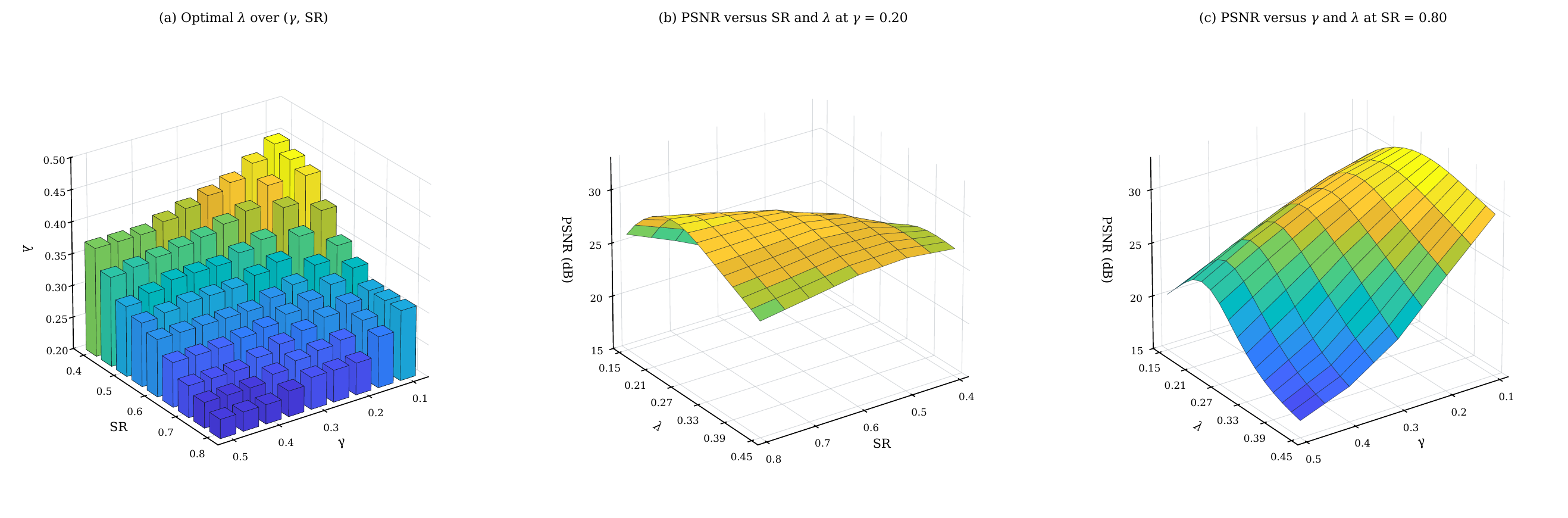}
\caption{Parameter selection and PSNR sensitivity. (a) Optimal $\lambda$ over the $(\gamma,\mathrm{SR})$ grid. (b) PSNR versus $\mathrm{SR}$ and $\lambda$ at $\gamma=0.20$. (c) PSNR versus $\gamma$ and $\lambda$ at $\mathrm{SR}=0.80$.}
\label{fig:parameter-choice}
\end{figure}

\begin{figure}[!t]
\centering
\begin{minipage}[t]{0.495\columnwidth}
\centering
\IfFileExists{figure4a.pdf}{\includegraphics[width=\linewidth]{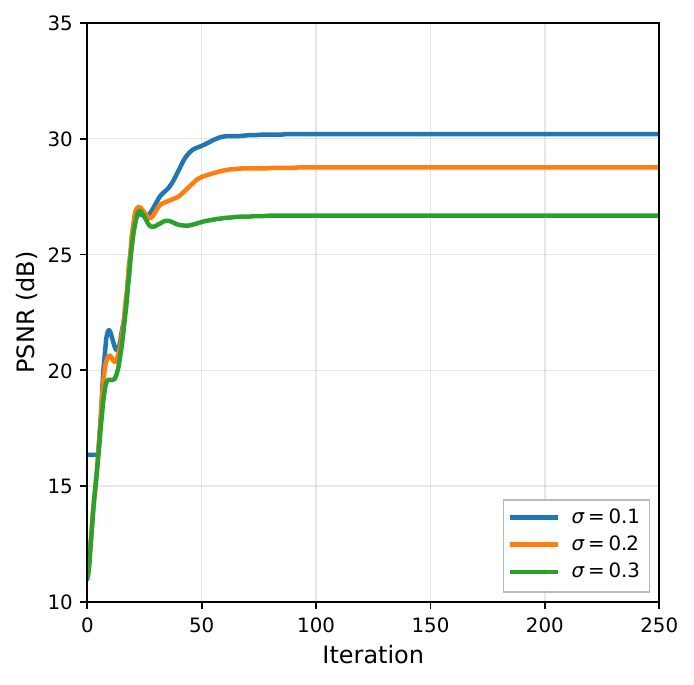}}{\ExpGraphic{figure4a.pdf}{\linewidth}{0.15\textheight}}\\[-1pt]
{\footnotesize (a) Random corruption: PSNR}
\end{minipage}\hfill
\begin{minipage}[t]{0.495\columnwidth}
\centering
\IfFileExists{figure4b.pdf}{\includegraphics[width=\linewidth]{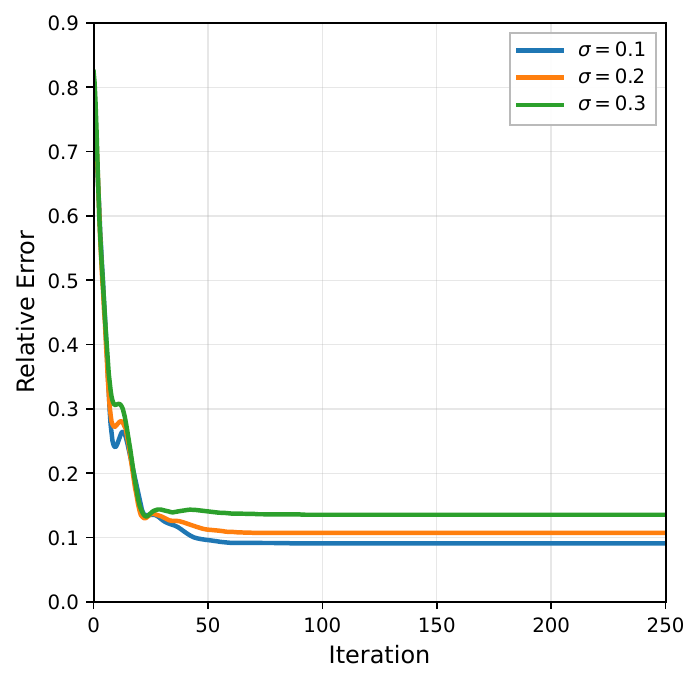}}{\ExpGraphic{figure4b.pdf}{\linewidth}{0.15\textheight}}\\[-1pt]
{\footnotesize (b) Random corruption: RSE}
\end{minipage}

\vspace{0.5mm}

\begin{minipage}[t]{0.495\columnwidth}
\centering
\IfFileExists{figure4c.pdf}{\includegraphics[width=\linewidth]{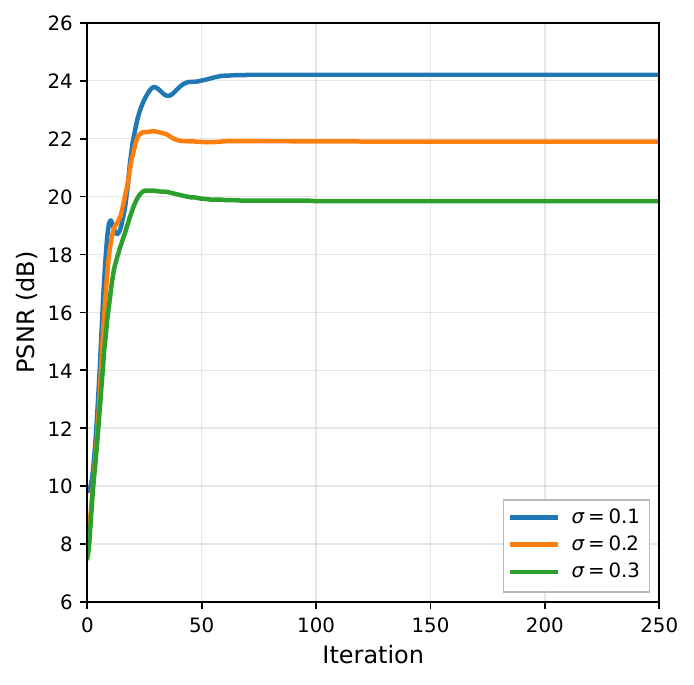}}{\ExpGraphic{figure4c.pdf}{\linewidth}{0.15\textheight}}\\[-1pt]
{\footnotesize (c) Salt-and-pepper corruption: PSNR}
\end{minipage}\hfill
\begin{minipage}[t]{0.495\columnwidth}
\centering
\IfFileExists{figure4d.pdf}{\includegraphics[width=\linewidth]{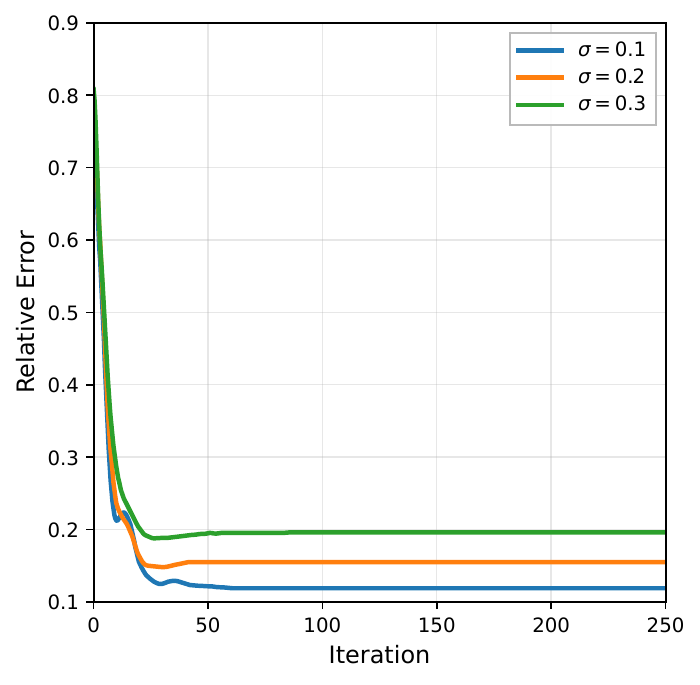}}{\ExpGraphic{figure4d.pdf}{\linewidth}{0.15\textheight}}\\[-1pt]
{\footnotesize (d) Salt-and-pepper corruption: RSE}
\end{minipage}
\caption{Convergence behavior of RCNNM at $\mathrm{SR}=0.3$ and
$\gamma=0.1$. PSNR and RSE are shown versus iteration under random-valued
and salt-and-pepper corruption.}
\label{fig:convergence}
\end{figure}

\subsection{Computational Complexity and Convergence}

\paragraph*{Computational complexity}
For $\bm H\in\RR^{N\times K}$ with $N\ge K$, a direct economy-size SVD has complexity $\mathcal O(NK^2)$. The Gram implementation forms $\bm G=\bm H^\trans\bm H$ in $\mathcal O(NK^2)$ operations, computes the eigendecomposition of the $K\times K$ symmetric matrix in $\mathcal O(K^3)$ operations, and reconstructs the thresholded matrix from the $r_\tau=|\Jtau|$ retained components in $\mathcal O(NKr_\tau)$ operations. The remaining updates consist of applications of $\A$ and $\A^*$ together with entrywise tensor operations. Hence the worst-case leading order of the spectral step remains $\mathcal O(NK^2)$; however, the eigendecomposition is performed on the smaller symmetric Gram matrix and discarded singular components need not be reconstructed. This implementation is particularly useful when $K\ll N$ and $r_\tau\ll K$.

\paragraph*{Convergence}
Once both penalty parameters reach the finite cap $\tau_{\max}$, the iteration becomes a fixed-penalty two-block ADMM with blocks $\mathcal L$ and $(\bm Z,\mathcal S)$. The second block is solved exactly because the $\bm Z$- and $\mathcal S$-subproblems are separable for fixed $\mathcal L$. Standard two-block ADMM convergence results therefore apply \cite{boyd2011admm}.

\section{Numerical Experiments}
\label{sec:experiments}

This section evaluates RCNNM through synthetic recovery tests and image reconstruction experiments. We first examine exact recovery, empirical phase transitions, parameter sensitivity, and numerical convergence. We then compare reflective and circular liftings in terms of spectral compactness and reconstruction quality on BSDS color images and CAVE multispectral images under random, scratch, and grid missing patterns. All experiments are conducted in MATLAB R2026a on a Lenovo Xiaoxin Pro 13 laptop with an Intel Core i5-10210U CPU and 16~GB of RAM.

\begin{table*}[t]
\centering
\begin{minipage}[t]{0.48\textwidth}
\centering
\captionof{table}{Synthetic recovery results on tensors of size $n\times n\times n$.
Here $r=\rank(\A(\mathcal L_0))$ and
$\widehat r=\rank(\A(\widehat{\mathcal L}))$.}
\label{tab:synthetic}
\small
\setlength{\tabcolsep}{4.0pt}
\renewcommand{\arraystretch}{1.08}
\begin{tabular*}{\columnwidth}{@{\extracolsep{\fill}}c|c|c|c||c|c|c|c@{}}
\hline
\multicolumn{4}{c||}{$r=0.05n,\ p=0.9,\ \gamma=0.1$} &
\multicolumn{4}{c}{$r=0.10n,\ p=0.8,\ \gamma=0.2$} \\
\hline
$n$ & $r$ & $\widehat r$ & RSE &
$n$ & $r$ & $\widehat r$ & RSE \\
\hline
100 & 5  & 5  & $5.34\!\times\!10^{-7}$ & 100 & 10 & 10 & $2.89\!\times\!10^{-6}$ \\
200 & 10 & 10 & $4.67\!\times\!10^{-7}$ & 200 & 20 & 20 & $6.87\!\times\!10^{-6}$ \\
300 & 15 & 15 & $1.32\!\times\!10^{-7}$ & 300 & 30 & 30 & $9.13\!\times\!10^{-6}$ \\
\hline
\end{tabular*}
\end{minipage}
\hfill
\begin{minipage}[t]{0.48\textwidth}
\centering
\captionof{table}{Spectral ranks of circular and reflective liftings on the five BSDS images using a $[10,10,3]$ convolution kernel.}
\label{tab:spectral}
\scriptsize
\setlength{\tabcolsep}{1.7pt}
\renewcommand{\arraystretch}{0.90}
\begin{tabular}{lcccccc}
\toprule
& \multicolumn{2}{c}{$r_{\mathrm{eff}}$}
& \multicolumn{2}{c}{$r_{99}$}
& \multicolumn{2}{c}{$r_{99.9}$}\\
\cmidrule(lr){2-3}\cmidrule(lr){4-5}\cmidrule(lr){6-7}
Image & Circular & Reflective & Circular & Reflective & Circular & Reflective\\
\midrule
Flower    & 1.6919 & 1.6540 & 12 & 11 & 56 & 51\\
Bear      & 1.5718 & 1.4433 & 15 & 12 & 66 & 62\\
Goose     & 1.3140 & 1.2403 & 11 & 9  & 70 & 69\\
House     & 1.4193 & 1.3488 & 9  & 7  & 78 & 77\\
Butterfly & 2.0304 & 1.9216 & 21 & 19 & 71 & 68\\
\midrule
Average   & 1.6055 & 1.5216 & 13.6 & 11.6 & 68.20 & 65.40\\
\bottomrule
\end{tabular}
\end{minipage}
\end{table*}

\begin{figure*}[t]
\centering
\IfFileExists{figure5.pdf}{%
  \includegraphics[width=\textwidth]{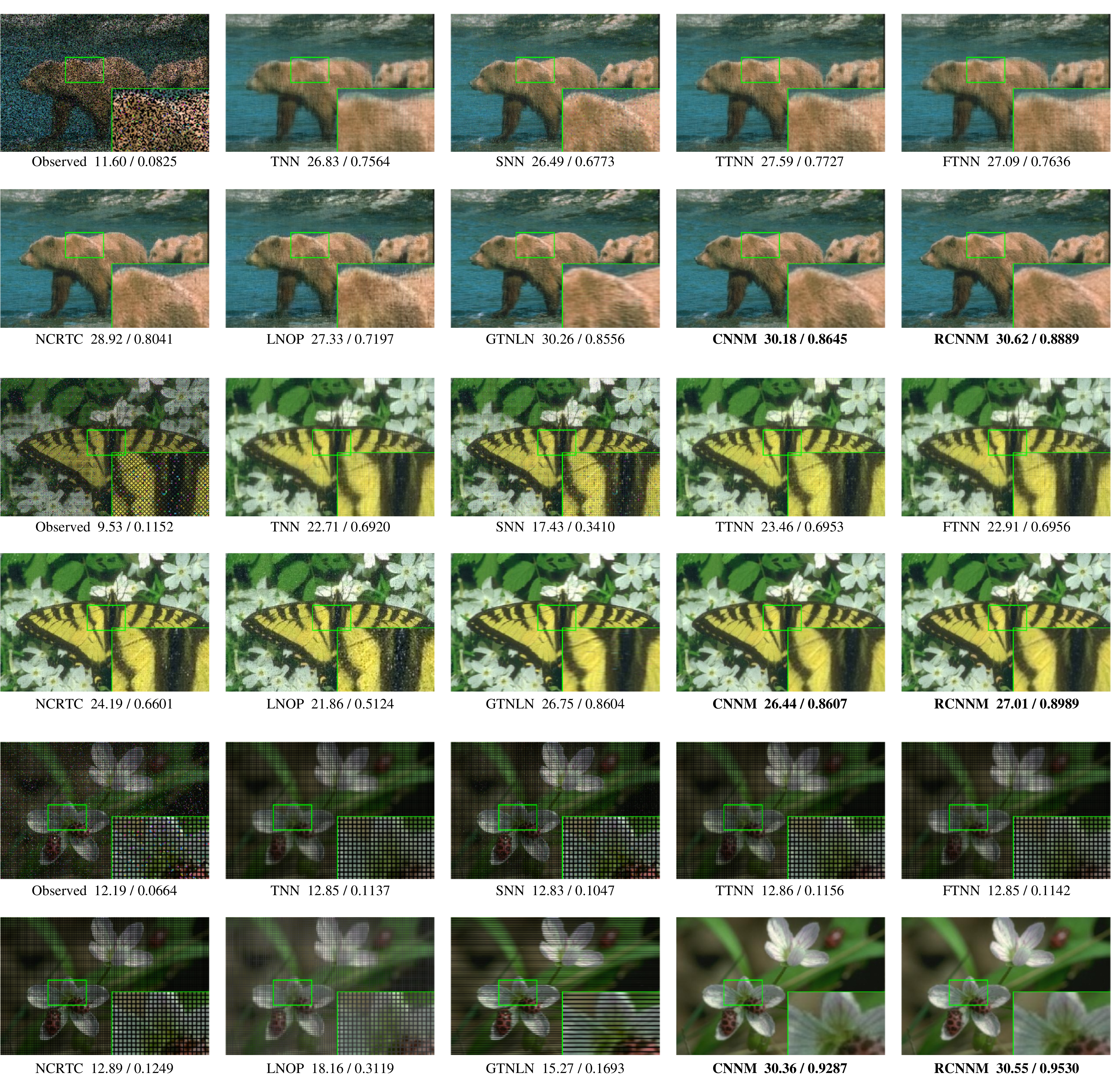}%
}{%
  \ExpGraphicFit{figure5.pdf}{\textwidth}{0.20\textheight}%
}
\caption{Visual comparison on three representative BSDS images at
$\mathrm{SR}=0.6$ and $\gamma=0.1$. From top to bottom: Bear with random
missing, Butterfly with scratch missing, and Flower with grid missing.
Values are PSNR/SSIM.}
\label{fig:color-overall}
\end{figure*}

\begin{table*}[t]
\centering
\centering
{\setlength{\belowcaptionskip}{0pt}
\captionof{table}{Mean global reconstruction results on five BSDS color images. Each method entry reports PSNR/SSIM; PSNR is in dB.}
\label{tab:color-global}}
\scriptsize
\setlength{\tabcolsep}{2.15pt}
\renewcommand{\arraystretch}{0.76}
\resizebox{0.94\textwidth}{!}{%
\begin{tabular}{c|c|c|ccccccccc}
\toprule
$\gamma$ & Missing & $\mathrm{SR}$ & TNN & SNN & LNOP & FTNN & TTNN & NCRTC & GTNLN & CNNM & RCNNM \\
\midrule
\multirow{9}{*}{0.1} & \multirow{3}{*}{Random} & 0.4 & 22.16/0.752 & 22.53/0.829 & 22.35/0.688 & 22.35/0.763 & 22.86/0.781 & 23.81/0.845 & 25.83/0.861 & \underline{26.36}/\underline{0.876} & \textbf{26.66}/\textbf{0.879} \\
 &  & 0.6 & 24.87/0.849 & 24.48/0.834 & 24.74/0.794 & 25.05/0.854 & 25.54/0.863 & 26.48/0.868 & \underline{27.82}/0.909 & 27.70/\underline{0.910} & \textbf{28.19}/\textbf{0.915} \\
 &  & 0.8 & 27.40/0.909 & 26.36/0.865 & 27.07/0.872 & 27.57/0.912 & 28.06/0.915 & 28.64/0.881 & \underline{29.70}/\underline{0.938} & 29.43/0.937 & \textbf{29.96}/\textbf{0.940} \\
\addlinespace[0.4pt]
 & \multirow{3}{*}{Scratch} & 0.4 & 21.58/0.828 & 8.29/0.164 & 15.56/0.327 & 8.55/0.158 & 8.70/0.166 & 8.39/0.164 & \underline{27.00}/\underline{0.891} & 26.53/0.888 & \textbf{27.25}/\textbf{0.895} \\
 &  & 0.6 & 25.00/0.846 & 18.09/0.361 & 25.28/0.813 & 24.39/0.800 & 25.10/0.824 & 25.33/0.664 & 28.62/\textbf{0.923} & \underline{28.99}/0.913 & \textbf{29.47}/\underline{0.915} \\
 &  & 0.8 & 27.35/0.908 & 25.37/0.664 & 27.04/0.871 & 26.84/0.871 & 27.55/0.888 & 27.69/0.762 & 29.74/\textbf{0.940} & \underline{30.30}/0.935 & \textbf{30.86}/\underline{0.938} \\
\addlinespace[0.4pt]
 & \multirow{3}{*}{Grid} & 0.4 & 9.52/0.119 & 9.25/0.225 & 16.66/0.354 & 9.52/0.120 & 9.52/0.124 & 9.35/0.225 & 11.71/0.216 & \underline{21.47}/\underline{0.553} & \textbf{21.70}/\textbf{0.556} \\
 &  & 0.6 & 10.55/0.168 & 10.24/0.257 & 17.69/0.429 & 10.55/0.170 & 10.56/0.174 & 10.25/0.257 & 12.96/0.273 & \underline{27.10}/\underline{0.889} & \textbf{27.41}/\textbf{0.891} \\
 &  & 0.8 & 13.73/0.349 & 13.35/0.356 & 20.63/0.625 & 13.73/0.351 & 13.74/0.355 & 13.66/0.357 & 16.46/0.515 & \underline{29.09}/\underline{0.932} & \textbf{29.50}/\textbf{0.933} \\
\midrule
\multirow{9}{*}{0.2} & \multirow{3}{*}{Random} & 0.4 & 21.28/0.707 & 21.14/0.810 & 19.72/0.587 & 21.40/0.718 & 21.82/0.733 & 21.99/0.833 & 24.42/0.809 & \underline{25.20}/\underline{0.838} & \textbf{25.39}/\textbf{0.844} \\
 &  & 0.6 & 23.66/0.801 & 23.00/0.819 & 22.31/0.730 & 23.76/0.805 & 23.97/0.802 & 24.71/0.856 & \underline{26.45}/0.872 & 26.18/\underline{0.876} & \textbf{26.66}/\textbf{0.879} \\
 &  & 0.8 & 25.62/0.856 & 24.59/0.836 & 24.73/0.825 & 25.69/0.858 & 25.65/0.845 & 26.90/0.870 & \textbf{28.07}/\textbf{0.909} & 27.35/\underline{0.902} & \underline{27.83}/0.901 \\
\addlinespace[0.4pt]
 & \multirow{3}{*}{Scratch} & 0.4 & 19.64/0.774 & 8.00/0.159 & 14.29/0.284 & 8.53/0.152 & 8.71/0.160 & 8.01/0.153 & \underline{25.23}/0.835 & 24.64/\underline{0.847} & \textbf{25.26}/\textbf{0.855} \\
 &  & 0.6 & 23.59/0.791 & 17.19/0.336 & 23.00/0.766 & 22.96/0.758 & 23.44/0.777 & 23.60/\underline{0.875} & \underline{27.04}/\textbf{0.886} & 27.00/0.873 & \textbf{27.54}/0.874 \\
 &  & 0.8 & 25.59/0.855 & 23.56/0.630 & 24.88/0.832 & 25.02/0.819 & 25.10/0.823 & 25.89/0.751 & 28.08/\textbf{0.908} & \underline{28.24}/\underline{0.900} & \textbf{28.74}/\underline{0.900} \\
\addlinespace[0.4pt]
 & \multirow{3}{*}{Grid} & 0.4 & 9.50/0.113 & 9.12/0.124 & 15.21/0.301 & 9.50/0.115 & 9.51/0.118 & 9.07/0.117 & 11.69/0.208 & \underline{20.73}/\textbf{0.539} & \textbf{20.99}/\underline{0.523} \\
 &  & 0.6 & 10.53/0.159 & 10.07/0.151 & 16.30/0.378 & 10.53/0.161 & 10.53/0.164 & 10.12/0.146 & 12.93/0.263 & \underline{25.69}/\textbf{0.859} & \textbf{25.91}/\underline{0.853} \\
 &  & 0.8 & 13.66/0.330 & 13.05/0.243 & 18.98/0.569 & 13.66/0.332 & 13.66/0.332 & 13.38/0.259 & 16.39/0.499 & \underline{27.16}/\textbf{0.899} & \textbf{27.54}/\underline{0.897} \\
\midrule
\multirow{9}{*}{0.3} & \multirow{3}{*}{Random} & 0.4 & 20.30/0.650 & 19.96/0.804 & 17.10/0.483 & 20.43/0.662 & 20.70/0.670 & 19.81/0.797 & 22.04/0.699 & \underline{23.73}/\underline{0.805} & \textbf{24.08}/\textbf{0.819} \\
 &  & 0.6 & 22.21/0.733 & 21.50/0.811 & 19.85/0.658 & 22.32/0.739 & 22.19/0.720 & 22.72/0.746 & 23.93/0.785 & \underline{24.31}/\underline{0.830} & \textbf{25.13}/\textbf{0.844} \\
 &  & 0.8 & 23.53/0.776 & 22.76/0.825 & 22.29/0.768 & 23.60/0.778 & 23.00/0.743 & 24.79/\underline{0.862} & \underline{25.44}/0.836 & 25.19/0.855 & \textbf{25.99}/\textbf{0.867} \\
\addlinespace[0.4pt]
 & \multirow{3}{*}{Scratch} & 0.4 & 17.90/0.716 & 7.76/0.150 & 12.95/0.238 & 8.35/0.137 & 8.55/0.149 & 7.55/0.146 & 22.34/0.714 & \underline{22.80}/\underline{0.804} & \textbf{23.53}/\textbf{0.814} \\
 &  & 0.6 & 22.04/0.718 & 16.56/0.317 & 20.66/0.714 & 21.35/0.697 & 21.40/0.698 & 21.59/0.804 & 24.22/0.796 & \underline{24.52}/\underline{0.811} & \textbf{25.53}/\textbf{0.830} \\
 &  & 0.8 & 23.51/0.776 & 21.85/0.862 & 22.66/0.788 & 22.85/0.748 & 22.43/0.734 & 23.84/0.853 & 25.43/0.836 & \underline{26.10}/\underline{0.867} & \textbf{26.94}/\textbf{0.873} \\
\addlinespace[0.4pt]
 & \multirow{3}{*}{Grid} & 0.4 & 9.47/0.104 & 8.90/0.125 & 13.76/0.247 & 9.48/0.106 & 9.48/0.109 & 8.90/0.123 & 11.63/0.192 & \underline{19.74}/\textbf{0.517} & \textbf{19.90}/\underline{0.498} \\
 &  & 0.6 & 10.49/0.146 & 9.90/0.156 & 14.88/0.324 & 10.49/0.148 & 10.48/0.150 & 9.98/0.156 & 12.85/0.243 & \underline{23.95}/\textbf{0.822} & \textbf{24.49}/\underline{0.819} \\
 &  & 0.8 & 13.53/0.299 & 12.90/0.255 & 17.32/0.507 & 13.53/0.300 & 13.47/0.295 & 13.22/0.253 & 16.21/0.464 & \underline{25.09}/\underline{0.853} & \textbf{25.83}/\textbf{0.864} \\
\bottomrule
\end{tabular}%
}
\vspace{1pt}
\centering
{\setlength{\belowcaptionskip}{0pt}
\captionof{table}{Mean boundary reconstruction results on five BSDS color images. Each method entry reports BPSNR/BSSIM; BPSNR is in dB and the boundary width is $w=5$ pixels.}
\label{tab:color-boundary-metrics}}
\scriptsize
\setlength{\tabcolsep}{2.15pt}
\renewcommand{\arraystretch}{0.50}
\resizebox{0.94\textwidth}{!}{%
\begin{tabular}{c|c|c|ccccccccc}
\toprule
$\gamma$ & Missing & $\mathrm{SR}$ & TNN & SNN & LNOP & FTNN & TTNN & NCRTC & GTNLN & CNNM & RCNNM \\
\midrule
\multirow{9}{*}{0.1} & \multirow{3}{*}{Random} & 0.4 & 22.92/0.766 & 23.14/0.846 & 22.69/0.699 & 23.15/0.780 & 23.71/0.798 & 25.05/0.862 & \underline{25.67}/0.860 & 23.44/\underline{0.887} & \textbf{26.78}/\textbf{0.889} \\
 &  & 0.6 & 25.70/0.862 & 25.36/0.846 & 25.36/0.805 & 25.94/0.868 & 26.42/0.875 & 27.52/0.884 & \underline{27.72}/0.908 & 24.61/\underline{0.920} & \textbf{28.16}/\textbf{0.924} \\
 &  & 0.8 & 28.49/0.919 & 27.20/0.878 & 27.72/0.883 & 28.60/0.922 & 28.90/0.925 & 28.51/0.893 & \textbf{29.72}/0.940 & 26.13/\underline{0.945} & \underline{29.32}/\textbf{0.946} \\
\addlinespace[0.4pt]
 & \multirow{3}{*}{Scratch} & 0.4 & 21.93/0.835 & 8.78/0.211 & 16.13/0.344 & 9.78/0.182 & 9.79/0.192 & 8.82/0.211 & 26.41/0.889 & \underline{26.86}/\underline{0.893} & \textbf{27.53}/\textbf{0.900} \\
 &  & 0.6 & 25.74/0.859 & 17.93/0.359 & 25.31/0.828 & 25.88/0.817 & 26.36/0.840 & 25.60/0.666 & \underline{28.07}/0.922 & 25.65/\underline{0.923} & \textbf{29.16}/\textbf{0.924} \\
 &  & 0.8 & 28.42/0.919 & 26.31/0.664 & 27.24/0.885 & 28.60/0.884 & 28.97/0.899 & 28.65/0.765 & \underline{29.91}/0.941 & 26.93/\underline{0.944} & \textbf{30.09}/\textbf{0.946} \\
\addlinespace[0.4pt]
 & \multirow{3}{*}{Grid} & 0.4 & 10.15/0.136 & 10.10/0.272 & 16.98/0.368 & 10.16/0.141 & 10.16/0.146 & 10.18/0.272 & 12.62/0.219 & \underline{18.60}/\underline{0.569} & \textbf{18.88}/\textbf{0.579} \\
 &  & 0.6 & 11.33/0.187 & 11.28/0.303 & 18.23/0.445 & 11.33/0.192 & 11.34/0.196 & 11.36/0.304 & 13.92/0.276 & \underline{24.69}/\underline{0.900} & \textbf{27.74}/\textbf{0.901} \\
 &  & 0.8 & 15.64/0.370 & 15.51/0.389 & 22.42/0.641 & 15.64/0.372 & 15.66/0.377 & 15.67/0.390 & 21.11/0.529 & \underline{26.57}/\underline{0.941} & \textbf{29.14}/\textbf{0.942} \\
\midrule
\multirow{9}{*}{0.2} & \multirow{3}{*}{Random} & 0.4 & 22.10/0.722 & 22.16/0.826 & 19.95/0.599 & 22.19/0.736 & 22.58/0.751 & 23.49/0.849 & \underline{23.60}/0.807 & 22.34/\underline{0.850} & \textbf{24.79}/\textbf{0.856} \\
 &  & 0.6 & 24.43/0.815 & 23.93/0.834 & 22.69/0.742 & 24.52/0.820 & 24.55/0.818 & \underline{26.08}/0.873 & 25.81/0.870 & 23.15/\underline{0.887} & \textbf{26.66}/\textbf{0.890} \\
 &  & 0.8 & 26.59/0.869 & 25.70/0.850 & 25.13/0.836 & 26.56/0.870 & 26.25/0.859 & 27.11/0.884 & \underline{27.50}/0.907 & 25.19/\textbf{0.912} & \textbf{27.95}/\underline{0.911} \\
\addlinespace[0.4pt]
 & \multirow{3}{*}{Scratch} & 0.4 & 20.00/0.781 & 8.76/0.206 & 14.64/0.301 & 9.76/0.176 & 9.77/0.186 & 8.79/0.201 & 24.19/0.832 & \underline{24.94}/\underline{0.853} & \textbf{25.54}/\textbf{0.860} \\
 &  & 0.6 & 24.11/0.806 & 17.56/0.334 & 23.16/0.782 & 24.21/0.776 & 24.24/0.795 & 24.32/0.876 & \underline{26.02}/0.883 & 23.95/\underline{0.886} & \textbf{27.26}/\textbf{0.887} \\
 &  & 0.8 & 26.35/0.869 & 24.52/0.630 & 25.03/0.848 & 26.47/0.834 & 26.14/0.838 & 27.09/0.754 & \underline{27.62}/\underline{0.907} & 25.05/\textbf{0.911} & \textbf{28.02}/\textbf{0.911} \\
\addlinespace[0.4pt]
 & \multirow{3}{*}{Grid} & 0.4 & 10.14/0.131 & 10.06/0.188 & 15.51/0.315 & 10.14/0.136 & 10.14/0.140 & 10.16/0.182 & 12.59/0.211 & \underline{18.14}/\underline{0.544} & \textbf{18.20}/\textbf{0.560} \\
 &  & 0.6 & 11.31/0.179 & 11.22/0.212 & 16.74/0.394 & 11.30/0.183 & 11.31/0.187 & 11.34/0.208 & 13.88/0.266 & \underline{23.12}/\underline{0.866} & \textbf{25.43}/\textbf{0.872} \\
 &  & 0.8 & 15.53/0.352 & 15.40/0.301 & 20.53/0.586 & 15.55/0.353 & 15.52/0.354 & 15.61/0.314 & 20.82/0.513 & \underline{24.52}/\underline{0.909} & \textbf{26.67}/\textbf{0.910} \\
\midrule
\multirow{9}{*}{0.3} & \multirow{3}{*}{Random} & 0.4 & 20.98/0.667 & 20.97/\underline{0.820} & 17.35/0.496 & 21.09/0.682 & 21.15/0.690 & \underline{21.20}/0.810 & 21.02/0.696 & 20.95/0.818 & \textbf{23.33}/\textbf{0.835} \\
 &  & 0.6 & 22.75/0.749 & 22.62/0.826 & 20.06/0.670 & 22.61/0.758 & 22.14/0.739 & \underline{24.09}/\underline{0.857} & 23.09/0.782 & 21.39/0.843 & \textbf{25.55}/\textbf{0.863} \\
 &  & 0.8 & 23.84/0.790 & 24.02/0.839 & 22.48/0.780 & 23.59/0.794 & 22.63/0.761 & \underline{25.96}/\underline{0.877} & 24.29/0.833 & 22.15/0.868 & \textbf{26.13}/\textbf{0.879} \\
\addlinespace[0.4pt]
 & \multirow{3}{*}{Scratch} & 0.4 & 18.23/0.723 & 8.72/0.200 & 13.11/0.256 & 9.73/0.162 & 9.74/0.175 & 8.73/0.195 & 20.77/0.709 & \underline{23.09}/\underline{0.809} & \textbf{23.79}/\textbf{0.819} \\
 &  & 0.6 & 22.39/0.737 & 17.04/0.315 & 20.96/0.731 & 22.41/0.717 & 21.94/0.719 & 22.39/0.806 & \underline{23.02}/0.792 & 21.90/\underline{0.827} & \textbf{26.35}/\textbf{0.845} \\
 &  & 0.8 & 23.80/0.794 & 22.99/0.865 & 22.77/0.805 & 23.77/0.766 & 22.84/0.753 & \underline{25.03}/\underline{0.876} & 24.36/0.833 & 23.04/0.868 & \textbf{27.24}/\textbf{0.881} \\
\addlinespace[0.4pt]
 & \multirow{3}{*}{Grid} & 0.4 & 10.11/0.123 & 10.02/0.188 & 14.00/0.262 & 10.10/0.128 & 10.09/0.131 & 10.13/0.186 & 12.49/0.195 & \underline{17.52}/\underline{0.521} & \textbf{17.88}/\textbf{0.540} \\
 &  & 0.6 & 11.26/0.167 & 11.18/0.216 & 15.20/0.341 & 11.25/0.171 & 11.22/0.173 & 11.31/0.216 & 13.76/0.246 & \underline{21.45}/\underline{0.835} & \textbf{24.53}/\textbf{0.838} \\
 &  & 0.8 & 15.29/0.323 & 15.25/0.311 & 18.58/0.525 & 15.24/0.323 & 15.06/0.317 & 15.49/0.309 & 20.04/0.476 & \underline{22.37}/\underline{0.867} & \textbf{25.99}/\textbf{0.879} \\
\bottomrule
\end{tabular}%
}
\end{table*}

\begin{figure*}[!t]
\centering
\setlength{\abovecaptionskip}{0pt}
\setlength{\belowcaptionskip}{0pt}
\vspace*{-3pt}
\centering
\makebox[\textwidth][c]{%
\begin{minipage}{\textwidth}
\centering
\includegraphics[width=\linewidth,trim=18bp 370bp 19bp 369bp,clip]{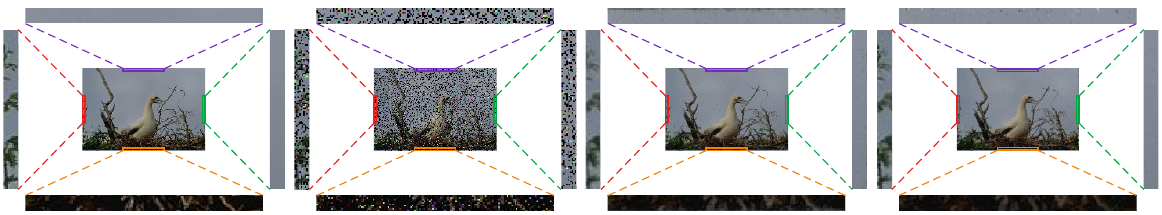}\\[-5pt]
{\footnotesize
\renewcommand{\arraystretch}{0.82}
\begin{tabularx}{\linewidth}{@{}*{4}{>{\centering\arraybackslash}X}@{}}
Original &
Corrupted $12.62/12.89$ &
CNNM $28.18/23.80$ &
RCNNM $\mathbf{28.28/27.80}$
\end{tabularx}}
\end{minipage}}

\vspace{0pt}

\makebox[\textwidth][c]{%
\begin{minipage}{\textwidth}
\centering
\includegraphics[width=\linewidth,trim=18bp 370bp 19bp 369bp,clip]{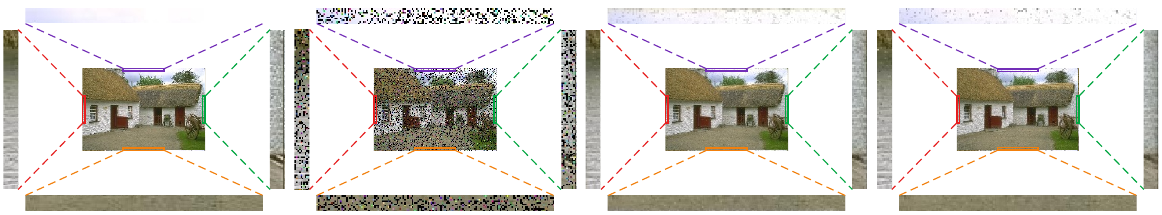}\\[-5pt]
{\footnotesize
\renewcommand{\arraystretch}{0.82}
\begin{tabularx}{\linewidth}{@{}*{4}{>{\centering\arraybackslash}X}@{}}
Original &
Corrupted $11.68/10.07$ &
CNNM $27.03/22.90$ &
RCNNM $\mathbf{27.26/24.10}$
\end{tabularx}}
\end{minipage}}

\vspace{-1pt}
\captionof{figure}{Boundary-region comparison on Goose (top) and House (bottom). The enlarged boundary regions correspond to the central one-third segment of each image edge. Values are PSNR/BPSNR in dB; BPSNR is computed on the five-pixel boundary strip.}
\label{fig:color-boundary}
\vspace{3pt}
\centering
\IfFileExists{figure7.pdf}{%
  \includegraphics[width=\textwidth]{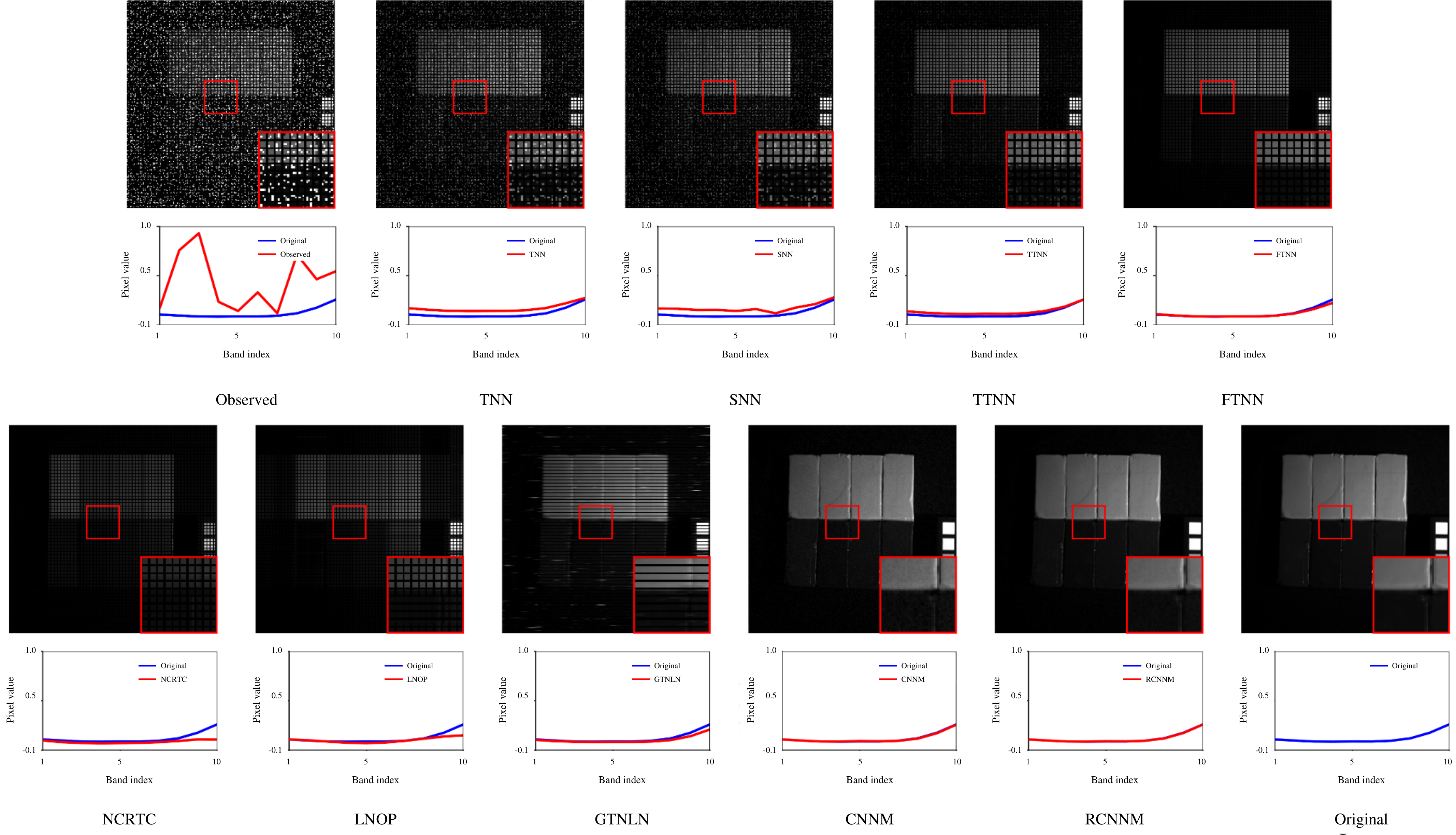}%
}{%
  \ExpGraphicFit{figure7.pdf}{\textwidth}{0.23\textheight}%
}
\captionof{figure}{Visual and spectral comparison on the CAVE Clay scene under grid
missing with $\mathrm{SR}=0.6$ and $\gamma=0.2$. The red boxes mark the selected spatial region and
its enlarged view. For each method, the spectral plot compares the recovered
(or observed) values at the selected pixel with the original spectrum over the
first ten bands.}
\label{fig:ms-clay-visual}
\end{figure*}

\begin{table*}[!t]
\centering
\setlength{\abovecaptionskip}{0pt}
\setlength{\belowcaptionskip}{0pt}
\captionof{table}{Mean global reconstruction results on three CAVE multispectral images with $\gamma=0.2$. Each method entry reports PSNR/SSIM; PSNR is in dB.}
\label{tab:ms-global}
\scriptsize
\setlength{\tabcolsep}{2.15pt}
\renewcommand{\arraystretch}{0.42}
\resizebox{0.94\textwidth}{!}{%
\begin{tabular}{c|c|ccccccccc}
\toprule
Pattern & $\mathrm{SR}$ & TNN & SNN & LNOP & FTNN & TTNN & NCRTC & GTNLN & CNNM & RCNNM \\
\midrule
\multirow{3}{*}{Random} & 0.5 & 28.05/0.4910 & 29.48/0.5899 & \underline{30.26}/0.7705 & 29.21/0.6906 & 29.33/0.7146 & 26.55/0.7653 & 29.03/0.8329 & 29.97/\underline{0.8678} & \textbf{30.97}/\textbf{0.8840} \\
 & 0.6 & 26.70/0.4051 & 30.31/0.6354 & \underline{31.09}/0.7903 & 23.27/0.2450 & 30.16/0.7428 & 27.77/0.7994 & \underline{31.09}/0.8665 & 30.89/\underline{0.8872} & \textbf{31.33}/\textbf{0.8985} \\
 & 0.7 & 25.50/0.3470 & 31.22/0.6742 & 31.82/0.8070 & 22.18/0.2128 & 31.38/0.7170 & 29.18/0.8224 & \textbf{32.30}/\textbf{0.9047} & 31.40/0.8805 & \underline{32.24}/\underline{0.8896} \\
\addlinespace[1pt]
\multirow{3}{*}{Scratch} & 0.5 & 24.34/0.4045 & 25.04/0.3340 & 26.13/0.6120 & 24.28/0.2452 & 24.97/0.5186 & 23.63/0.4485 & 23.44/0.5113 & \underline{29.17}/\underline{0.8416} & \textbf{29.63}/\textbf{0.8573} \\
 & 0.6 & 25.26/0.4140 & 27.56/0.4367 & 27.20/0.6409 & 24.31/0.2255 & 29.98/0.6725 & 25.71/0.5126 & 28.08/0.7119 & \underline{30.22}/\underline{0.8256} & \textbf{30.85}/\textbf{0.8523} \\
 & 0.7 & 26.65/0.4578 & 30.40/0.5843 & 28.26/0.6713 & 23.86/0.2104 & \textbf{32.06}/0.8135 & 28.66/0.6456 & 30.36/0.7712 & 31.54/\textbf{0.8528} & \underline{31.70}/\underline{0.8504} \\
\addlinespace[1pt]
\multirow{3}{*}{Grid} & 0.5 & 24.25/0.3902 & 23.50/0.3162 & 26.28/0.5972 & 22.61/0.2549 & 25.21/0.5455 & 24.63/0.5274 & 26.99/0.7461 & \underline{29.82}/\underline{0.8637} & \textbf{30.27}/\textbf{0.8759} \\
 & 0.6 & 24.14/0.3345 & 24.75/0.3905 & 27.21/0.6414 & 22.02/0.2094 & 25.93/0.5511 & 25.47/0.5812 & 27.98/0.7797 & \underline{30.87}/\underline{0.8857} & \textbf{31.11}/\textbf{0.8914} \\
 & 0.7 & 23.80/0.2833 & 26.34/0.4887 & 28.44/0.6937 & 21.25/0.1746 & 26.60/0.5148 & 26.62/0.6444 & 29.50/0.8278 & \underline{31.45}/\underline{0.8903} & \textbf{31.90}/\textbf{0.9078} \\
\bottomrule
\end{tabular}%
}

\vspace{5pt}
\centering
\captionof{table}{Mean boundary reconstruction results on three CAVE multispectral images with $\gamma=0.2$. Each method entry reports BPSNR/BSSIM; BPSNR is in dB and the boundary width is $w=5$ pixels.}
\label{tab:ms-boundary}
\scriptsize
\setlength{\tabcolsep}{2.15pt}
\renewcommand{\arraystretch}{0.42}
\resizebox{0.94\textwidth}{!}{%
\begin{tabular}{c|c|ccccccccc}
\toprule
Pattern & $\mathrm{SR}$ & TNN & SNN & LNOP & FTNN & TTNN & NCRTC & GTNLN & CNNM & RCNNM \\
\midrule
\multirow{3}{*}{Random} & 0.5 & 29.00/0.6170 & 29.44/0.6478 & \underline{30.75}/0.8445 & 29.49/0.7017 & 29.97/0.7449 & 26.48/0.7964 & 29.14/0.8314 & 30.34/\underline{0.8924} & \textbf{31.52}/\textbf{0.8957} \\
 & 0.6 & 27.69/0.5544 & 30.33/0.6849 & \underline{31.70}/0.8497 & 23.28/0.3648 & 31.52/0.7665 & 27.73/0.8287 & 31.10/0.8619 & 31.27/\underline{0.9135} & \textbf{33.24}/\textbf{0.9180} \\
 & 0.7 & 26.48/0.5093 & 31.18/0.7157 & \underline{32.96}/0.8600 & 22.18/0.3385 & 31.86/0.7498 & 29.14/0.8480 & 32.19/0.8803 & 31.51/\underline{0.9103} & \textbf{33.98}/\textbf{0.9368} \\
\addlinespace[1pt]
\multirow{3}{*}{Scratch} & 0.5 & 25.13/0.4851 & 24.93/0.3855 & 29.01/0.7161 & 24.22/0.3235 & 25.05/0.5376 & 23.54/0.4908 & 23.45/0.5025 & \underline{29.53}/\underline{0.8658} & \textbf{29.80}/\textbf{0.8767} \\
 & 0.6 & 26.06/0.5057 & 27.46/0.4913 & 30.25/0.7410 & 24.26/0.3134 & 30.12/0.6974 & 25.70/0.5471 & 28.37/0.7313 & \underline{30.50}/\underline{0.8434} & \textbf{30.93}/\textbf{0.8635} \\
 & 0.7 & 27.50/0.5521 & 30.46/0.6444 & 31.27/0.7679 & 23.87/0.2981 & \underline{31.49}/0.8423 & 28.62/0.6780 & 30.54/0.8133 & 30.72/\underline{0.8745} & \textbf{32.19}/\textbf{0.8748} \\
\addlinespace[1pt]
\multirow{3}{*}{Grid} & 0.5 & 25.18/0.4278 & 23.45/0.3431 & 28.53/0.6902 & 22.56/0.3168 & 25.13/0.4865 & 24.55/0.4977 & 26.91/0.7387 & \underline{29.92}/\underline{0.8857} & \textbf{30.40}/\textbf{0.8972} \\
 & 0.6 & 25.10/0.4255 & 24.71/0.4041 & 29.03/0.7246 & 22.00/0.3119 & 25.89/0.5082 & 25.42/0.5613 & 27.91/0.7755 & \underline{30.58}/\underline{0.9013} & \textbf{31.33}/\textbf{0.9155} \\
 & 0.7 & 24.79/0.4245 & 26.33/0.4976 & 30.27/0.7964 & 21.26/0.3082 & 26.58/0.5161 & 26.59/0.6360 & 29.38/0.8209 & \underline{31.15}/\underline{0.9054} & \textbf{32.39}/\textbf{0.9299} \\
\bottomrule
\end{tabular}%
}
\vspace{-2pt}
\end{table*}

\subsection{Experimental Setup}
\label{sec:experimental-setup}

We compare RCNNM with TNN \cite{jiang2019rtc}, SNN \cite{huang2014provable}, LNOP \cite{chen2020lnop}, FTNN \cite{jiang2020ftnn}, TTNN \cite{song2020robust}, NCRTC \cite{zhao2020ncrtc}, and GTNLN \cite{shu2026gtnln}. To examine the effect of the boundary condition, we also consider the corresponding circular-boundary version obtained by replacing the reflective lifting $\A$ with the periodic convolutional lifting $\A_{\mathrm{circ}}$ used in CNNM \cite{liu2023cnnm}, while keeping the sparse-error model, constraints, parameter selection, and ADMM solver unchanged. This circular-boundary version is denoted by CNNM throughout the experiments.

Published or author-recommended parameters are used for the competing methods whenever available. For GTNLN, only the input interface is adapted to the image tensors; its objective and update rules are left unchanged. All methods are evaluated on identical degraded observations without post-processing. In the quantitative tables, the best and second-best results are shown in boldface and underlined, respectively.

For image data, $\mathcal L_0\in\RR^{m\times n\times b}$, with $b=3$ for RGB images and $b=10$ for the multispectral experiments. Intensities are normalized to $[0,1]$, and the sampling ratio is $\mathrm{SR}=|\Oset|/(mnb)$. The corruption density is denoted by $\gamma$; corruption locations are shared across channels or bands, and corrupted values are independently drawn from $\mathrm{Unif}[0,1]$. 
Three missing patterns are considered: random missing, grid missing formed by uniformly distributed missing rows and columns, and scratch missing obtained by rotating the line-mask construction by $45^\circ$. The latter two patterns, together with channel-shared corruption supports, deliberately depart from the independence assumptions in Theorem~\ref{thm:core-main}; the real-image experiments should therefore be interpreted as robustness tests rather than direct verification of the probabilistic recovery theorem.

Reconstruction quality is measured by PSNR and SSIM \cite{wang2004ssim}. Because the proposed modification acts at the finite image boundary, we additionally report boundary PSNR (BPSNR) and boundary SSIM (BSSIM) on a strip of width $w=5$ pixels. Let
\begin{equation}
\Bset_w
=
\{(i,j,k): i\le w\ \text{or}\ i>m-w\ \text{or}\ j\le w\ \text{or}\ j>n-w\}.
\label{eq:boundary-set}
\end{equation}
BPSNR is the PSNR computed on $\Bset_w$, and BSSIM is the mean SSIM over the same boundary region. For synthetic tensors, recovery accuracy is measured by
\begin{equation}
\mathrm{RSE}
=
\frac{\norm{\widehat{\mathcal L}-\mathcal L_0}_F}{\norm{\mathcal L_0}_F}.
\label{eq:rse}
\end{equation}


\subsection{Synthetic Experiments}
\label{sec:synthetic}

For the exact-recovery test, clean tensors are generated from separable low-frequency cosine components and normalized before sampling and corruption. The number of components is adjusted to obtain the prescribed reflective lifted rank $r=\rank(\A(\mathcal L_0))$. We consider $(r,p,\gamma)=(0.05n,0.9,0.1)$ and $(0.1n,0.8,0.2)$ for $n\in\{100,200,300\}$. Table~\ref{tab:synthetic} shows that the recovered tensors match the target lifted rank in all six cases, with RSE values between $10^{-7}$ and $10^{-6}$, indicating near-exact numerical recovery in the tested regimes.

We next examine empirical phase transitions on $100\times100\times10$ tensors using a $[5,5,5]$ reflective kernel. For each $(r,\gamma,p)$ combination, ten independent trials are performed and recovery is declared successful when $\mathrm{RSE}<10^{-2}$. As shown in Fig.~\ref{fig:phase-transition}, the successful-recovery region shrinks with increasing rank or corruption and expands with the sampling probability, in agreement with the qualitative dependence predicted by the recovery conditions.

\subsection{Parameter Settings and Convergence}
\label{sec:parameter-sensitivity}

Unless otherwise stated, the ADMM parameters are $\tau_1=5\times10^{-3}$, $\tau_2=10^{-1}$, $\varrho=1.05$, and $\tau_{\max}=100$. We use at most 300 iterations and the stopping tolerance $\varepsilon_{\mathrm{tol}}=10^{-4}$ in \eqref{eq:stopping-residuals}.

For RGB data, we use kernels of size $[k,k,3]$. Figure~\ref{fig:k-choice} shows the sensitivity to $k$ on Flower at $\mathrm{SR}=0.4$ and $\gamma=0.1$ under random missing. Reconstruction quality improves as the spatial kernel grows, at the cost of increased running time; we therefore set $k=7$ in the subsequent RGB reconstruction experiments. The spectral diagnostic in Section~\ref{sec:spectral-compactness} uses a larger $[10,10,3]$ kernel to resolve the lifted spectrum more clearly.

For each RGB degradation setting, $\lambda$ is selected once by grid search on an additional color image excluded from the test set and is then fixed across the five BSDS test images. Figure~\ref{fig:parameter-choice} illustrates both the selected parameter values and the sensitivity of the reconstruction quality to $\lambda$. Panel (a) shows the optimal $\lambda$ over the $(\gamma,\mathrm{SR})$ grid. Panel (b) shows the PSNR landscape over $(\mathrm{SR},\lambda)$ at fixed $\gamma=0.20$, whereas panel (c) shows the PSNR landscape over $(\gamma,\lambda)$ at fixed $\mathrm{SR}=0.80$. The results show that the preferred regularization level changes systematically with the degradation setting, motivating the setting-dependent parameter selection used in the subsequent experiments.

Figure~\ref{fig:convergence} illustrates the numerical behavior of the solver at $\mathrm{SR}=0.3$ and $\gamma=0.1$ under random-valued and salt-and-pepper corruption. The same initialization, stopping criterion, and penalty-parameter update are used throughout the remaining experiments.

\subsection{Spectral Compactness}
\label{sec:spectral-compactness}

To isolate the effect of the boundary extension before reconstruction, we compare the spectra of circular and reflective liftings on the clean BSDS images using the same $[10,10,3]$ kernel. Let $\{\sigma_i\}$ be the singular values of a lifted matrix and define
\begin{equation}
\pi_i=\frac{\sigma_i^2}{\sum_j\sigma_j^2},
\qquad
r_{\mathrm{eff}}
=
\Exp\!\left(-\sum_i \pi_i\Log \pi_i\right),
\label{eq:energy-effective-rank}
\end{equation}
where $0\Log0$ is interpreted as zero. We further denote by $r_{99}$ and $r_{99.9}$ the smallest numbers of leading singular values required to contain at least $99\%$ and $99.9\%$ of the squared singular-value energy, respectively.

Table~\ref{tab:spectral} shows lower values for the reflective lifting on every image and under all three measures. On average, $r_{\mathrm{eff}}$, $r_{99}$, and $r_{99.9}$ decrease from $1.6055$, $13.6$, and $68.20$ to $1.5216$, $11.6$, and $65.40$, respectively. Thus, for the tested images, reflective continuation yields a more concentrated lifted spectrum than circular continuation.

\subsection{Color Images}
\label{sec:color-quantitative}

We evaluate color-image recovery on Flower, Bear, Goose, House, and Butterfly from BSDS \cite{martin2001bsds}, each represented as a $321\times481\times3$ tensor. We consider $\gamma\in\{0.1,0.2,0.3\}$, $\mathrm{SR}\in\{0.4,0.6,0.8\}$, and the three missing patterns defined above. Figure~\ref{fig:color-overall} shows representative reconstructions, while Tables~\ref{tab:color-global} and~\ref{tab:color-boundary-metrics} report the global and boundary-region averages, respectively. The two evaluations are complementary: the global metrics measure overall reconstruction quality, whereas the boundary metrics directly probe the part of the image affected by the choice between circular and reflective boundary extension.

The quantitative results reveal a consistent distinction between overall reconstruction quality and boundary fidelity. In Table~\ref{tab:color-global}, RCNNM attains the highest mean PSNR in $26$ of the $27$ degradation settings. Relative to CNNM, the average PSNR gains are $0.49$, $0.67$, and $0.36$ dB under random, scratch, and grid missing, respectively, showing that replacing circular extension by reflection improves the reconstruction without sacrificing global accuracy. More importantly, the advantage becomes substantially more pronounced when the evaluation is restricted to the five-pixel boundary strip. As shown in Table~\ref{tab:color-boundary-metrics}, RCNNM achieves the highest BPSNR in $26$ of the $27$ settings and the highest BSSIM in $26$ of the $27$ settings. Its mean BPSNR gains over CNNM increase to $3.26$, $2.62$, and $1.94$ dB for random, scratch, and grid missing, respectively, with individual improvements reaching $4.16$ dB for random missing and $4.45$ dB for scratch missing. The much larger separation in BPSNR than in global PSNR is particularly informative: boundary pixels occupy only a small fraction of the image, so improvements localized near the frame are strongly diluted by a full-image average. This behavior is therefore consistent with the intended role of the reflective lifting, whose principal difference from the circular construction lies precisely in how finite boundaries are extended. Figure~\ref{fig:color-boundary} provides the corresponding visual evidence on Goose and House, where the reflective reconstruction suppresses boundary artifacts more effectively while the interior reconstruction remains comparable. Taken together, the global and boundary results indicate that the gain of RCNNM is not a uniform numerical shift, but a targeted improvement in the region where circular continuation introduces the strongest modeling mismatch.

\newpage

\subsection{Multispectral Images}
\label{sec:multispectral}

We further evaluate RCNNM on Cloth, Superballs, and Clay from the CAVE database \cite{yasuma2010cave}. Each $512\times512\times31$ scene is resized to $256\times256$, and the first ten spectral bands are retained. We set $\gamma=0.2$, $\mathrm{SR}\in\{0.5,0.6,0.7\}$, and use a $[9,9,5]$ reflective kernel. The three missing patterns are generated as in the RGB experiments. 

For each degradation setting, $\lambda$ is selected once on a reference scene and then fixed across the three test scenes. Tables~\ref{tab:ms-global} and~\ref{tab:ms-boundary} report mean global and boundary metrics, respectively.

Figure~\ref{fig:ms-clay-visual} gives a representative spatial--spectral comparison on Clay at $\mathrm{SR}=0.6$ and $\gamma=0.2$, including the recovered spectrum at a selected pixel over the first ten bands. The global results in Table~\ref{tab:ms-global} show that RCNNM improves the circular-boundary version by an average of $0.76$, $0.42$, and $0.38$ dB under random, scratch, and grid missing, respectively. The corresponding BPSNR gains in Table~\ref{tab:ms-boundary} are larger: $1.87$, $0.72$, and $0.82$ dB. Under grid missing, RCNNM also improves the circular-boundary version by $0.0118$ in SSIM and $0.0167$ in BSSIM on average.

Across all baselines, GTNLN and TTNN yield slightly higher global PSNR at $\mathrm{SR}=0.7$ for random and scratch missing, respectively; RCNNM nevertheless attains the highest BPSNR and BSSIM in both cases. Under random missing, its BPSNR gain over the circular-boundary version increases from $1.18$ dB at $\mathrm{SR}=0.5$ to $2.47$ dB at $\mathrm{SR}=0.7$. The CAVE results therefore reproduce the main trend observed on RGB images: replacing circular continuation by reflection produces modest but consistent global improvements and a larger gain near the spatial boundary.

\section{Conclusion}
\label{sec:conclusion}

We proposed RCNNM, a robust tensor completion framework based on endpoint-nonrepeating reflective convolutional lifting. Reflection changes the Gram geometry from $\A_{\mathrm{circ}}^*\A_{\mathrm{circ}}=K\Iop$ to $\A_{\mathrm{ref}}^*\A_{\mathrm{ref}}=\Dop$, which supports both the recovery analysis and the closed-form tensor update. Under random sampling and random-sign sparse corruption, we established exact recovery guarantees and stability to bounded dense perturbations. The convex model is solved by a two-block ADMM, with singular-value thresholding implemented through the smaller right Gram matrix.

Experiments on synthetic tensors, BSDS color images, and CAVE multispectral images show that reflective lifting consistently outperforms the circular counterpart, with the clearest gains in boundary-restricted metrics. Future work will consider practical parameter selection and nonconvex low-rank regularization. Replacing the nuclear norm with a tighter nonconvex spectral surrogate may better approximate rank and reduce shrinkage bias on dominant singular values, but would require corresponding optimization, convergence, and recovery analysis.

\appendix
\renewcommand{\thetheorem}{\Alph{section}.\arabic{theorem}}

\section{Reflective Geometry and Deterministic Recovery}
\label{app:det-proof}

\subsection{Weighted Gram Geometry}
\label{app:pub-geometry}

\begin{proof}[Proof of Proposition~\ref{thm:core-gram}]
Index the rows of a lifted matrix by $\bm{i}\in\Iset$ and its columns by
$\bm{t}\in\Kset$, consistently with \eqref{eq:core-A}.  From the adjoint
identity and \eqref{eq:core-Rt},
\[
\ip{\A(\mathcal X)}{\bm{Z}}
=
\sum_{\bm{i}\in\Iset}\sum_{\bm{t}\in\Kset}
\mathcal X_{\phi(\bm{i}+\bm{t})}\bm{Z}_{\bm{i},\bm{t}}.
\]
Collecting the coefficients associated with each original entry $\mathcal X_{\bm{u}}$ gives
\[
[\A^*(\bm{Z})]_{\bm{u}}=
\sum_{\bm{i}\in\Iset}\sum_{\bm{t}\in\Kset}
\1_{\{\phi(\bm{i}+\bm{t})=\bm{u}\}}\bm{Z}_{\bm{i},\bm{t}}.
\]
For $\bm{Z}=\A(\mathcal X)$, whenever $\phi(\bm{i}+\bm{t})=\bm{u}$, the corresponding lifted entry is
$\bm{Z}_{\bm{i},\bm{t}}=\mathcal X_{\bm{u}}$.  Hence
\[
[\A^*(\A(\mathcal X))]_{\bm{u}}
=
\left(\sum_{\bm{i},\bm{t}}\1_{\{\phi(\bm{i}+\bm{t})=\bm{u}\}}\right)\mathcal X_{\bm{u}}
=
d_{\bm{u}}\mathcal X_{\bm{u}},
\]
where the last equality is \eqref{eq:core-multiplicity-count}.  This proves
\eqref{eq:core-gram}.  The zero offset maps each coordinate to itself, so
$d_{\bm{u}}\ge1$.  Moreover,
\[
\norm{\A(\mathcal X)}_F^2
=
\ip{\A(\mathcal X)}{\A(\mathcal X)}
=
\ip{\mathcal X}{\A^*\A(\mathcal X)}
=
\sum_{\bm{u}}d_{\bm{u}}\mathcal X_{\bm{u}}^2.
\]
Bounding each $d_{\bm{u}}$ between $d_{\min}$ and $d_{\max}$ proves
\eqref{eq:core-weighted}.  Thus $\A$ is injective, and pulling the nuclear
norm back through this injective linear map gives a norm.
\end{proof}

For the clean observed set $\Lset$, define
\begin{equation}
\alpha
=
\Inf_{\substack{\mathcal H\in\V\\\mathcal H\ne0}}
\frac{\norm{\mathscr P_\Lset (\mathcal H)}_F}{\norm{\mathcal H}_F}.
\label{eq:core-alpha}
\end{equation}

\subsection{Deterministic Identifiability}
\label{app:pub-identifiability}

\begin{proposition}[Identifiability]
\label{prop:pub-ident}
There is no nonzero feasible ambiguity
\[
\mathcal H\in\V,\qquad
\supp(\mathcal J)\subseteq\Gset,\qquad
\mathscr P_\Oset(\mathcal H+\mathcal J)=0
\]
if and only if $\alpha>0$.
\end{proposition}

\begin{proof}
Assume first that $\alpha>0$.  Since $\supp(\mathcal J)\subseteq\Gset$ and
$\Lset\cap\Gset=\varnothing$, $\mathscr P_\Lset (\mathcal J)=0$.  Feasibility gives
$\mathscr P_\Lset (\mathcal H)=-\mathscr P_\Lset (\mathcal J)=0$.  Since $\mathcal H\in\V$, the definition of $\alpha$
forces $\mathcal H=0$.  Feasibility then gives $\mathscr P_\Oset (\mathcal J)=0$, while
$\supp(\mathcal J)\subseteq\Gset\subseteq\Oset$, so $\mathcal J=0$.

Conversely, suppose $\alpha=0$.  The unit sphere of finite-dimensional
$\V$ is compact, so the infimum in \eqref{eq:core-alpha} is attained.
There is therefore a nonzero $\mathcal H\in\V$ with $\mathscr P_\Lset (\mathcal H)=0$.  Setting
$\mathcal J=-\mathscr P_\Gset (\mathcal H)$ gives $\supp(\mathcal J)\subseteq\Gset$ and
\[
\mathscr P_\Oset(\mathcal H+\mathcal J)
=
\mathscr P_\Lset (\mathcal H)+\mathscr P_\Gset (\mathcal H)-\mathscr P_\Gset (\mathcal H)
=
0.
\]
Thus $(\mathcal H,\mathcal J)$ is a nonzero feasible ambiguity.
\end{proof}

\subsection{Strict Dual-Certificate Conditions}
\label{app:pub-dual}

\begin{lemma}[Strict common-dual certificate]
\label{thm:core-strict}
Assume $\alpha>0$.  Suppose there exist $\bm{W},\mathcal F,\mathcal Y$ such that
\[
\mathscr P_{\Tsp}(\bm{W})=0,\qquad \norm{\bm{W}}_2<1,
\]
\[
\mathscr P_\Gset (\mathcal F)=\mathscr P_{\Oset^c}(\mathcal F)=0,\qquad
\norm{\mathscr P_\Lset (\mathcal F)}_\infty<1,
\]
and
\begin{equation}
\mathscr P_{\Oset^c}(\mathcal Y)=0,\qquad
\A^*(\bm{U}\bm{V}^\trans+\bm{W})
=
\mathcal Y
=
\lambda(\mathcal S_0^{\mathrm{sgn}}+\mathcal F).
\label{eq:core-common-dual}
\end{equation}
Then $(\mathcal L_0,\mathcal S_0)$ is the unique minimizer of
\eqref{eq:core-rcnnm}.
\end{lemma}

We use the following strict subgradient inequalities.

\begin{lemma}[Strict subgradient margins]
\label{lem:pub-subgradient}
If $\bm{W}\in \Tsp^\perp$ with $\norm{\bm{W}}_2<1$, then for every matrix $\bm{Z}$,
\begin{align}
\norm{\bm{X_0}+\bm{Z}}_*
\ge{}&
\norm{\bm{X_0}}_*
+\ip{\bm{U}\bm{V}^\trans+\bm{W}}{\bm{Z}}\nonumber\\
&+
(1-\norm{\bm{W}}_2)\norm{\mathscr P_{\Tsp^\perp}(\bm{Z})}_*.
\label{eq:pub-nuclear-margin}
\end{align}
If $\mathcal F$ is supported on $\Lset$ and
$\norm{\mathcal F}_\infty<1$, then for every tensor $\mathcal J$,
\begin{align}
\norm{\mathcal S_0+\mathcal J}_1
\ge{}&
\norm{\mathcal S_0}_1+\ip{\mathcal S_0^{\mathrm{sgn}}+\mathcal F}{\mathcal J}\nonumber\\
&+
(1-\norm{\mathcal F}_\infty)\norm{\mathscr P_\Lset (\mathcal J)}_1.
\label{eq:pub-l1-margin}
\end{align}
\end{lemma}

\begin{proof}
Choose a polar factor of $\mathscr P_{\Tsp^\perp}(\bm{Z})$ whose spectral norm is one and
add it to $\bm{U}\bm{V}^\trans$ in the standard nuclear-norm subgradient
description; the contribution of $\bm{W}$ can reduce the resulting normal
term by at most $\norm{\bm{W}}_2\norm{\mathscr P_{\Tsp^\perp}(\bm{Z})}_*$.  The $\ell_1$ statement
is the coordinatewise analogue: choose the signs of the clean-set
perturbation as the free part of the $\ell_1$ subgradient, and use
$\abs{\mathcal F_{\bm{u}}}<1$.
\end{proof}

\begin{proof}[Proof of Lemma~\ref{thm:core-strict}]
Let $(\mathcal H,\mathcal J)$ be a feasible perturbation and define
\[
x=\norm{\mathscr P_{\Tsp^\perp}(\A(\mathcal H))}_*,
\qquad
y=\norm{\mathscr P_\Lset (\mathcal J)}_1.
\]
Applying Lemma~\ref{lem:pub-subgradient} and
\eqref{eq:core-common-dual} gives
\[
\Delta f
\ge
(1-\norm{\bm{W}}_2)x
+
\lambda(1-\norm{\mathscr P_\Lset (\mathcal F)}_\infty)y
+
\ip{\mathcal Y}{\mathcal H+\mathcal J}.
\]
Since $\mathcal Y$ is supported on $\Oset$ and
$\mathscr P_\Oset(\mathcal H+\mathcal J)=0$, the last term vanishes.  Thus a nonzero perturbation
with $x>0$ or $y>0$ strictly increases the objective.

If $x=y=0$, then $\A(\mathcal H)\in\Tsp$, hence $\mathcal H\in\V$.  Also
$\mathscr P_\Lset (\mathcal J)=0$, so feasibility gives $\mathscr P_\Lset (\mathcal H)=0$.  The condition
$\alpha>0$ forces $\mathcal H=0$, and then $\mathcal J=0$.  Therefore the true pair is
the unique minimizer.
\end{proof}

The golfing construction below first produces a certificate with a small residual in $\V$.  We therefore state the following approximate version.

\subsection{Approximate Dual Certificate}
\label{app:pub-approx}

\begin{proposition}[Approximate dual certificate]
\label{thm:core-approx}
Assume $\alpha>0$ and $\beta_{\Tsp}>0$.  Suppose that
$\bm{W}\in\RR^{N\times K}$, $\mathcal F,\mathcal R\in\Hh$ satisfy
\begin{equation}
\mathscr P_{\Tsp}(\bm{W})=0,\qquad
\mathscr P_\Gset (\mathcal F)=0,\qquad
\mathscr P_{\Oset^c}(\mathcal F)=0,
\qquad
\mathcal R\in\V,
\label{eq:core-approx-support}
\end{equation}
together with the approximate common-dual balance
\begin{equation}
\A^*(\bm{U}\bm{V}^\trans+\bm{W})
=
\lambda(\mathcal S_0^{\mathrm{sgn}}+\mathcal F)+\mathcal R.
\label{eq:core-approx-balance}
\end{equation}
Define
\[
\delta_2=1-\norm{\bm{W}}_2,\qquad
\delta_\infty=1-\norm{\mathscr P_\Lset (\mathcal F)}_\infty,
\]
and assume $\delta_2,\delta_\infty>0$.  If
\begin{equation}
\norm{\mathcal R}_F
<
\alpha\Min\{\beta_{\Tsp}\delta_2,\lambda\delta_\infty\},
\label{eq:core-residual}
\end{equation}
then $(\mathcal L_0,\mathcal S_0)$ is the unique minimizer of
\eqref{eq:core-rcnnm}.
\end{proposition}

\begin{proof}[Proof of Proposition~\ref{thm:core-approx}]
For a feasible perturbation use the same $x,y$.  The approximate balance
\eqref{eq:core-approx-balance} gives
\[
\Delta f
\ge
\delta_2x+\lambda\delta_\infty y+\ip{\mathcal R}{\mathcal H}.
\]
Write $\mathcal H=\mathcal H_\V+\mathcal H_{\V^\perp}$.  Since $\mathcal R\in\V$,
\[
\abs{\ip{\mathcal R}{\mathcal H}}
\le\norm{\mathcal R}_F\norm{\mathcal H_\V}_F.
\]
Feasibility on $\Lset$ gives
\[
\mathscr P_\Lset (\mathcal H_\V)
=
-\mathscr P_\Lset (\mathcal J)-\mathscr P_\Lset (\mathcal H_{\V^\perp}),
\]
hence
\[
\alpha\norm{\mathcal H_\V}_F
\le
y+\norm{\mathcal H_{\V^\perp}}_F.
\]
By the definition of $\beta_{\Tsp}$,
\[
\beta_{\Tsp}\norm{\mathcal H_{\V^\perp}}_F
\le
\norm{\mathscr P_{\Tsp^\perp}(\A(\mathcal H))}_F
\le x.
\]
Therefore
\[
\abs{\ip{\mathcal R}{\mathcal H}}
\le
\frac{\norm{\mathcal R}_F}{\alpha}
\left(y+\frac{x}{\beta_{\Tsp}}\right).
\]
Condition \eqref{eq:core-residual} makes both remaining coefficients
strictly positive.  Equality can therefore occur only for $x=y=0$, and
the identifiability argument above then forces $\mathcal H=\mathcal J=0$.
\end{proof}

\section{Proof of the Exact Recovery Theorem}
\label{app:exact-proof}

Throughout this appendix, $C>0$ denotes a universal numerical constant whose
value may change from line to line unless otherwise specified.

\subsection{Random Sampling and Restricted Clean Observations}
\label{app:pub-random}

The sampling conditions of Theorem~\ref{thm:core-main} imply $q>0$; throughout
this section we therefore work with $0<q\le1$.  Write $\theta=p\gamma$.  When $q<1$, conditional on $\Lset$, each
coordinate in $\Lset^c$ is corrupted independently with probability
\[
\theta'=\frac{\theta}{1-q}.
\]
When $q=1$, equivalently $p=1,\gamma=0$, the clean set is the full
coordinate set and $\Gset=\varnothing$ deterministically.

\begin{lemma}[Equivalent sequential generation]
\label{lem:pub-sequential}
Fix $J\ge1$ and set
\[
q_0=1-(1-q)^{1/J}.
\]
If $q<1$, generate independent batches
$\Lset_j\sim\Ber(q_0)$ and set
$\Lset=\bigcup_{j=1}^J\Lset_j$.  Conditional on $\Lset$, generate
the corruption indicators on $\Lset^c$ independently with probability
$\theta'$, and then assign independent Rademacher signs.  If $q=1$, take
all $\Lset_j$ equal to the full coordinate set and
$\Gset=\varnothing$.  This construction has the same three-state law as
the model in Section~\ref{sec:core-prob}, and
\begin{equation}
q_0\ge q/J.
\label{eq:pub-q0-lower}
\end{equation}
\end{lemma}

\begin{proof}
For $q<1$, the union indicator is Bernoulli with success probability
$1-(1-q_0)^J=q$.  Conditional on failure of the clean event,
$(1-q)\theta'=\theta$, and the remaining probability is $1-p$.
Independence follows coordinatewise.  For $q=1$ the statement is
deterministic.  Concavity of $x^{1/J}$ on $[0,1]$ gives
$(1-q)^{1/J}\le1-q/J$, proving \eqref{eq:pub-q0-lower}.
\end{proof}

\begin{remark}[Degenerate branches]
The simplified theorem is stated for $1\le r<K$ and $s>0$.  If
$r=0$, the low-rank golfing component is absent.  If $r=K$, the associated
linear space $\Tsp$ is the full matrix space and the positive orthogonal-complement gain used in
the simplified theorem is no longer the relevant parameter.  If
$\gamma=0$, set $\mathcal S_0^{\mathrm{sgn}}=\mathcal Y_s=\bm{W}_s=0$ and omit the sparse certificate.
\end{remark}

For use in the concentration bounds, define
\[
\ell=\Max_{\bm{u}\in\Iset}\norm{\mathscr P_\V (\mathcal E_{\bm{u}})}_F^2
=\frac{\mu_\V s}{N},
\qquad
\mathscr M_{\Lset}
=
\left.\mathscr P_\V \mathscr P_\Lset \mathscr P_\V\right|_\V.
\]

\begin{lemma}[Restricted sampling concentration]
\label{lem:pub-restricted}
For every $0<\delta<1$,
\begin{equation}
\Pp\!\left(
\norm{q^{-1}\mathscr P_\V \mathscr P_\Lset \mathscr P_\V-\mathscr P_\V}_{\op}
\ge\delta
\right)
\le
2s\Exp\!\left(-\frac{3\delta^2q}{8\ell}\right).
\label{eq:pub-restricted-tail}
\end{equation}
On the complementary event,
\[
q(1-\delta)\mathscr I_{\V}
\preceq
\mathscr M_{\Lset}
\preceq
q(1+\delta)\mathscr I_{\V},
\]
and
\begin{equation}
\alpha\ge\sqrt{q(1-\delta)}.
\label{eq:pub-alpha-lower}
\end{equation}
\end{lemma}

\begin{proof}
Let $\mathcal V_{\bm{u}}=\mathscr P_\V(\mathcal E_{\bm{u}})$ and
$\eta_{\bm{u}}:=\1_{\{\bm{u}\in\Lset\}}\sim\Ber(q)$.  Then
\[
\mathscr P_\V \mathscr P_\Lset \mathscr P_\V-q\mathscr P_\V
=
\sum_{\bm{u}}(\eta_{\bm{u}}-q)(\mathcal V_{\bm{u}}\otimes \mathcal V_{\bm{u}}).
\]
Each summand has operator norm at most $\ell$, and the variance operator is
bounded by $q\ell \mathscr P_\V$.  Self-adjoint matrix Bernstein
 gives \eqref{eq:pub-restricted-tail}.  The eigenvalue bounds
follow immediately; for $\mathcal H\in\V$,
\[
\norm{\mathscr P_\Lset (\mathcal H)}_F^2
=
\ip{\mathcal H}{\mathscr M_{\Lset} (\mathcal H)}
\ge q(1-\delta)\norm{\mathcal H}_F^2,
\]
which proves \eqref{eq:pub-alpha-lower}.
\end{proof}

\subsection{Golfing Construction for the Low-Rank Certificate}
\label{app:pub-golf}

Let
\[
\begin{aligned}
\mathcal H_0&=\mathscr P_\V(\mathcal B_0),\\
\mathcal Y_j&=q_0^{-1}\mathscr P_{\Lset_j}(\mathcal H_{j-1}),\\
\mathcal H_j&=\mathcal H_{j-1}-\mathscr P_\V(\mathcal Y_j).
\end{aligned}
\]
and define
\[
\mathcal Y_b=\sum_{j=1}^J\mathcal Y_j,\qquad \mathcal R_J=\mathcal H_J.
\]
Because
$\mathscr P_\V (\mathcal Y_b)=\mathscr P_\V (\mathcal B_0)-\mathcal H_J$, the normal correction
\[
\bm{W}_b=\KT^\dagger(\mathscr P_{\V^\perp}(\mathcal Y_b-\mathcal B_0))
\]
satisfies
\[
\A^*(\bm{U}\bm{V}^\trans+\bm{W}_b)=\mathcal Y_b+\mathcal R_J.
\]

Set
\[
\mathscr Q_j
=
\mathscr P_\V-q_0^{-1}\mathscr P_\V \mathscr P_{\Lset_j}\mathscr P_\V.
\]

\begin{lemma}[Golfing contractions]
\label{lem:pub-golf-contract}
For $0<\rho_F<1$,
\[
\Pp(\norm{\mathscr Q_j}_{\op}>\rho_F)
\le
2s\Exp\!\left(-\frac{3\rho_F^2q_0}{8\ell}\right).
\]
Hence on the corresponding event
\[
\norm{\mathcal H_j}_F\le\rho_F^j\norm{\mathcal H_0}_F.
\]
Moreover, if $\mathcal H\in\V$ is independent of the fresh batch
$\Lset_j$, then for every $t_\infty>0$,
\begin{equation}
\norm{\mathscr Q_j(\mathcal H)}_\infty
\le
\left[
\sqrt{\frac{2\ell t_\infty}{q_0}}
+
\frac{2\ell t_\infty}{3q_0}
\right]\norm{\mathcal H}_\infty
\label{eq:pub-inf-contract}
\end{equation}
with probability at least $1-2Ne^{-t_\infty}$.
\end{lemma}

\begin{proof}
The Frobenius contraction follows from the same self-adjoint Bernstein estimate as Lemma~\ref{lem:pub-restricted}, with $q$ replaced by $q_0$.
For the infinity bound, let
$\eta_{j,\bm{v}}:=\1_{\{\bm{v}\in\Lset_j\}}\sim\Ber(q_0)$ denote the
fresh-batch indicators.  For each output coordinate $\bm{u}$,
\[
[\mathscr Q_j(\mathcal H)]_{\bm{u}}
=
\sum_{\bm{v}}\left(1-\frac{\eta_{j,\bm{v}}}{q_0}\right)
\mathcal H_{\bm{v}}\ip{\mathscr P_\V (\mathcal E_{\bm{u}})}{\mathscr P_\V (\mathcal E_{\bm{v}})}.
\]
Conditional on $\mathcal H$, the summands are independent and mean zero.  Their
variance is at most $\ell\norm{\mathcal H}_\infty^2/q_0$, and each summand is
bounded by $\ell\norm{\mathcal H}_\infty/q_0$.  Two-sided scalar Bernstein and a
union bound over $\bm{u}$ yield \eqref{eq:pub-inf-contract}.
\end{proof}

For the normal leakage define
\[
\bm{N}_{\bm{u}}=\KT^\dagger(\mathscr P_{\V^\perp}(\mathcal E_{\bm{u}})),\qquad
\nu_{\max}=\Max_{\bm{u}}\norm{\bm{N}_{\bm{u}}}_2,
\]
and
\[
\nu^2=
\Max\left\{
\norm{\sum_{\bm{u}}\bm{N}_{\bm{u}}\bm{N}_{\bm{u}}^\trans}_2,\,
\norm{\sum_{\bm{u}}\bm{N}_{\bm{u}}^\trans \bm{N}_{\bm{u}}}_2
\right\}.
\]
\begin{lemma}[Universal normal-atom bounds]
\label{lem:pub-normal-atom}
The normal atoms satisfy
\begin{equation}
\nu_{\max}\le\frac1{\beta_{\Tsp}},
\qquad
\nu^2\le\frac{N}{\beta_{\Tsp}^2}.
\label{eq:pub-normal-atom}
\end{equation}
\end{lemma}

\begin{proof}
By definition, $\beta_{\Tsp}$ is the smallest singular value of
$\KT^*:\V^\perp\to\Tsp^\perp$. Since $\beta_{\Tsp}>0$, the operator
$\KT^*$ is injective on $\V^\perp$. In finite dimensions,
$\range(\KT)=\Ker(\KT^*)^\perp=\V^\perp$, and therefore
\[
\norm{\KT^\dagger}_{F\to F}
=
\frac1{\sigma_{\min}(\KT^*)}
\le\frac1{\beta_{\Tsp}}.
\]
This proves $\nu_{\max}\le1/\beta_{\Tsp}$.
Let
$\mathscr J=\KT^\dagger \mathscr P_{\V^\perp}$.
For $\bm{x}\in\RR^N$,
\[
\bm{x}^\trans\left(\sum_{\bm{u}}\bm{N}_{\bm{u}}\bm{N}_{\bm{u}}^\trans\right)\bm{x}
=
\sum_{j=1}^K
\norm{\mathscr J^*(\bm{x}\bm{e}_j^\trans)}_F^2
\le
\frac K{\beta_{\Tsp}^2}\norm{\bm{x}}_2^2.
\]
The right variance sum is bounded analogously by
$N/\beta_{\Tsp}^2$.  Since $K\le N$, taking the larger dimension proves
\eqref{eq:pub-normal-atom}.
\end{proof}

\begin{lemma}[Fresh-batch normal leakage]
\label{lem:pub-normal-leak}
Conditional on the preceding batches, for every $t_2>0$,
\begin{align}
\norm{\KT^\dagger(\mathscr P_{\V^\perp}(\mathcal Y_j))}_2
\le{}&
\left[
\sqrt{\frac{2\nu^2t_2}{q_0}}
+
\frac{2\nu_{\max}t_2}{3q_0}
\right]
\norm{\mathcal H_{j-1}}_\infty
\label{eq:pub-normal-leak}
\end{align}
with probability at least $1-2(N+K)e^{-t_2}$.
\end{lemma}

\begin{proof}
Conditional on the preceding batches, $\mathcal H_{j-1}\in\V$ is fixed and
independent of the fresh Bernoulli indicators
$\eta_{j,\bm{u}}:=\1_{\{\bm{u}\in\Lset_j\}}\sim\Ber(q_0)$.  Since
\[
\sum_{\bm{u}} \mathcal H_{j-1,\bm{u}}\bm{N}_{\bm{u}}
=
\KT^\dagger(\mathscr P_{\V^\perp}(\mathcal H_{j-1}))=0,
\]
\[
\KT^\dagger(\mathscr P_{\V^\perp}(\mathcal Y_j))
=
\sum_{\bm{u}}
\left(\frac{\eta_{j,\bm{u}}}{q_0}-1\right)
\mathcal H_{j-1,\bm{u}}\bm{N}_{\bm{u}}.
\]
The summands are independent and mean zero.  Their spectral norms are
bounded by
$\nu_{\max}\norm{\mathcal H_{j-1}}_\infty/q_0$, and the two rectangular variance
sums are bounded by
$\nu^2\norm{\mathcal H_{j-1}}_\infty^2/q_0$.  Rectangular matrix Bernstein
 gives \eqref{eq:pub-normal-leak}.
\end{proof}

\begin{lemma}[Bounds for the low-rank certificate]
\label{lem:pub-lowrank}
Assume the infinity contraction factor in
\eqref{eq:pub-inf-contract} is at most $\rho_\infty<1$.  On the golfing
events,
\begin{equation}
a_\infty:=\norm{\mathcal Y_b}_\infty
\le
\frac{\norm{\mathcal H_0}_\infty}{q_0(1-\rho_\infty)},
\label{eq:pub-ainf}
\end{equation}
\begin{equation}
a_2:=\norm{\bm{W}_b}_2
\le
c_{\mathrm{ref}}
+
\frac{\Xi(q_0,t_2)}{1-\rho_\infty}
\norm{\mathcal H_0}_\infty,
\label{eq:pub-a2}
\end{equation}
where
\[
\Xi(q_0,t_2)
=
\sqrt{\frac{2\nu^2t_2}{q_0}}
+
\frac{2\nu_{\max}t_2}{3q_0},
\]
and
\begin{equation}
\rho_J:=\norm{\mathcal R_J}_F
\le
\rho_F^J\norm{\mathcal H_0}_F.
\label{eq:pub-rhoJ}
\end{equation}
\end{lemma}

\begin{proof}
Sum the geometric infinity contractions for the $\mathcal Y_j$ terms and the
geometric normal-leakage bounds for the corresponding
$\KT^\dagger(\mathscr P_{\V^\perp}(\mathcal Y_j))$ terms.  The deterministic term
$-\KT^\dagger(\mathscr P_{\V^\perp}(\mathcal B_0))$ contributes $c_{\mathrm{ref}}$.  The residual
bound is the Frobenius contraction.
\end{proof}

\subsection{Sparse-Certificate Interpolation and Concentration}
\label{app:pub-sparse}

When $\gamma=0$, take the zero sparse certificate.  Assume
$\gamma>0$ below.  On the restricted-sampling event set
\[
\begin{aligned}
\mathcal Z_s&=(\mathscr M_{\Lset}^{-1}\mathscr P_\V)
(\mathcal S_0^{\mathrm{sgn}}),\\
\mathcal Y_s&=\mathscr P_\Lset(\mathcal Z_s),\\
\bm W_s&=\KT^\dagger(\mathcal S_0^{\mathrm{sgn}}-\mathcal Y_s).
\end{aligned}
\]
Then
\[
\mathscr P_\V(\mathcal Y_s)=\mathscr P_\V(\mathcal S_0^{\mathrm{sgn}}),\qquad
\A^*(\bm{W}_s)=\mathcal S_0^{\mathrm{sgn}}-\mathcal Y_s.
\]

\begin{lemma}[Sparse infinity bound]
\label{lem:pub-sparse-inf}
For every $t_s>0$, conditional on $\Lset$, with probability at least
$1-2Ne^{-t_s}$,
\begin{equation}
b_\infty:=\norm{\mathcal Y_s}_\infty
\le
\frac{\sqrt{2\theta'\ell t_s}}{q(1-\delta)}
+
\frac{2\ell t_s}{3q(1-\delta)}.
\label{eq:pub-binf}
\end{equation}
\end{lemma}

\begin{proof}
For $\bm{i}\in\Lset$, set
$\mathcal W_{\bm{i}}=(\mathscr M_{\Lset}^{-1}\mathscr P_\V)(\mathcal E_{\bm{i}})$.  Then
\[
(\mathcal Y_s)_{\bm{i}}
=
\sum_{\bm{u}\notin\Lset}
\zeta_{\bm{u}}\ip{\mathcal W_{\bm{i}}}{\mathscr P_\V (\mathcal E_{\bm{u}})},
\]
where the $\zeta_{\bm{u}}$ are conditionally independent sparse Rademacher
variables with variance $\theta'$.  On the restricted-sampling event,
\[
\norm{\mathcal W_{\bm{i}}}_F\le\frac{\sqrt\ell}{q(1-\delta)}.
\]
Hence the conditional variance is at most
$\theta'\ell/[q^2(1-\delta)^2]$, and each summand is bounded by
$\ell/[q(1-\delta)]$.  Scalar Bernstein and a union bound prove the
claim.
\end{proof}

Because $\Gset\cap\Lset=\varnothing$,
\begin{equation}
\norm{\mathcal S_0^{\mathrm{sgn}}-\mathcal Y_s}_F^2
=
\abs\Gset+\norm{\mathcal Y_s}_F^2,
\qquad
\norm{\mathcal Y_s}_F^2
\le
\frac{\norm{\mathscr P_\V (\mathcal S_0^{\mathrm{sgn}})}_F^2}{q(1-\delta)}.
\label{eq:pub-qs}
\end{equation}

\begin{lemma}[Random-sign projection onto $\V$]
\label{lem:pub-rad-proj}
For every $t_s\ge1$, with probability at least
$1-(2s+3)e^{-t_s}$,
\[
\abs\Gset\le C(\theta N+t_s)
\]
and
\begin{equation}
\norm{\mathscr P_\V (\mathcal S_0^{\mathrm{sgn}})}_F^2
\le
C\left[\theta(s+t_s)+\ell t_s^2\right].
\label{eq:pub-proj}
\end{equation}
\end{lemma}

\begin{proof}
Let $\mathcal V_{\bm{u}}=\mathscr P_\V (\mathcal E_{\bm{u}})$.  Conditional on the support $\Gset$,
\[
\mathscr P_\V (\mathcal S_0^{\mathrm{sgn}})=\sum_{\bm{u}\in\Gset}\varepsilon_{\bm{u}}\mathcal V_{\bm{u}},
\]
where the $\varepsilon_{\bm{u}}$ are independent Rademacher signs.  Define
\[
\varphi(\boldsymbol{\varepsilon})
=
\norm{\sum_{\bm{u}\in\Gset}\varepsilon_{\bm{u}}\mathcal V_{\bm{u}}}_F.
\]
The map $\varphi$ is convex.  Moreover, for two sign vectors
$\boldsymbol{\varepsilon},\boldsymbol{\varepsilon}'$,
\[
\abs{\varphi(\boldsymbol{\varepsilon})-\varphi(\boldsymbol{\varepsilon}')}
\le
\norm{\sum_{\bm{u}\in\Gset}(\varepsilon_{\bm{u}}-\varepsilon'_{\bm{u}})\mathcal V_{\bm{u}}}_F
\le
\sqrt{\rho_\Gset}\,
\norm{\boldsymbol{\varepsilon}-\boldsymbol{\varepsilon}'}_2,
\]
where
\[
\rho_\Gset
=
\norm{\sum_{\bm{u}\in\Gset}\mathcal V_{\bm{u}}\otimes \mathcal V_{\bm{u}}}_{\op}
=
\norm{\mathscr P_\V \mathscr P_\Gset \mathscr P_\V}_{\op}.
\]
Also,
\[
\E_{\boldsymbol{\varepsilon}}\varphi(\boldsymbol{\varepsilon})^2
=
\sum_{\bm{u}\in\Gset}\norm{\mathcal V_{\bm{u}}}_F^2
=:\tau_\Gset,
\]
so Jensen's inequality gives
$\E_{\boldsymbol{\varepsilon}}\varphi(\boldsymbol{\varepsilon})\le\sqrt{\tau_\Gset}$.
Convex-Lipschitz concentration on the Rademacher cube
therefore implies, conditionally on $\Gset$, for a universal numerical
constant $C_R>0$,
\[
\norm{\mathscr P_\V (\mathcal S_0^{\mathrm{sgn}})}_F
\le
\sqrt{\tau_\Gset}
+
C_R\sqrt{\rho_\Gset t_s}
\]
except on an event of probability at most $2e^{-t_s}$.

It remains to bound the support-dependent quantities.  Since
\[
\sum_{\bm{u}}\norm{\mathcal V_{\bm{u}}}_F^2=\Tr(\mathscr P_\V)=s,
\qquad
\norm{\mathcal V_{\bm{u}}}_F^2\le\ell,
\]
scalar Bernstein gives
\[
\tau_\Gset
\le
\theta s+\sqrt{2\theta\ell s t_s}
+\frac{2\ell t_s}{3}
\]
with failure probability at most $e^{-t_s}$.  Likewise,
self-adjoint matrix Bernstein applied to
\[
\sum_{\bm{u}}(\1_{\{\bm{u}\in\Gset\}}-\theta)(\mathcal V_{\bm{u}}\otimes \mathcal V_{\bm{u}})
\]
gives
\[
\rho_\Gset
\le
\theta+\sqrt{2\theta\ell t_s}
+\frac{2\ell t_s}{3}
\]
with failure probability at most $s e^{-t_s}$ after enlarging the
universal constants.  Since $t_s\ge1$, squaring the conditional
Rademacher bound and repeatedly using $2\sqrt{ab}\le a+b$ yields
\[
\norm{\mathscr P_\V (\mathcal S_0^{\mathrm{sgn}})}_F^2
\le
C\left[\theta(s+t_s)+\ell t_s^2\right].
\]
Finally, scalar Bernstein for
$\abs\Gset\sim\Bin(N,\theta)$ gives
$\abs\Gset\le C(\theta N+t_s)$ with failure probability at most
$e^{-t_s}$.  A union bound is dominated by the stated
$(2s+3)e^{-t_s}$ failure probability.
\end{proof}

\begin{lemma}[Sparse spectral bound]
\label{lem:pub-sparse-spec}
On the events above,
\begin{equation}
b_2:=\norm{\bm{W}_s}_2
\le
\frac C{\beta_{\Tsp}}
\sqrt{
\theta N+t_s+
\frac{\theta(s+t_s)+\ell t_s^2}{q(1-\delta)}
}.
\label{eq:pub-b2}
\end{equation}
\end{lemma}

\begin{proof}
Use
\[
\norm{\bm{W}_s}_2\le\norm{\bm{W}_s}_F
\le\frac{\norm{\mathcal S_0^{\mathrm{sgn}}-\mathcal Y_s}_F}{\beta_{\Tsp}}
\]
and substitute \eqref{eq:pub-qs} and
\eqref{eq:pub-proj}.
\end{proof}

Combine the two certificate components as
\[
\bm{W}=\bm{W}_b+\lambda \bm{W}_s,\qquad
\mathcal F=\lambda^{-1}\mathcal Y_b-\mathcal Y_s.
\]
Then
\[
\A^*(\bm{U}\bm{V}^\trans+\bm{W})
=
\lambda(\mathcal S_0^{\mathrm{sgn}}+\mathcal F)+\mathcal R_J,
\]
and
\begin{align*}
\norm{\mathscr P_\Lset (\mathcal F)}_\infty
&\le
\frac{\norm{\mathcal Y_b}_\infty}{\lambda}+\norm{\mathcal Y_s}_\infty,\\
\norm{\bm{W}}_2
&\le
\norm{\bm{W}_b}_2+\lambda\norm{\bm{W}_s}_2.
\end{align*}

To distinguish the realized random certificate norms from their deterministic high-probability envelopes, let
$\bar a_\infty,\bar a_2,\bar\rho_J,\bar b_\infty,\bar b_2$
denote the right-hand sides of
\eqref{eq:pub-ainf}, \eqref{eq:pub-a2},
\eqref{eq:pub-rhoJ}, \eqref{eq:pub-binf}, and
\eqref{eq:pub-b2}, respectively.  On the simultaneous
concentration event used below,
\[
\norm{\mathcal Y_b}_\infty\le\bar a_\infty,\qquad
\norm{\bm{W}_b}_2\le\bar a_2,\qquad
\norm{\mathcal R_J}_F\le\bar\rho_J,
\]
and
\[
\norm{\mathcal Y_s}_\infty\le\bar b_\infty,\qquad
\norm{\bm{W}_s}_2\le\bar b_2.
\]
On the same restricted-sampling event, Lemma~\ref{lem:pub-restricted}
provides the deterministic lower bound
\begin{equation}
\underline\alpha:=
\sqrt{q(1-\delta)}
\le \alpha.
\label{eq:pub-alpha-envelope}
\end{equation}
We use $\underline\alpha$ in the sufficient conditions below so that
they depend only on $q,\delta$, and the model parameters.  At the endpoint $q=1$,
the clean set is the full coordinate set and in fact $\alpha=1$; the
displayed envelope remains a valid conservative lower bound.

\subsection{Instance-Dependent Recovery Criterion}
\label{app:pub-instance}

\begin{theorem}[Instance-dependent recovery criterion]
\label{thm:pub-instance}
Fix
$\delta,\rho_F,\rho_\infty\in(0,1)$,
$J\in\mathbb N$, $t_\infty,t_2>0$, and $t_s\ge1$.  Suppose the factor in
\eqref{eq:pub-inf-contract} is at most $\rho_\infty$.  If
\begin{equation}
\frac{\bar a_\infty}{\lambda}+\bar b_\infty
+\frac{\bar\rho_J}{\underline\alpha\,\lambda}<1
\label{eq:pub-instance-inf}
\end{equation}
and
\begin{equation}
\bar a_2+\lambda\bar b_2
+\frac{\bar\rho_J}{\underline\alpha\,\beta_{\Tsp}}<1,
\label{eq:pub-instance-spec}
\end{equation}
then, on the simultaneous concentration event, the certificate satisfies
all hypotheses of Proposition~\ref{thm:core-approx}; consequently
$(\mathcal L_0,\mathcal S_0)$ is the unique minimizer.  That event has probability at least
$1-\varepsilon_{\mathrm{tot}}$, where
\begin{align}
\varepsilon_{\mathrm{tot}}\le{}&
2s\Exp\!\left(-\frac{3\delta^2q}{8\ell}\right)\nonumber\\
&+
2Js\Exp\!\left(-\frac{3\rho_F^2q_0}{8\ell}\right)\nonumber\\
&+
2JNe^{-t_\infty}
+
2J(N+K)e^{-t_2}\nonumber\\
&+
(2N+2s+4)e^{-t_s}.
\label{eq:pub-total-failure}
\end{align}
\end{theorem}

\begin{proof}
On the simultaneous concentration event, the realized certificate norms
are bounded by the deterministic envelopes introduced above and
$\alpha\ge\underline\alpha$ by
\eqref{eq:pub-alpha-envelope}.  Hence the two displayed deterministic
inequalities imply
\[
\frac{\norm{\mathcal Y_b}_\infty}{\lambda}+\norm{\mathcal Y_s}_\infty
+\frac{\norm{\mathcal R_J}_F}{\alpha\lambda}<1
\]
and
\[
\norm{\bm{W}_b}_2+\lambda\norm{\bm{W}_s}_2
+\frac{\norm{\mathcal R_J}_F}{\alpha\beta_{\Tsp}}<1.
\]
These are precisely the effective infinity and spectral margins required by Proposition~\ref{thm:core-approx}.  The probability bound is the
union of the restricted-sampling event, the $J$ Frobenius contractions,
the $J$ infinity contractions, the $J$ rectangular matrix-Bernstein
events, and the sparse-sign events.  The fresh-batch estimates are
conditional, but Lemma~\ref{lem:pub-sequential} places them in one
probability space.  The tower property and a union bound suffice; the
final events need not be independent.
\end{proof}

\subsection{Completion of the High-Probability Recovery Proof}
\label{app:pub-closure}

Let
\[
d_\Sigma=N+K,\qquad \ell_{NK}=\Log(ed_\Sigma),
\]
and choose
\[
\delta=\rho_F=\rho_\infty=\frac12,\qquad
t_\infty=t_2=t_s=c_t\ell_{NK},
\]
\[
J=J_0+\lceil8\ell_{NK}\rceil,\qquad
q_0=1-(1-q)^{1/J}.
\]

\begin{lemma}[Residual bound]
\label{lem:pub-residual-budget}
Assume $\chi_{\Tsp}\ge\chi_0$,
$q N\ge C_m\ell_{NK}^2$, and $\gamma\le1/2$.
After increasing the fixed integer $J_0$ if necessary,
\begin{equation}
\frac{\bar\rho_J}{\sqrt{q/2}\,\lambda}\le\frac1{32},
\qquad
\frac{\bar\rho_J}{\sqrt{q/2}\,\beta_{\Tsp}}\le\frac1{32}.
\label{eq:pub-residual-budget}
\end{equation}
\end{lemma}

\begin{proof}
Since $\rho_F=1/2$,
$\bar\rho_J\le2^{-J}\norm{\mathcal H_0}_F$.
Moreover
$\norm{\mathcal H_0}_F\le\sqrt{d_{\max}r}$,
$\beta_{\Tsp}^2\ge\chi_0d_{\max}$,
$r\le K$, and
$p/q=1/(1-\gamma)\le2$.  Consequently,
\[
\frac{\bar\rho_J}{\sqrt{q/2}\,\lambda}
\le C\,2^{-J}N,
\qquad
\frac{\bar\rho_J}{\sqrt{q/2}\,\beta_{\Tsp}}
\le C\,2^{-J}\frac{N}{\ell_{NK}},
\]
where $C$ depends only on the fixed lower bound $\chi_0$, on
$c_\lambda$, and on $C_m$.  Since $N\le d_\Sigma$ and
$J=J_0+\lceil8\ell_{NK}\rceil$, increasing the fixed integer $J_0$ makes both
quantities at most $1/32$.
\end{proof}

\begin{lemma}[Budget bounds for the low-rank certificate]
\label{lem:pub-lowrank-budget}
There are universal constants
$C_\infty,C_{2,1},C_{2,2}>0$ such that
\begin{align}
\frac{\bar a_\infty}{\lambda}
&\le
\frac{C_\infty \ell_{NK}}{c_\lambda}
\sqrt{
\frac{\mu_b r p}
{\chi_{\Tsp}q^2K}},
\label{eq:pub-ainf-budget}\\
\bar a_2
&\le
c_{\mathrm{ref}}
+
C_{2,1}\ell_{NK}
\sqrt{
\frac{\mu_b r}
{\chi_{\Tsp}qK}}\nonumber\\
&\quad+
C_{2,2}\frac{\ell_{NK}^2}{q}
\sqrt{
\frac{\mu_b r}
{\chi_{\Tsp}NK}}.
\label{eq:pub-a2-budget}
\end{align}
Under the rank condition \eqref{eq:core-rank},
the right-hand sides are bounded by
\[
\frac{C_\infty}{c_\lambda}\sqrt{c_r}
\]
and
\[
c_{\mathrm{ref}}
+
C_{2,1}\sqrt{c_r}
+
C_{2,2}\sqrt{c_r/C_m},
\]
respectively.
\end{lemma}

\begin{proof}
Use \eqref{eq:pub-q0-lower},
$J=O(\ell_{NK})$,
\[
\norm{\mathcal H_0}_\infty^2
\le
\frac{\mu_b r d_{\max}}{NK},
\qquad
\beta_{\Tsp}^2=\chi_{\Tsp}d_{\max},
\]
together with \eqref{eq:pub-ainf},
\eqref{eq:pub-a2}, and
\eqref{eq:pub-normal-atom}.  The square-root and linear matrix-Bernstein
terms give the two terms in \eqref{eq:pub-a2-budget}.  The rank condition
\[
r\le c_r\frac{\chi_{\Tsp}q^2K}{\mu_b p\ell_{NK}^2}
\]
immediately yields
\[
\frac{\ell_{NK}}{c_\lambda}
\sqrt{\frac{\mu_b r p}{\chi_{\Tsp}q^2K}}
\le
\frac{\sqrt{c_r}}{c_\lambda}
\]
and, since $q\le p$,
\[
\ell_{NK}\sqrt{\frac{\mu_b r}{\chi_{\Tsp}qK}}
\le
\sqrt{c_r}\sqrt{\frac qp}
\le
\sqrt{c_r}.
\]
For the linear Bernstein term,
\[
\frac{\ell_{NK}^2}{q}
\sqrt{\frac{\mu_b r}{\chi_{\Tsp}NK}}
\le
\sqrt{c_r}\frac{\ell_{NK}}{\sqrt{pN}}
\le
\sqrt{\frac{c_r}{C_m}},
\]
because $pN\ge qN\ge C_m\ell_{NK}^2$.
\end{proof}

\begin{lemma}[Bounds for the sparse certificate]
\label{lem:pub-sparse-budget}
Assume $\gamma\le1/2$ and
$qN\ge C_s\mu_\V s\ell_{NK}^2$.  There are universal constants
$C_{s,1},C_{s,2}>0$ such that
\[
\bar b_\infty
\le
C_{s,1}
\left[
\sqrt{\frac{\ell \ell_{NK}}{q}}
+
\frac{\ell \ell_{NK}}{q}
\right],
\]
and
\begin{equation}
\lambda \bar b_2
\le
C_{s,2}c_\lambda
\left[
\sqrt\gamma+
\sqrt{\frac{\ell_{NK}+1}{pN}}
\right].
\label{eq:pub-sparse-budget}
\end{equation}
\end{lemma}

\begin{proof}
If $\gamma=0$, both quantities vanish.  Otherwise,
$q=p(1-\gamma)$ and $p\le1$ imply
\[
1-q=1-p(1-\gamma)\ge\gamma.
\]
Hence, for $\gamma\le1/2$,
\[
\frac{\theta'}q
=
\frac{\gamma}{(1-\gamma)(1-q)}
\le
\frac1{1-\gamma}
\le2,
\]
which gives the infinity estimate from
\eqref{eq:pub-binf}.  The sample condition and $\mu_\V\ge1$ imply
\[
s\le\frac{qN}{C_s\ell_{NK}^2},
\qquad
\frac{\ell \ell_{NK}^2}{q}\le\frac1{C_s}.
\]
For the spectral term, multiply \eqref{eq:pub-b2} by
$\lambda=c_\lambda\beta_{\Tsp}/\sqrt{pN}$.  The leading contribution satisfies
\[
\frac{\theta N}{pN}=\gamma,
\]
while
\[
\frac{\theta}{q}
=
\frac{\gamma}{1-\gamma}
\le2\gamma.
\]
Using the two sampling bounds above, the remaining terms inside the square
root are bounded, up to universal constants, by
\[
\frac{\ell_{NK}+1}{pN}
+
\frac{\gamma}{\ell_{NK}}
\le
C\left(\gamma+\frac{\ell_{NK}+1}{pN}\right).
\]
Taking square roots and enlarging the universal constant gives
\eqref{eq:pub-sparse-budget}.
\end{proof}

\begin{proof}[Proof of Theorem~\ref{thm:core-main}]
Choose the constants in the following order.  First choose $c_t$ large enough
to make all terms with factors $e^{-t_\infty}$,
$e^{-t_2}$, and $e^{-t_s}$ polynomially small.  Next choose $C_m$
large enough for the finite-dimensional term in
\eqref{eq:pub-sparse-budget}, and choose $J_0$ large enough for
Lemma~\ref{lem:pub-residual-budget}.  Then choose $C_s$ large enough
for restricted sampling, golfing, the infinity contraction, and the sparse
infinity budget.  Choose $c_\gamma$ small enough for the
$\sqrt\gamma$ term, and finally choose $c_r$ small enough for the
low-rank budgets.

With these choices,
\[
\frac{\bar a_\infty}{\lambda}\le\frac18,\qquad
\bar b_\infty\le\frac18,
\]
\[
\bar a_2\le c_{\mathrm{ref}}+\frac18,\qquad
\lambda \bar b_2\le\frac18,
\]
and Lemma~\ref{lem:pub-residual-budget}, together with
$\underline\alpha=\sqrt{q/2}\le\alpha$, gives
\[
\frac{\bar\rho_J}{\underline\alpha\,\lambda}\le\frac1{32},
\qquad
\frac{\bar\rho_J}{\underline\alpha\,\beta_{\Tsp}}\le\frac1{32}.
\]
Hence
\[
\frac{\bar a_\infty}{\lambda}+\bar b_\infty
+\frac{\bar\rho_J}{\underline\alpha\,\lambda}
\le\frac9{32}<1,
\]
while
\[
\bar a_2+\lambda \bar b_2+
\frac{\bar\rho_J}{\underline\alpha\,\beta_{\Tsp}}
\le
c_0+\frac9{32}
<
\frac{25}{32}
<1,
\]
because $c_0<1/2$.  Thus
Theorem~\ref{thm:pub-instance} applies.

For the probability, \eqref{eq:pub-q0-lower} and
$qN\ge C_s\mu_\V s\ell_{NK}^2$ imply
\[
\frac{q_0}{\ell}
\ge
\frac{qN}{J\mu_\V s}
\gtrsim C_s\ell_{NK}.
\]
Since $c_t$ is fixed before $C_s$, increasing $C_s$ makes the bracket in
\eqref{eq:pub-inf-contract} at most $1/2$ and makes the golfing exponent
dominate the polynomial prefactor.  Substituting
$t_\infty=t_2=t_s=c_t\ell_{NK}$ and $J=O(\ell_{NK})$ into
\eqref{eq:pub-total-failure} yields
\[
\varepsilon_{\mathrm{tot}}
\le
C_1d_\Sigma^{-C_2}.
\]
This proves \eqref{eq:core-probability}.
\end{proof}

\section{Reflection-Specific Bounds}
\label{app:reflection-bounds}

\subsection{Reflection Multiplicity and Collision Geometry}
\label{app:pub-structural}

Let
\[
\B=\A \Dop^{-1/2}.
\]
Then $\B^*\B=\Iop$, so $\B$ is an isometry from the original tensor space
onto the lifting range.  Let
\[
\kappa_D=\frac{d_{\max}}{d_{\min}},
\qquad
\boldsymbol{\Psi}_{\bm{u}}=\B (\mathcal E_{\bm{u}})=\frac{\A(\mathcal E_{\bm{u}})}{\sqrt{d_{\bm{u}}}}.
\]

Under the standing admissible-kernel condition
\[
m_a=1\Rightarrow k_a=1,
\qquad
m_a\ge2\Rightarrow 1\le k_a\le m_a,
\]
one automatically has $k_a\le2m_a-2$ in every non-singleton mode.
Let $d_+=\#\{a:m_a>1\}$.

For a matrix $\bm{Z}$, let $\degR(\bm{Z})$ denote the maximum number of nonzero
entries in a row and let $\degC(\bm{Z})$ denote the maximum number of nonzero
entries in a column.  For $\bm{Z}=\A(\mathcal E_{\bm{u}})$, these quantities measure how many
copies of coordinate $\bm{u}$ can collide within one lifted row or column.

\begin{proposition}[Reflection collision degrees]
\label{prop:pub-collision}
For every reflection atom,
\[
\degR(\A(\mathcal E_{\bm{u}}))\le2^{d_+},\qquad
\degC(\A(\mathcal E_{\bm{u}}))\le2^{d_+}.
\]
\end{proposition}

\begin{proof}
Fix a non-singleton mode $a$. On the finite interval
$1\le j\le2m_a-1$ used in the definition of $\phi_a$, each reflected
coordinate has at most two preimages because $\phi_a$ increases to $m_a$
and then decreases. For a fixed row index $i_a$, the values
$j=i_a+t_a$ with $0\le t_a\le k_a-1$ therefore contain a prescribed
reflected coordinate at most twice. Likewise, for a fixed offset $t_a$,
the values $j=i_a+t_a$ with $1\le i_a\le m_a$ contain it at most twice.
Thus each one-dimensional row or column has at most two nonzero entries
belonging to the same reflection atom. Singleton modes contribute a factor
of one, and the row/column degrees multiply across modes, giving the bound
$2^{d_+}$.
\end{proof}

\subsection{Leverage Bounds under Multiplicity Distortion}
\label{app:pub-leverage-sec}

\begin{lemma}[Leverage bound under multiplicity distortion]
\label{lem:pub-leverage}
One has
\[
\ell
\le
\kappa_D
\Max_{\bm{u}}\norm{\mathscr P_{\Tsp}(\boldsymbol{\Psi}_{\bm{u}})}_F^2.
\]
If
\[
\Max_i\norm{\bm{U}^\trans \bm{e}_i}_2^2\le\frac{\mu_U r}{N},
\qquad
\Max_j\norm{\bm{V}^\trans \bm{e}_j}_2^2\le\frac{\mu_R r}{K},
\]
then
\[
\ell
\le
2^{d_+}\kappa_D r
\left(\frac{\mu_U}{N}+\frac{\mu_R}{K}\right).
\]
\end{lemma}

\begin{proof}
Let $\Sset=\range(\B)\cap\Tsp$. Because
$\B=\A\Dop^{-1/2}$ and $\B^*\B=\Iop$, a tensor $\mathcal H$ belongs to $\V$ if and only if
$\B\Dop^{1/2}(\mathcal H)\in\Sset$. Let
$\mathbb W=\B^*\Sset\subseteq\Hh$. In vectorized tensor coordinates, choose
$\bm Q\in\RR^{N\times\Dim(\mathbb W)}$ with orthonormal columns spanning
$\vecop(\mathbb W)$, and let $\bm e_{\bm u}\in\RR^N$ denote the coordinate
vector corresponding to $\mathcal E_{\bm u}$. Then
$\vecop(\V)=\bm D^{-1/2}\range(\bm Q)$, where
$\bm D=\diag(\vecop(\mathcal D))$. With
$\bm C=\bm D^{-1/2}\bm Q$, the standard projector formula gives
\[
\bm P_{\vecop(\V)}
=
\bm D^{-1/2}\bm Q
(\bm Q^\trans \bm D^{-1}\bm Q)^{-1}
\bm Q^\trans \bm D^{-1/2}.
\]
Since $\bm D^{-1}\succeq d_{\max}^{-1}\bm I$,
\[
\bm Q^\trans \bm D^{-1}\bm Q\succeq d_{\max}^{-1}\bm I,
\qquad
(\bm Q^\trans \bm D^{-1}\bm Q)^{-1}\preceq d_{\max}\bm I.
\]
It follows that
\[
\norm{\mathscr P_\V(\mathcal E_{\bm u})}_F^2
\le
\frac{d_{\max}}{d_{\bm u}}\norm{\bm Q^\trans\bm e_{\bm u}}_2^2.
\]
Because the isometry $\B$ maps $\mathbb W=\B^*\Sset$ onto $\Sset$,
\begin{align*}
\norm{\bm Q^\trans\bm e_{\bm u}}_2
&=
\norm{\mathscr P_{\mathbb W}(\mathcal E_{\bm u})}_F\\
&=
\norm{\mathscr P_{\Sset}(\B(\mathcal E_{\bm u}))}_F
=
\norm{\mathscr P_{\Sset}(\boldsymbol{\Psi}_{\bm u})}_F\\
&\le
\norm{\mathscr P_{\Tsp}(\boldsymbol{\Psi}_{\bm u})}_F.
\end{align*}
Therefore
\[
\norm{\mathscr P_\V (\mathcal E_{\bm{u}})}_F^2
\le
\frac{d_{\max}}{d_{\bm{u}}}
\norm{\mathscr P_{\Tsp}(\boldsymbol{\Psi}_{\bm{u}})}_F^2
\le
\kappa_D\norm{\mathscr P_{\Tsp}(\boldsymbol{\Psi}_{\bm{u}})}_F^2,
\]
and maximizing over $\bm{u}$ proves the first claim.

For the coherence bound, use
\begin{align*}
\norm{\mathscr P_{\Tsp}(\bm{Z})}_F^2
&=
\norm{\bm{U}^\trans \bm{Z}}_F^2+
\norm{\bm{Z}\bm{V}}_F^2-
\norm{\bm{U}^\trans \bm{Z}\bm{V}}_F^2\\
&\le
\norm{\bm{U}^\trans \bm{Z}}_F^2+\norm{\bm{Z}\bm{V}}_F^2.
\end{align*}
Every nonzero entry of $\boldsymbol{\Psi}_{\bm{u}}$ has magnitude $d_{\bm{u}}^{-1/2}$.
If $c_j$ denotes the number of nonzeros in column $j$, then
Cauchy--Schwarz and the left singular-space coherence give
\[
\norm{\bm{U}^\trans(\boldsymbol{\Psi}_{\bm{u}})_{:,j}}_2^2
\le
\frac{c_j}{d_{\bm{u}}}
\sum_{i\in\supp((\boldsymbol{\Psi}_{\bm{u}})_{:,j})}\norm{\bm{U}^\trans \bm{e}_i}_2^2
\le
\frac{c_j^2}{d_{\bm{u}}}\frac{\mu_U r}{N}.
\]
Proposition~\ref{prop:pub-collision} gives
$c_j\le2^{d_+}$ and $\sum_jc_j=d_{\bm{u}}$, hence
\[
\norm{\bm{U}^\trans\boldsymbol{\Psi}_{\bm{u}}}_F^2
\le
2^{d_+}\frac{\mu_U r}{N}.
\]
Applying the same argument to the row counts yields
\[
\norm{\boldsymbol{\Psi}_{\bm{u}}\bm{V}}_F^2
\le
2^{d_+}\frac{\mu_R r}{K}.
\]
Combining the last three displays with the first claim proves
\[
\ell
\le
2^{d_+}\kappa_D r
\left(\frac{\mu_U}{N}+\frac{\mu_R}{K}\right).
\]
\end{proof}

\subsection{Lifted Normal-Angle Bound}
\label{app:pub-angle-sec}

\begin{proposition}[Lifted normal-angle bound]
\label{prop:pub-angle}
Let
\[
c_{\A,\Tsp}
=
\Inf_{\substack{\mathcal H\in\V^\perp\\\mathcal H\ne0}}
\frac{\norm{\mathscr P_{\Tsp^\perp}(\A(\mathcal H))}_F}{\norm{\A(\mathcal H)}_F}.
\]
Then
\[
\beta_{\Tsp}\ge\sqrt{d_{\min}}\,c_{\A,\Tsp},
\qquad
\chi_{\Tsp}\ge\frac{c_{\A,\Tsp}^2}{\kappa_D}.
\]
\end{proposition}

\begin{proof}
For $\mathcal H\in\V^\perp$ with $\norm{\mathcal H}_F=1$,
\[
\norm{\mathscr P_{\Tsp^\perp}(\A(\mathcal H))}_F
\ge
c_{\A,\Tsp}\norm{\A(\mathcal H)}_F
\ge
c_{\A,\Tsp}\sqrt{d_{\min}}.
\]
Taking the infimum proves the first inequality.  Squaring it and dividing
by $d_{\max}$ gives
\[
\chi_{\Tsp}=\frac{\beta_{\Tsp}^2}{d_{\max}}
\ge
\frac{d_{\min}}{d_{\max}}c_{\A,\Tsp}^2
=
\frac{c_{\A,\Tsp}^2}{\kappa_D},
\]
which proves the second.
\end{proof}

The quantity $c_{\A,\Tsp}$ admits a direct subspace interpretation.  It is the sine
of the smallest angle between the lifted subspace $\A(\V^\perp)$ and the
linear space $\Tsp$: values near one correspond to strongly transverse
subspaces, whereas values near zero indicate that a lifted direction from
$\V^\perp$ is nearly aligned with $\Tsp$.  Proposition~\ref{prop:pub-angle} shows that $\chi_{\Tsp}$
combines this angular separation with the multiplicity distortion
$\kappa_D=d_{\max}/d_{\min}$ introduced by reflection.

\subsection{Reference-Lift Compatibility}
\label{app:pub-reference}

\begin{proposition}[Reference-lift compatibility]
\label{prop:pub-cref}
Let
\[
\zeta_{\mathrm{ref}}
=
\norm{\mathscr P_{\Tsp^\perp}(\A(\A^*(\bm{U}\bm{V}^\trans)))}_F.
\]
Then
\[
c_{\mathrm{ref}}
\le
\frac{\zeta_{\mathrm{ref}}}{\beta_{\Tsp}^2}.
\]
\end{proposition}

\begin{proof}
Let
$\mathcal Q_{\mathrm{ref}}=\mathscr P_{\V^\perp}(\mathcal B_0)$.  The adjoint restricted to its natural domain is
\[
\KT^*=\left.\mathscr P_{\Tsp^\perp}\A\right|_{\V^\perp},
\]
and its smallest singular value is at least $\beta_{\Tsp}$.  Therefore
\[
\norm{\mathcal Q_{\mathrm{ref}}}_F
\le
\frac{\norm{\KT^*(\mathcal Q_{\mathrm{ref}})}_F}{\beta_{\Tsp}}
=
\frac{\zeta_{\mathrm{ref}}}{\beta_{\Tsp}},
\]
because $\KT^*$ vanishes on $\V$.  Also
\[
c_{\mathrm{ref}}
=
\norm{\KT^\dagger(\mathcal Q_{\mathrm{ref}})}_2
\le
\norm{\KT^\dagger(\mathcal Q_{\mathrm{ref}})}_F
\le
\frac{\norm{\mathcal Q_{\mathrm{ref}}}_F}{\beta_{\Tsp}}.
\]
Combining the two inequalities proves the claim.
\end{proof}

For an interpretation of this bound, let $\mathcal Q_{\mathrm{ref}}=\mathscr P_{\V^\perp}(\mathcal B_0)$.  This is the component of the pulled-back reference subgradient that cannot be represented within the preimage space $\V$.  A dual certificate must
cancel $\mathcal Q_{\mathrm{ref}}$ using a normal matrix in $\Tsp^\perp$.  Hence $c_{\mathrm{ref}}$ is the operator-norm cost of the canonical
minimum-Frobenius correction, whereas $\chi_{\Tsp}$ measures angular separation.
The bound above separates this correction cost into a
reference-lift mismatch, measured by $\zeta_{\mathrm{ref}}$, and the ability of
$\Tsp^\perp$ to act on $\V^\perp$, measured by $\beta_{\Tsp}$.

\subsection{Scope of the Reflection-Specific Assumptions}
\label{app:pub-scope}

Theorem~\ref{thm:core-main} uses two distinct reflection-specific
compatibility requirements in addition to coherence.  A positive lower
bound on $\chi_{\Tsp}$ excludes lifted perturbations from $\V^\perp$ that are
nearly aligned with the linear space $\Tsp$, while an upper bound on
$c_{\mathrm{ref}}$ ensures that correcting the pullback mismatch does not exhaust
the spectral-norm margin of the dual certificate.  Neither condition follows
from low rank alone.  The preceding propositions relate $\chi_{\Tsp}$ to a
subspace angle and multiplicity distortion and bound $c_{\mathrm{ref}}$ through
the reference-lift mismatch.  They are sufficient geometric conditions for
the present proof and are not claimed to be necessary.  The parameter
$\mu_b$ likewise remains instance dependent; the present analysis does not
provide a general low-rank-only bound for it.

The probabilistic recovery theorem also uses independent Rademacher
corruption signs in the sparse-certificate concentration step.  This is a
condition of the present recovery analysis rather than an input required by
the RCNNM optimization model \eqref{eq:core-rcnnm}.

\section{Stability and Optimization Proofs}
\label{app:stability-opt}

\subsection{Descent Cone and Restricted Observation Gain}
\label{app:descent-gain}

Define the feasible descent cone of the RCNNM objective in \eqref{eq:core-rcnnm} at $(\mathcal L_0,\mathcal S_0)$ by
\begin{equation}
\begin{aligned}
\Cset(\mathcal L_0,\mathcal S_0)=\{(\mathcal H,\mathcal J):\;&\mathscr P_{\Oset^c}(\mathcal J)=0,\ \exists\,t>0\ \text{such that}\\
&\norm{\A(\mathcal L_0+t\mathcal H)}_*+\lambda\norm{\mathcal S_0+t\mathcal J}_1\\
&\le\norm{\A(\mathcal L_0)}_*+\lambda\norm{\mathcal S_0}_1\}.
\end{aligned}
\label{eq:pub-descent-cone}
\end{equation}
Its restricted observation gain is
\begin{equation}
\kappa
=
\Inf_{\substack{(\mathcal H,\mathcal J)\in\Cset(\mathcal L_0,\mathcal S_0)\\(\mathcal H,\mathcal J)\ne(0,0)}}
\frac{\norm{\mathscr P_\Oset(\mathcal H+\mathcal J)}_F}
{\bigl(\norm{\mathcal H}_F^2+\norm{\mathcal J}_F^2\bigr)^{1/2}}.
\label{eq:pub-restricted-gain}
\end{equation}
Thus every nonzero $(\mathcal H,\mathcal J)\in\Cset(\mathcal L_0,\mathcal S_0)$ satisfies
\[
\norm{\mathscr P_\Oset(\mathcal H+\mathcal J)}_F
\ge
\kappa\sqrt{\norm{\mathcal H}_F^2+\norm{\mathcal J}_F^2}.
\]
A positive restricted gain prevents an objective-descent direction from becoming nearly invisible on the observed coordinates and converts the observed residual into a reconstruction error bound.

For a certificate
\[
\A^*(\bm{U}\bm{V}^\trans+\bm{W})=\mathcal Y+\mathcal R,\qquad
\mathcal Y=\lambda(\mathcal S_0^{\mathrm{sgn}}+\mathcal F),\qquad \mathcal R\in\V,
\]
set
\[
\rho=\norm{\mathcal R}_F,
\qquad
\bar\delta_2
=
\delta_2-\frac{\rho}{\alpha\beta_{\Tsp}},
\qquad
\bar\delta_\infty
=
\delta_\infty-\frac{\rho}{\alpha\lambda},
\]
and
\[
B_R=\norm{\mathcal Y}_F+\frac{\rho}{\alpha}.
\]

\subsection{Approximate-Certificate Stability}
\label{app:approx-stability}

\begin{theorem}[Approximate-certificate stability]
\label{thm:pub-approx-stability}
Under the hypotheses of Proposition~\ref{thm:core-approx}, the quantities
$\bar\delta_2$ and $\bar\delta_\infty$ defined above are positive and
\[
\kappa\ge C_{\mathrm{app}}^{-1},
\]
where
\begin{align}
C_{\mathrm{app}}={}&
1+\frac2\alpha
+
\frac{2B_R}{\alpha\lambda\bar\delta_\infty}\nonumber\\
&+
2\left(1+\frac1\alpha\right)
\frac{B_R}{\beta_{\Tsp}\bar\delta_2}.
\label{eq:pub-Capp}
\end{align}
Consequently every minimizer of the noise-tolerant RCNNM model \eqref{eq:core-noisy-constraint} satisfies
\[
\sqrt{
\norm{\widehat{\mathcal L}-\mathcal L_0}_F^2+
\norm{\widehat{\mathcal S}-\mathcal S_0}_F^2}
\le
2\eps C_{\mathrm{app}}.
\]
\end{theorem}

\begin{proof}
For a descent direction $(\mathcal H,\mathcal J)$ define
\begin{align*}
e&=\norm{\mathscr P_\Oset(\mathcal H+\mathcal J)}_F,\\
x&=\norm{\mathscr P_{\Tsp^\perp}(\A(\mathcal H))}_*,
\qquad
y=\norm{\mathscr P_\Lset (\mathcal J)}_1.
\end{align*}
The approximate certificate gives
\[
\delta_2x+\lambda\delta_\infty y
\le
\norm{\mathcal Y}_F e+\abs{\ip{\mathcal R}{\mathcal H}}.
\]
As in the uniqueness proof,
\[
\abs{\ip{\mathcal R}{\mathcal H}}
\le
\frac{\rho}{\alpha}
\left(e+y+\frac{x}{\beta_{\Tsp}}\right).
\]
Hence
\[
\bar\delta_2x+\lambda\bar\delta_\infty y
\le B_Re.
\]
Furthermore,
\[
\alpha\norm{\mathcal H_\V}_F
\le
e+y+\frac{x}{\beta_{\Tsp}},
\qquad
\norm{\mathcal H_{\V^\perp}}_F\le\frac{x}{\beta_{\Tsp}}.
\]
Using the preceding bounds on $x$ and $y$ therefore gives
\[
\norm{\mathcal H}_F
\le
\left[
\frac1\alpha
+
\frac{B_R}{\alpha\lambda\bar\delta_\infty}
+
\left(1+\frac1\alpha\right)
\frac{B_R}{\beta_{\Tsp}\bar\delta_2}
\right]e.
\]
Since $\mathscr P_{\Oset^c}(\mathcal J)=0$,
\[
\norm{\mathcal J}_F
\le
\norm{\mathscr P_\Oset(\mathcal H+\mathcal J)}_F+\norm{\mathscr P_\Oset(\mathcal H)}_F
\le
e+\norm{\mathcal H}_F.
\]
Hence
\[
\sqrt{\norm{\mathcal H}_F^2+\norm{\mathcal J}_F^2}
\le
\norm{\mathcal H}_F+\norm{\mathcal J}_F
\le
C_{\mathrm{app}}e,
\]
which proves \eqref{eq:pub-Capp}.  Two feasible noisy points differ on the
observed set by at most $2\eps$, proving the final statement.
\end{proof}

\subsection{Stable Recovery Bound}
\label{app:stable-recovery-bound}

\begin{proof}[Proof of Corollary~\ref{cor:core-stability}]
For the certificate in Theorem~\ref{thm:core-main}, the bounds established in \ref{app:pub-closure} give
\[
\delta_\infty\ge\frac34,\qquad
\delta_2>\frac14,
\]
and
\[
\frac{\rho}{\alpha\lambda}\le\frac1{32},
\qquad
\frac{\rho}{\alpha\beta_{\Tsp}}\le\frac1{32}.
\]
Thus
\[
\bar\delta_\infty\ge\frac{23}{32},
\qquad
\bar\delta_2>\frac7{32}.
\]
Also, $\mathcal S_0^{\mathrm{sgn}}+\mathcal F$ is supported on $\Oset$ and each of its entries has
magnitude at most one.  Therefore
\[
\norm{\mathcal Y}_F
=
\lambda\norm{\mathcal S_0^{\mathrm{sgn}}+\mathcal F}_F
\le
\lambda\sqrt{\abs\Oset}
\le
\lambda\sqrt N,
\]
and hence
\[
B_R\le\lambda\left(\sqrt N+\frac1{32}\right).
\]
Using
$\alpha\ge\sqrt{q/2}$ and
$\lambda/\beta_{\Tsp}=c_\lambda/\sqrt{pN}$
in \eqref{eq:pub-Capp} yields a finite constant
$\widetilde C_{\mathrm{stab}}>0$ depending only on the sampling and recovery
parameters in Theorem~\ref{thm:core-main}, such that
$C_{\mathrm{app}}\le\widetilde C_{\mathrm{stab}}$.
Absorbing the factor $2$ in the preceding stability estimate into
$C_{\mathrm{stab}}:=2\widetilde C_{\mathrm{stab}}$ gives
\eqref{eq:core-stability-bound}.

\end{proof}

\subsection{Gram-Based SVT Identity}
\label{app:gram-svt-proof}

\begin{proof}[Proof of Proposition~\ref{prop:gram-svt}]
Let a compact SVD of $\bm H$ be
\[
\bm H=\bm U\bm\Sigma\bm V^\trans.
\]
Then
\[
\bm H^\trans\bm H=\bm V\bm\Sigma^2\bm V^\trans.
\]
Hence the Gram eigenvectors are the right singular vectors, and the positive singular values are the square roots of the positive Gram eigenvalues. Let
$\Jtau=\{i:\sigma_i>\tau\}$. Since
\[
\bm H\bm V_\tau=\bm U_\tau\bm\Sigma_\tau,
\]
we have
\begin{align*}
\bm H\bm V_\tau\bigl(\bm I-\tau\bm\Sigma_\tau^{-1}\bigr)\bm V_\tau^\trans
&=\bm U_\tau(\bm\Sigma_\tau-\tau\bm I)\bm V_\tau^\trans\\
&=\SVT_\tau(\bm H).
\end{align*}
Terms with $\sigma_i\le\tau$ vanish. The clipping of a numerically negative
$\xi_i$ to zero does not alter the exact-arithmetic identity because
$\bm H^\trans\bm H\succeq0$.
\end{proof}


\section*{Acknowledgments}
This work was supported in part by the National Key Research and Development Program of China under Grant 2023YFA1008502; in part by the Fundamental Research Funds for the Central Universities under Grant SWU-KR25013; and in part by the National Natural Science Foundation of China under Grant 12101512.

\end{document}